\PassOptionsToPackage{numbers}{natbib}
\documentclass{article}

\usepackage[final]{neurips_2022}

\usepackage{hyperref}
\usepackage{url}
\usepackage[utf8]{inputenc}
\usepackage{todonotes}
\usepackage{amssymb}
\usepackage{url}
\usepackage{amsmath,bm,amsfonts,amsthm,mathtools}

\usepackage{geometry}

\usepackage{mathrsfs}
\usepackage{amssymb}
\usepackage{amsthm}
\usepackage{bm}
\usepackage[titletoc]{appendix}
\usepackage{clipboard}
\usepackage[shortlabels]{enumitem}
\usepackage{float}
\usepackage{graphicx} 
\usepackage{pdfpages}
\usepackage{subfigure}
\usepackage[font=small]{caption}

\usepackage[lined,boxed,ruled, commentsnumbered]{algorithm2e}

\usepackage{tikz}

\usepackage[utf8]{inputenc} 
\usepackage[T1]{fontenc}    
\usepackage{hyperref}       
\usepackage{url}            
\usepackage{booktabs}       
\usepackage{amsfonts}       
\usepackage{nicefrac}       
\usepackage{microtype}      
\usepackage{xcolor}         
\usepackage{pifont}
\usepackage{hhline}
\usepackage{threeparttable}
\usepackage{footnote}
\usepackage{makecell}

\makeatletter
\newcommand*{\inlineequation}[2][]{%
  \begingroup
    \refstepcounter{equation}%
    \ifx\\#1\\%
    \else
      \label{#1}%
    \fi
    \relpenalty=10000 %
    \binoppenalty=10000 %
    \ensuremath{%
      #2%
    }%
    ~\@eqnnum
  \endgroup
}
\makeatother

\newtheorem{theorem}{Theorem}
\newtheorem{corollary}{Corollary}
\newtheorem{proposition}{Proposition}

\newtheorem{remark}{Remark}
\newtheorem{assumption}{Assumption}
\newtheorem{appxlem}{Lemma}[section]
\newtheorem{appxcor}{Corollary}[section]

\begin{document}
\newcommand{\E}{\mathbb{E}}
\newcommand{\rs}[2]{\binom{[#2]}{#1}}  
\newcommand{\R}{\mathbb{R}}
\newcommand{\sus}{s_2}  
\newcommand{\bma}{\bm{a}}
\newcommand{\bmb}{\bm{b}}
\newcommand{\x}{\bm{x}}
\newcommand{\X}{\bm{X}}
\newcommand{\y}{\bm{y}}
\newcommand{\z}{\bm{z}}
\newcommand{\w}{\bm{w}}
\newcommand{\uu}{\bm{u}}
\newcommand{\A}{\bm{A}}
\newcommand{\e}{\bm{\varepsilon}}
\newcommand{\stoch}{\bm{\xi}}
\newcommand{\OO}{\mathcal{O}}

\title{Zeroth-Order Hard-Thresholding:\\Gradient Error vs. Expansivity}

\author{%
  William de Vazelhes$^{1}$, Hualin Zhang$^{2}$, Huimin Wu$^{2}$, Xiao-Tong Yuan$^{2}$, Bin Gu$^{1}$ \\
  $^{1}$Mohamed bin Zayed University of Artificial Intelligence \\$^{2}$Nanjing University of Information Science \& Technology\\
 \texttt{\{wdevazelhes,zhanghualin98,xtyuan1980,jsgubin\}@gmail.com}, \\ \texttt{wuhuimin@nuist.edu.cn} \\
}

\date{2022}

\maketitle

\begin{abstract}
$\ell_0$ constrained optimization is prevalent in machine learning, particularly for high-dimensional problems, because it is a fundamental approach to achieve sparse learning. Hard-thresholding gradient descent is a dominant technique to solve this problem. However, first-order gradients of the objective function may be either unavailable or expensive to calculate in a lot of real-world problems, where zeroth-order (ZO) gradients could be a good surrogate. Unfortunately, whether ZO gradients can work with the hard-thresholding operator is still an unsolved problem.
To solve this puzzle, in this paper, we focus on the $\ell_0$ constrained black-box stochastic optimization problems, and propose a new stochastic zeroth-order gradient hard-thresholding (SZOHT) algorithm with  a general ZO gradient estimator powered by a novel random support sampling. We provide the convergence analysis of SZOHT under standard assumptions.   Importantly, we   reveal a conflict between  the deviation of  ZO estimators and  the expansivity of the hard-thresholding operator,  and provide a theoretical   minimal value of the number of random directions in ZO gradients. In addition,  we find that the query complexity of SZOHT is independent or weakly dependent on the dimensionality under different settings.  Finally, we illustrate the utility of our method on a portfolio optimization problem as well as black-box adversarial attacks.
  \end{abstract}

\section{Introduction}
\label{sec:introduction}

$\ell_0$ constrained optimization is prevalent in machine learning, particularly for high-dimensional problems, because it is a fundamental approach to achieve sparse learning. In addition to improving the memory, computational and environmental footprint of the models, these sparse constraints help reduce overfitting and obtain consistent statistical estimation \cite{yuan2021stability,buhlmann2011statistics,raskutti2011minimax,negahban2012unified}.
We formulate the problem as follows:
\begin{equation}
  \min \limits_{\x\in \R^d} \left\{ f(\x) := \E_{\stoch} f(\x,\stoch) \right\},  \quad  \text{ s.t. }  \quad   \|\x\|_0 \leq k
  \label{eq:opt_pb}  
\end{equation}
where $f(\cdot, \stoch): \mathbb{R}^d \rightarrow \mathbb{R}$ is a differentiable function and $\stoch$ is a noise term, for instance related to an underlying finite sum structure in $f$, of the form: $\E_{\stoch} f(\x, \stoch) = \frac{1}{n}\sum_{i=1}^n f_i(\x)$.
Hard-thresholding gradient algorithm \cite{Jain14,nguyen2017linear,yuan2017gradient} is a dominant technique to solve this problem. It generally consists in alternating between a gradient step, and a hard-thresholding operation which only keeps the $k$-largest components (in absolute value) of the current iterate. The advantage of hard-thresholding over its convex relaxations (\cite{tibshirani1996regression,van2008high}) is that it can often attain similar precision, but is more computationally efficient, since it can directly ensure a desired sparsity level instead of tuning an $\ell_1$ penalty or constraint. The only expensive computation in hard-thresholding is the hard-thresholding step itself, which requires finding the top $k$ elements of the current iterate.  Hard-thresholding was originally developed in its full gradient form \cite{Jain14}, but has been later on extended to the stochastic setting by \citet{nguyen2017linear}, which developed a stochastic gradient descent (SGD) version of hard thresholding (StoIHT), and further more with \citet{Zhou18}, \citet{shen2017tight} and \citet{li2016nonconvex}, which used variance reduction technique to improve upon StoIHT.

\begin{table}[t]
  \caption{{Complexity of sparsity-enforcing algorithms. We give the query complexity for a precision $\varepsilon$, up to the system error (see section \ref{sec:glob}). For first-order algorithms (FO), we give it in terms of number of first order oracle calls (\#IFO), that is, calls to $\nabla f(x, \stoch)$, and for ZO algorithms,  in terms of calls of $f(\stoch, \cdot)$. Here $\kappa$ denotes the condition number $\frac{L}{\nu}$, with $L$ is the smoothness (or RSS) constant and $\nu$ is the strong-convexity (or RSC) constant.
}}
  \label{table:sota}
  \begin{threeparttable}
  \begin{tabular}[c]{l l l l l}
    Type& Name &  Assumptions & \#IZO/\#IFO  & \#HT ops.\\
    \hhline{= = = = =}
    FO/$\ell_0$ &    StoIHT \cite{nguyen2017linear} & RSS, RSC & $\OO(\kappa \log(\frac{1}{\varepsilon }))$ & $\OO(\kappa \log (\frac{1}{\varepsilon}))$\\
    \midrule
    ZO/$\ell_1$ & RSPGF \cite{ghadimi2016mini} & smooth\tnote{3}  & $\OO(\frac{d}{\varepsilon^2})$ & --- \\
    \midrule
        ZO/$\ell_1$ & ZSCG\tnote{2} \cite{Balasubramanian18} & \makecell[l]{convex, smooth} & $\OO(\frac{d}{\varepsilon^2})$ & --- \\
    \midrule
    ZO/$\ell_1$ & ZORO \cite{Cai20} & \makecell[l]{$s$-sparse gradient,\\
                                      weakly sparse hessian,\\
    smooth\tnote{3},  $\text{RSC}_{\text{bis}}$\tnote{1}} & $\OO(s \log(d) \log(\frac{1}{\varepsilon}))$ & --- \\
    \midrule
    ZO/$\ell_0$ & \textbf{SZOHT} & RSS, RSC&  $ \OO((k + \frac{d}{\sus}) \kappa^2 \log(\frac{1}{\varepsilon}))$ & $ \OO(\kappa^2\log(\frac{1}{\varepsilon}))$ \\
    \midrule
    ZO/$\ell_0$  & \textbf{SZOHT} &smooth, RSC&  $\OO(k \kappa^2\log(\frac{1}{\varepsilon}))$ & $\OO(\kappa^2 \log (\frac{1}{\varepsilon}))$ \\
    \bottomrule
  \end{tabular}
  \begin{tablenotes}\footnotesize
  \item[1] The definition of Restricted Strong Convexity from \cite{Cai20} is different from ours and that of \cite{nguyen2017linear}, hence the $\text{bis}$ subscript. \item[2] We refer to the modified version of ZSCG (Algorithm 3 in \cite{Balasubramanian18}).  \item[3] RSPGF and ZORO minimize $f(x) + \lambda \|x\|_1$: only $f$ needs to be smooth.
      \end{tablenotes}
  \end{threeparttable}
  \vspace{-16pt}
\end{table}
However, the first-order gradients used in the above methods may be either unavailable or expensive to calculate in a lot of real-world problems. For example, in certain graphical modeling tasks \cite{wainwright2008graphical}, obtaining the gradient of the objective function is computationally hard. Even worse, in some settings, the gradient is inaccessible by nature, for instance in bandit problems \cite{Shamir17}, black-box adversarial attacks \cite{Tu19,chen2017zoo,chen2019zo}, or reinforcement learning \cite{salimans2017evolution,mania2018simple,choromanski2020provably}. To tackle those problems, ZO  optimization methods have been developed \cite{Nesterov17}. Those methods usually replace the inaccessible gradient by its finite difference approximation which can be computed only from function evaluations, following the idea that for a differentiable function $f: \mathbb{R}\rightarrow \mathbb{R}$, we have: $f'(x) = \lim_{h\rightarrow 0}\frac{f(x+h) - f(x)}{h}$. Later on, ZO methods have been adapted to deal with a convex constraints set, and can therefore be used to solve the $\ell_1$ convex relaxation of problem \eqref{eq:opt_pb}. To that end, \citet{ghadimi2016mini}, and \citet{Cai20} introduce  proximal ZO algorithms, \citet{liu2018zeroth} introduce a  ZO projected gradient algorithm and \citet{Balasubramanian18} introduce a ZO conditional gradient \cite{levitin1966constrained} algorithm.  We provide a review of those results in Table \ref{table:sota}. As can be seen from the table, their query complexity is high (linear in $d$), except \cite{Cai20} that has a complexity of $\OO(s \log(d) \log(\frac{1}{\varepsilon}))$, but assumes that gradients are sparse. In addition, those methods must introduce a hyperparameter $\lambda$ (the strength of the $\ell_1$ penalty) or $R$ (the radius of the $\ell_1$ ball), which need to be tuned to find which value ensures the right sparsity level. Therefore, it would be interesting to use the hard-thresholding techniques described in the previous paragraph, instead of those convex relaxations.

Unfortunately, ZO hard-thresholding gradient algorithms have not been exploited formally. Even more, whether ZO gradients can work with the hard-thresholding operator is still an unknown problem. Although there was one related algorithm proposed recently by \citet{Balasubramanian18}, they did not target $\ell_0$ constrained optimization problem and importantly have strong assumptions in their convergence analysis. Indeed, they assume that the gradients,  as well as the solution of the unconstrained problem, are $s$-sparse: $\|\nabla f(\x)\|_0\leq s $ and $\|\x^*\|_0\leq s^* \approx s$, where $\x^* = \arg\min_{\x} f(\x)$. In addition, it was recently shown by \citet{Cai20} that they must in fact assume that the support of the gradient is fixed for all $\x\in \mathcal{X}$, for their convergence result to hold, which is a hard limitation, since that amounts to say that the function $f$  depends on $s$ fixed variables. 

To fill this gap, in this paper, we focus on the $\ell_0$ constrained black-box stochastic optimization problems, and propose a novel stochastic zeroth-order gradient hard-thresholding (SZOHT) algorithm. Specifically, we propose a dimension friendly ZO gradient estimator powered by a novel random support sampling technique, and then embed it into the standard hard-thresholding operator. 

We then provide the convergence and complexity analysis of SZOHT under the standard assumptions of sparse learning, which are restricted strong smoothness (RSS), and restricted strong convexity (RSC)  \cite{nguyen2017linear,shen2017tight}, to retain generality, therefore providing a positive answer to the question of whether ZO gradients can work with the hard-thresholding operator. Crucial to our analysis is to provide carefully tuned requirements on the parameters $q$ (the number of random directions used to estimate the gradient, further defined in Section \ref{sec:definition-algorithm}) and $k$. Finally, we illustrate the utility of our method on a portfolio optimization problem as well as black-box adversarial attacks, by showing that our method can achieve competitive performance in comparison to state of the art methods for sparsity-enforcing zeroth-order algorithm described in Table \ref{table:sota}, such as  \cite{ghadimi2016mini,Balasubramanian18,Cai20}.

Importantly, we also show that in the smooth case, the query complexity of SZOHT is independent of the dimensionality, which is significantly different to the dimensionality dependent results for most existing ZO algorithms. Indeed, it is known from  \citet{Jamieson12} that the worst case query complexity of ZO optimization over the class $\mathcal{F}_{\nu, L}$ of $\nu$-strongly convex and $L$-smooth functions defined over a convex set $\mathcal{X}$ is linear in $d$. Our work is thus in line with other works achieving dimension-insensitive query complexity in zeroth-order optimization such as \cite{Golovin19,sokolov2018sparse,Wang18,Cai20,cai2021zeroth,Balasubramanian18,Cai20,Liu21,Jamieson12}, but contrary to those, instead of making further assumptions (i.e. restricting the class $\mathcal{F}_{\nu, L}$ to a smaller class), we bypass the impossibility result by replacing the convex feasible set $\mathcal{X}$ by a \textit{non-convex} set (the $\ell_0$ ball), which is how we can avoid making stringent assumptions on the class of functions $f$.

\vspace{-4pt}
\paragraph{Contributions.}
\label{sec:contributions}
We summarize the main contributions of our paper as follows:
\begin{enumerate}[leftmargin=0.2in]
\vspace{-5pt}
\item We propose a new algorithm SZOHT that is, up to our knowledge, the first zeroth-order sparsity constrained algorithm that is dimension independent under the smoothness assumption, without assuming any gradient sparsity.
\vspace{-3pt}
\item We reveal an interesting conflict between the error from zeroth-order estimates and the hard-thresholding operator, which results in a minimal value for the number of random directions $q$ that is necessary to ensure at each iteration. 
\vspace{-3pt}  \item We also provide the convergence analysis of our algorithm in the more general RSS setting, providing, up to our knowledge, the first zeroth-order algorithm that can work with the usual assumptions of RSS/RSC from the hard-thresholding literature.
\vspace{-3pt}
  \end{enumerate}

\section{Preliminaries}\label{sec:prelim}
Throughout this paper, we denote by  $\|\x \|$ the Euclidean norm for a vector $\x \in \R^d$, by $\|\x \|_{\infty}$ the maximum absolute component of that vector, and by $\|\x\|_0$ the $\ell_0$ norm (which is not a proper norm). For simplicity, we denote $f_{\stoch}(\cdot):= f(\cdot, \stoch)$. We call $\uu_F$ (resp. $\nabla_F f(\x)$) the vector which  sets all  coordinates $i\not\in F$ of  $\uu$ (resp. $\nabla f(\x)$) to $0$. We also denote by $\x^*$ the solution of problem \eqref{eq:opt_pb} defined in the introduction, for some target sparsity $k^*$ which could be smaller than $k$. To derive our result, we will need the following assumptions on $f$.

\begin{assumption}[$(\nu_s, s)$-RSC, \cite{Jain14,negahban2009unified,loh2013regularized,yuan2017gradient,li2016nonconvex,shen2017tight,nguyen2017linear}]\label{ass:RSC}
$f$ is  said to be $\nu_s$ restricted strongly convex with sparsity parameter $s$ if it is differentiable, and there exist a generic constant $\nu_s$ such that for all $(\x, \y) \in \R^d$ with $\|\x - \y\|_0 \leq s$:
  $$ f(\y) \geq  f(\x) +\langle \nabla f(\x), \y-\x \rangle + \frac{\nu_s}{2} \| \x- \y\|^2 $$ 
\end{assumption}

\begin{assumption}[$(L_{s}, s)$-RSS, \cite{shen2017tight,nguyen2017linear}]
\label{ass:RSS}
  For almost any $\stoch$, $f_{\stoch}$ is said to be $L_{s}$ restricted smooth with sparsity level $s$, if it is differentiable, and there exist a generic constant $L_{s}$ such that for all $(\x, \y) \in \R^d$ with $\|\x-\y\|_0\leq s$:
  $$\|\nabla f_{\stoch}(\x) - \nabla f_{\stoch}(\y)\| \leq L_{s} \| \x-\y\|$$
\end{assumption}
\begin{assumption}[$\sigma^2$-FGN \cite{Gower19}, Assumption 2.3 (Finite Gradient Noise)]
  $f$ is said to have $\sigma$-finite gradient noise if for almost any $\stoch$, $f_{\stoch}$ is differentiable  and the gradient noise $\sigma = \sigma(f, \stoch)$ defined below is finite:
  $$\sigma^2 = \E_{\stoch} [\|\nabla f_{\stoch}(\x^*)\|_{\infty}^2]$$
\end{assumption}
\begin{remark}
Even though the original version of \cite{Gower19} uses the $\ell_2$ norm, we use the $\ell_{\infty}$ norm here, in order to give more insightful results in terms of $k$ and $d$, as is done classically in $\ell_0$ optimization, similarly to \cite{Zhou18}. We also note that in \cite{Gower19}, $\x^*$ denotes an unconstrained minimum when in our case it denotes the constrained minimum for some sparsity $k^*$.
\end{remark}

For Corollary \ref{cor:s2d}, we will also need the more usual  smoothness assumption: 
\begin{assumption}[$L$-smooth]
\label{ass:smooth}
  For almost any $\stoch$, $f_{\stoch}$ is said to be $L$ smooth, if it is differentiable, and for all $(\x, \y) \in \R^d$ :
  $$\|\nabla f_{\stoch}(\x) - \nabla f_{\stoch}(\y)\| \leq L \| \x-\y\|$$
\end{assumption}

\section{Algorithm}
\subsection{Random support Zeroth-Order estimate}
\label{sec:definition-algorithm}

In this section, we describe our zeroth-order gradient estimator. It is basically composed of a random support sampling step, followed by a random direction with uniform smoothing on the sphere supported by this support. We also use the technique of averaging our estimator over $q$ dimensions, as described in \cite{Liu20}. More formally, our gradient estimator is described below:
\begin{equation}
\hat{\nabla}f_{\stoch}(\x) = \frac{d}{q\mu}  \sum_{i=1}^q\left(f_{\stoch}(\x+\mu \uu_i) - f_{\stoch}(\x)\right) \uu_i  \label{eq:zoest}
\end{equation}
where each random direction $\uu_i$  is a unit vector sampled uniformly from the set $\mathcal{S}_{\sus}^d:= \{\uu \in \mathbb{R}^d: \|\uu\|_0\leq \sus, \|\uu\| = 1 \} $. We can obtain such vectors $\uu$ by sampling first a random support $S$ (i.e. a set of coordinates) of size $\sus$ from $[d]$, (denoted as $S\sim \mathcal{U}(\binom{[d]}{\sus})$ in Algorithm \ref{alg:SZOHT}) and then  by sampling a random unit vector $\uu$ supported on that support $S$, that is, uniformly sampled from the set $\mathcal{S}_S^{d} :=\{\uu\in \mathbb{R}^d: \uu_{[d]-S}=\bm{0}, \|\uu\| = 1\}$, (denoted as $\uu \sim \mathcal{U}(\mathcal{S}_S^{d}$) in algorithm \ref{alg:SZOHT}). The original uniform smoothing technique on the sphere is described in more detail in   \cite{gao2018information}. However, in our case, the sphere along which we sample is restricted to a random support of size $\sus$.
Our general estimator, through the setting of the variable $\sus$, can take several forms, which are similar to pre-existing gradient estimators from the literature described below:
\begin{itemize}[leftmargin=0.2in]
\vspace{-2pt}
\item If $\sus=d$, $\hat{\nabla} f_{\stoch}(\x)$ is the \textit{usual vanilla estimator with uniform smoothing on the sphere} \cite{gao2018information}.
\vspace{-2pt}
\item If $1 \leq \sus \leq d$, our estimator is similar to the Random Block-Coordinate gradient estimator from  \citet{lian2016comprehensive}, except that the blocks are not fixed at initialization but chosen randomly, and that we use a uniform smoothing with forward difference on the given block instead of a coordinate-wise estimation with central difference. This random support technique allows us to give a convergence analysis under the classical assumptions of the hard-thresholding literature (see Remark \ref{rem:generality}), and to deal with huge scale optimization, when sampling uniformly from a unit $d$-sphere is costly  \cite{Cai20,cai2021zeroth}: in the distributed setting for instance, each worker would just need to sample an $s_2$-sparse random vector, and only the centralized server would materialize the full gradient approximation containing up to $q s_2$ non-zero entries.
\vspace{-2pt}
\end{itemize}

\textbf{Error Bounds of the  Zeroth-Order Estimator.} \
We now derive  error bounds on the gradient estimator, that will be useful in the convergence rate proof, except that we consider \textit{only the restriction to some support $F$} (that is, we consider a subset of components of the gradient/estimator). Indeed, proofs in the hard-thresholding literature (see for instance \cite{yuan2017gradient}), are usually written only on that support. That is the key idea which explains how the dimensionality dependence is reduced when doing SZOHT compared to vanilla ZO optimization. We give  more insight on the shape of the original distribution of gradient estimators, and the distribution of their projection onto a hyperplane $F$ in Figure \ref{fig:fig} in Appendix \ref{sec:appfig}. We can observe that even if the original gradient estimator is poor, in the projected space, the estimation error is reduced, which we quantify in the proposition below.

\begin{proposition}(Proof in Appendix \ref{sec:proof_final_zo} )\label{prop:zograd} Let us consider any support $F\subset [d]$ of size $s$ ($|F|=s$). For the Z0 gradient estimator in \eqref{eq:zoest}, with $q$ random directions, and random supports of size $\sus$, and assuming that each $f_{\stoch}$ is $(L_{\sus}, \sus)$-RSS, we have, with $\hat{\nabla}_Ff_{\stoch}(\x)$ denoting the hard thresholding of the gradient $\nabla f_{\stoch}(\x)$ on $F$ (that is, we set all coordinates not in $F$ to $0$):
  \begin{enumerate}[(a)]
  \item  $\|\E \hat{\nabla}_F f_{\stoch}(\x) - \nabla_F f_{\stoch}(\x)\|^2\leq \varepsilon_{\mu} \mu^2 $ 
  \item $\E\|\hat{\nabla}_{F} f_{\stoch}(\x)\|^{2} \leq \varepsilon_F \|\nabla_F f_{\stoch}(\x)\|^{2} +  \varepsilon_{F^c}  \|\nabla_{F^c} f_{\stoch}(\x)\|^{2}+   \varepsilon_{abs} \mu ^2$
  \item $\E \|\hat{\nabla}_F f_{\stoch}(\x) -  \nabla_F f_{\stoch}(\x)\|^2 \leq  2(\varepsilon_F +1)\|\nabla_F f_{\stoch}(\x)\|^{2} +  2 \varepsilon_{F^c}  \|\nabla_{F^c} f_{\stoch}(\x)\|^{2}+  2 \varepsilon_{abs} \mu ^2 $
    \end{enumerate}

    \begin{equation}\label{eq:constants}
      \begin{aligned}
       &\text{with} \quad \varepsilon_{\mu} = L_{\sus}^2sd, \quad \varepsilon_F = \frac{2d}{q(\sus + 2)}   \left(\frac{(s-1)(\sus-1)}{d-1} + 3\right) + 2, \\
        \quad \varepsilon_{F^c} &=    \frac{2d}{q(\sus + 2)}  \left( \frac{s(\sus-1)}{d-1}\right)
      \text{ and }~\varepsilon_{\text{abs}} =   \frac{2d L_{\sus}^2 s \sus}{q}\left(\frac{(s-1)(\sus-1)}{d-1}+1\right) +  L_{\sus}^2sd      \end{aligned}
    \end{equation}

\end{proposition}

\subsection{SZOHT Algorithm}

We now present our full algorithm to optimize problem \ref{eq:opt_pb}, which we name SZOHT (Stochastic Zeroth-Order Hard Thresholding). Each iteration of our algorithm is composed of two steps: (i) the gradient estimation step, and (ii) the hard thresholding step, where the gradient estimation step is the one described in the section above, and the hard-thresholding is described in more detail in the following paragraph. We give the full formal description of our algorithm in Algorithm \ref{alg:SZOHT}.

In the hard thresholding step, we only keep the $k$ largest (in magnitude) components of the current iterate $x^t$. This ensures that all our iterates (including the last one) are $k$-sparse. This hard-thresholding operator has been studied for instance in \cite{shen2017tight}, and possesses several interesting properties. Firstly, it can be seen as a projection on the $\ell_0$ ball. Second, importantly, it is not non-expansive, contrary to other operators like the soft-thresholding operator \cite{shen2017tight}. That expansivity plays an important role in the analysis of our algorithm, as we will see later.

Compared to previous works, our algorithm can be seen as a variant of Stochastic Hard Thresholding (StoIHT from \cite{nguyen2017linear}) 
, where we replaced the true gradient of $f_{\stoch}$ by the estimator $\hat{\nabla}f_{\stoch}(\x) $. It is also very close to Algorithm 5 from \citet{Balasubramanian18} (Truncated-ZSGD), with just a different zeroth-order gradient estimator: we use a uniform smoothing, random-block estimator, instead of their gaussian smoothing, full support vanilla estimator. This allows us to deal with very large dimensionalities, in the order of millions, similarly to \citet{cai2021zeroth}. Furthermore, as described in the Introduction, contrary to \citet{Balasubramanian18}, we provide the analysis of our algorithm without any gradient sparsity assumption.

The key challenge arising in our analysis is described in Figure \ref{fig:conflict}: the hard-thresholding operator being expansive \cite{shen2017tight}, each approximate gradient step must approach the solution enough  to stay close to it even after  hard-thresholding. Therefore, it is \textit{a priori} unclear whether the zeroth-order estimate can be accurate enough to guarantee the convergence of SZOHT. Hopefully, as we will see in the next section, we can indeed ensure convergence, as long as we carefully choose the value of $q$.
\begin{figure}[h]
  \centering
  \includegraphics[scale=0.4]{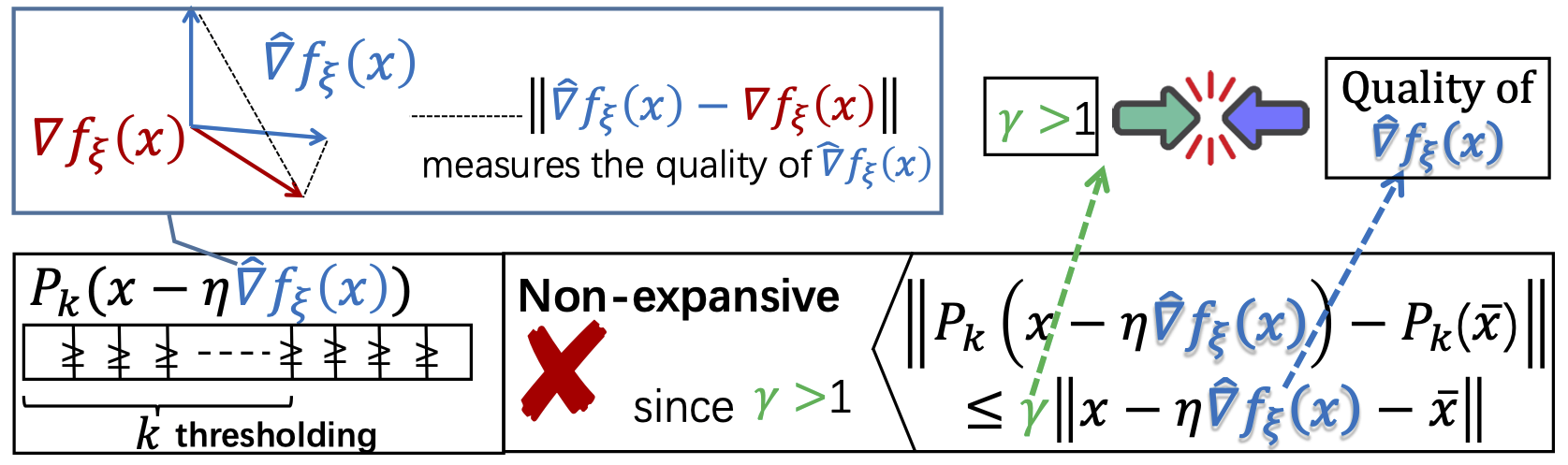}
  \caption{Conflict between the hard-thresholding operator and the zeroth-order estimate.}
  \label{fig:conflict}
  \vspace{-12pt}
\end{figure}

\begin{algorithm}
\SetKwInOut{Input}{Input}\SetKwInOut{Output}{Output}\SetKw{Initialization}{Initialization:}
\Initialization{Learning rate $\eta$, maximum number of iterations $T$, size of the random directions support $\sus$, number of random directions $q$, number of coordinates to keep at each iteration $k= \OO(\kappa^4 k^*)$, initial point $\x^{(0)}$ with $\|\x^{(0)}\|_0\le k^*$ (typically
$\x^{(0)}=0$), .}

\Output{$\x^{T}$.}

\For{$t=1, ..., T$} {
  Sample $\stoch$ (for instance sample a train sample $i$) \\
  \For{$i=1, ..., q$} {
    Sample a random support $S \sim  \mathcal{U}(\binom{[d]}{\sus})$\\
  Sample a random direction $\uu_i$ from the unit sphere supported on $S$:  $\uu_i\sim \mathcal{U}\left(\mathcal{S}^{d}_S\right)$

  Compute $\hat{\nabla}f_{\stoch}(\x^{t-1}; \uu_i)=\frac{d}{\mu}  \left(f_{\stoch}(\x+\mu \uu_i) - f_{\stoch}(\x)\right) \uu_i $;\\
}
Compute $\hat{\nabla}f_{\stoch}(\x^{t-1}) =\frac{1}{q}  \sum_{i=1}^q  \hat{\nabla}f_{\stoch}(\x^{t-1}; \uu_j)$

Compute $\tilde \x^{t} = \x^{t-1} - \eta \hat{\nabla} f_{\stoch}(\x^{t-1})$;

Compute $\x^{t} = \tilde \x^{t}_k$ as the truncation of $ \tilde
\x^{t}$ with top $k$ entries preserved;

}
 \caption{Stochastic Zeroth-Order Hard-Thresholding (SZOHT)}
 \label{alg:SZOHT}
\end{algorithm}
\section{Convergence analysis}
\vspace{-8pt}
\label{sec:glob}

In this section, we provide the convergence analysis of SZOHT, using the assumptions from section \ref{sec:prelim}, and discuss an interesting property of the combination of the zeroth-order gradient estimate and the hard-thresholding operator, providing a positive answer to the question from the previous section.

\begin{theorem}(Proof in Appendix \ref{sec:proof_cvrate})\label{thrm:cvrate}
 Assume that that each $f_{\stoch}$ is $(L_{s'}, s':=\max(\sus, s))$-RSS, and that $f$ is $(\nu_s, s)$-RSC and $\sigma$-FGN, with $s = 2k + k^* \leq d$, with $\frac{d - k^*}{2} \geq k\geq \rho^2 k^*/(1-\rho^2)^2$, with $\rho$ defined as below. Suppose that we run SZOHT with random supports of size $\sus$, $q$ random directions,  a learning rate of $\eta = \frac{\nu_s}{(4 \varepsilon_{F} + 1)L_{s'}^2} $, and   $k$ coordinates kept at each iterations. Then, we have a geometric convergence rate, of the following form, with $\x^{(t)}$ denoting the $t$-iterate of SZOHT:
  $$   \E  \|\x^{(t)}-\x^{*}\| \leq \left( \gamma \rho\right)^{t}\|\x^{(0)}-\x^*\|+\left(\frac{\gamma a}{1 - \gamma \rho} \right) \sigma +  \left(\frac{\gamma b}{1 - \gamma \rho}\right)  \mu$$ 
  \begin{equation}
    \label{eq:constantsbis}
    \begin{aligned}
      \text{with}\quad a &= \eta \left(\sqrt{ (4 \varepsilon_{F} s+2) +  \varepsilon_{F^c} (d-k) }  +\sqrt{s}\right), ~ b=\left(\frac{\sqrt{\varepsilon_{\mu}}}{L_{s'}} + \eta \sqrt{2 \varepsilon_{\text{abs}} }\right), \\
      \rho^2&=  1 - \frac{\nu_s^2}{(4  \varepsilon_F+1)L_{s'}^2}, ~\text{and}~  \gamma = \sqrt{1 + \left( k^*/k + \sqrt{\left(4 + k^*/k\right)k^*/k}\right)/2}\\
                   \text{and} ~\quad &\varepsilon_{F}, \varepsilon_{abs}, ~\text{and}~ \varepsilon_{\mu} ~ \text{are defined in \eqref{eq:constants}.}
    \end{aligned}
  \end{equation}
\end{theorem}
\begin{remark}[System error] The format of our result is similar to the ones in \cite{yuan2017gradient} and \cite{nguyen2017linear}, in that it contains a linear convergence term, and a system error which depends on the expected norm of the gradient at $\x^*$ (through the variable $\sigma$). We note that if $f$ has a $k^*$-sparse unconstrained minimizer, which could happen in sparse reconstruction, or with overparameterized deep networks (see for instance \cite[Assumption (2)]{peste2021ac}), then we would have $\|\nabla f(\x^*)\| = 0$, and hence that part of the system error would vanish. In addition to that usual system error, we also have here another system error, which depends on the smoothing radius $\mu$,  due to the error from the ZO estimate.
\end{remark}

\begin{remark}[Generality]\label{rem:generality}
If we take $s_2\leq s$, the first assumption of Theorem \ref{thrm:cvrate} becomes the requirement that $f_{\stoch}$ is $(L_s, s)$-RSS. Therefore, SZOHT as well as the theorem above  provides, up to our knowledge, the first algorithm that can work in the usual setting of hard-thresholding algorithms (that is, $(L_s, s)$-RSS and $(\nu_s, s)$-RSC \cite{nguyen2017linear,shen2017tight}), as well as its convergence rate.
\end{remark}

\textbf{Interplay between hard-thresholding and zeroth-order error}
Importantly, contrary to previous works in ZO optimization, $q$ must be chosen carefully here, due to our specific setting combining ZO and hard-thresholding. Indeed, as described in \cite{shen2017tight}, the hard-thresholding operator is not non-expansive (contrary to projection onto the $\ell_1$ ball) so it can drive the iterates away from the solution. Therefore, enough descent must be made by the (approximate) gradient step to get close enough to the solution, and it is therefore crucial to limit errors in gradient estimation. This problem arises with any kind of gradient errors: for instance with SGD errors \cite{nguyen2017linear,Zhou18}, it is generally dealt with either by ensuring some conditions on the function $f$ \cite{nguyen2017linear}, forming bigger batches of samples \cite{Zhou18}, and/or considering a larger number of components $k$ kept in hard-thresholding (to make the hard-thresholding less expansive). In our work, similarly to \citet{Zhou18}, we deal with this problem by relaxing $k$ and sampling more directions $\uu_i$ (which is the ZO equivalent to taking bigger batch-size in SGD). However, there is an additional effect that happens in our case, specific to ZO estimation: as described in Proposition \ref{prop:zograd}, the quality of our estimator \textit{also depends on $k$}. Therefore, it may be hard to make the algorithm converge only by considering larger $k$: \textit{higher $k$ means less expansivity (which helps convergence), but worse gradient estimate (which harms convergence)}. We further illustrate this conflict between  the non-expansiveness of hard-thresholding (quantified by the parameter $\gamma$ \cite{shen2017tight}), and the error from the zeroth-order estimate, in Figure \ref{fig:conflict}. Therefore, it is even more crucial to tune precisely our remaining degree of freedom at hand which is $q$. More precisely, a minimal value of $q$ is always necessary to ensure convergence in our setting, contrary to most ZO setting (in which taking even $q=1$ can work, as long as other constants like $\eta$ are well chosen, see for instance \cite[Corollary 3]{liuadmm}). The remark below gives some necessary conditions on $q$ to illustrate that fact.

\begin{remark}[Some necessary condition on $q$, proof in \ref{sec:firstcond}]\label{rem:firstcond}
  Let $k^* \in \mathbb{N}^*$ and assume, that $k$ is such that $k > \rho^2 k^*/(1-\rho^2)^2$ (which ensures that $\rho \gamma <1$), and that  $ k\leq \frac{d - k^*}{2}$. These conditions imply the following necessary (but not sufficient) condition on $q$:
  \begin{itemize}
  \item \textbf{if $s_2>1$}:  $ q \geq  \frac{16 d (\sus-1) k^* \kappa^{2}}{\left(\sus+2\right)(d-1)} \left[18 \kappa^2 -1+2 \sqrt{9\kappa^2(9 \kappa^2-1)+\frac{1}{2}-\frac{1}{2 k^{*}}+\frac{3}{2} \frac{d-1}{k^*(\sus-1)}}\right]$
   \item \textbf{if $s_2=1$}: $ q \geq  \frac{8 \kappa^2 d}{\sqrt{\frac{d}{k^*}} + 1}$
   \end{itemize}
 \end{remark}

Remark \ref{rem:firstcond} is just a warning that usual rules from ZO do not apply to SZOHT, but it does not say how to choose $q$ to ensure convergence: for that we would need some sufficient conditions on $q$ for Theorem 1 to apply. We give such conditions in the next section.

\subsection{Weak/non dependence on dimensionality of the query complexity.}\label{ref:indep}

In this section, we provide Corollaries \ref{cor:specialq} and \ref{cor:s2d}, following from Theorem \ref{thrm:cvrate}, which give an example of $q$ that is sufficient to converge (that is, to obtain $\gamma \rho <1$ in Theorem \ref{thrm:cvrate}), and that achieves weak dimensionality dependence in the case of RSS, and complete dimension independence in the case of smoothness.

\begin{corollary}[RSS $f_{\stoch}$, proof in Appendix \ref{sec:proof_specialq}\label{cor:specialq}]
  Assume that that almost all $f_{\stoch}$ are $(L_{s'}, s':=\max(\sus, s))$-RSS, and that $f$ is $(\nu_s, s)$-RSC and $\sigma$-FGN, with $s = 2k + k^* \leq d$, with $\frac{d - k^*}{2} \geq k\geq  (86\kappa^4 - 12\kappa^2 )k^*$ (with  $\kappa:= \frac{L_{s'}}{\nu_s}$) . Suppose that we run SZOHT with random support of size $\sus$, a learning rate of $\eta = \frac{\nu_s}{13L_{s'}^2}$,  with $k$ coordinates kept at each iterations by the hard-thresholding, and with $q\geq 2s + 6\frac{d}{\sus}$.  Then, we have a geometric convergence rate, of the following form, with $\x^{(t)}$ denoting the $t$-iterate of SZOHT:
  $$   \E  \|\x^{(t)}-\x^{*}\| \leq \left( \gamma\rho\right)^{t}\|\x^{(0)}-\x^*\|+\left(\frac{\gamma a}{1 - \gamma \rho} \right)  \sigma +  \left(\frac{\gamma b}{1 - \gamma \rho}\right)  \mu$$

  with $a$, $b$ and $\gamma$ are defined in \eqref{eq:constantsbis}, and $\rho=\sqrt{1 - \frac{2}{13\kappa^2}}$.
    Therefore, the query complexity (QC) to ensure that $  \E  \|\x^{(t)}-\x^{*}\| \leq  \varepsilon  +\left(\frac{\gamma a}{1 - \gamma \rho} \right)  \sigma +  \left(\frac{\gamma b}{1 - \gamma \rho}\right)  \mu $ is $\OO(\kappa^2 (k + \frac{d}{\sus}) \log(\frac{1}{\varepsilon}))$.
\end{corollary}

We now turn to the case where the functions $f_{\stoch}$ are smooth. The key result in that case is that we can have a query complexity independent of the dimension $d$, which is, up to our knowledge, the first result of such kind for sparse zeroth-order optimization without assuming any gradient sparsity.
\begin{corollary}[Smooth $f_{\stoch}$, proof in Appendix \ref{sec:proof_s2d})\label{cor:s2d}]
  Assume that, in addition to the conditions from Corollary \ref{cor:specialq} above, almost all $f_{\stoch}$ are $L$-smooth, with  $\frac{d - k^*}{2} \geq k\geq  (86\kappa^4 - 12\kappa^2 )k^*$ (with  $\kappa:= \frac{L}{\nu_s}$), and take $q\geq 2(s+2)$, and $s_2=d$ (that is, no random support sampling).  Then, we have a geometric convergence rate, of the following form, with $\x^{(t)}$ denoting the $t$-iterate of SZOHT:
  $$   \E  \|\x^{(t)}-\x^{*}\| \leq \left( \gamma\rho\right)^{t}\|\x^{(0)}-\x^*\|+\left(\frac{\gamma a}{1 - \gamma \rho} \right)  \sigma +  \left(\frac{\gamma b}{1 - \gamma \rho}\right)  \mu$$
  Therefore, the QC to ensure that $  \E  \|\x^{(t)}-\x^{*}\| \leq   \varepsilon  +\left(\frac{\gamma a}{1 - \gamma \rho} \right)  \sigma +  \left(\frac{\gamma b}{1 - \gamma \rho}\right)  \mu$ is $\OO(\kappa^2 k \log(\frac{1}{\varepsilon}))$.
\end{corollary}
Additionally, our convergence rate highlights an interesting connection between the geometry of $f$ (defined by the condition number $\kappa=L_{s'}/\nu_s$), and the number of random directions that we need to take at each iteration: if the problem is ill-conditioned, that is $\kappa$ is high, then we need a bigger $k$. This result is standard in the $\ell_0$ litterature (see for instance \cite{yuan2017gradient}). But specifically, in our ZO case, it also impacts the query complexity: since the projected gradient is harder to approximate when the dimension $k$ of the projection is larger, $q$ needs to grow too, resulting in higher query complexity. We believe this is an interesting result for the sparse zeroth-order optimization community: it reveals that the query complexity may in fact depend on some notion of intrinsic dimension to the problem, related to both the sparsity of the iterates $k$, and the geometry of the function $f$ for a given $\sus$ (through the restricted condition number $\kappa$), rather than  the dimension of the original space $d$ as in previous works like \cite{ghadimi2016mini}. 
\section{Experiments}
\subsection{Sensitivity analysis}
We first conduct a sensitivity parameter analysis on a toy example, to highlight the importance of the choice of $q$, as discussed in Section \ref{sec:glob}. We fix a target sparsity $k^*=5$, choose $k=74 k^*$, and consider a sparse quadric function $f: \mathbb{R}^{5000} \rightarrow \mathbb{R}$, with: $f(\x)= \frac{1}{2}\|\bm{a} \odot (\x-\bm{b})\|^2$ ($\odot$ denotes the elementwise product), with $\bm{a}_i=1$ if $i \geq d-s$ and $0$ otherwise (to ensure $f$ is $s$-RSC and smooth, with $\nu_s=L=1$), and $\bm{b}_i = \frac{i}{100 d}$ for all $i \in [d - 70 k^*]$ and $0$ for all $d - 70 k^* \leq i \leq d$ (we make such a choice in order to have  $\|\nabla f(\x^*)\|$ small enough). We choose $\eta$ as in Theorem \ref{thrm:cvrate}: $\eta = \frac{1}{(4 \varepsilon_F + 1)}$ with $\varepsilon_F$ defined in Proposition \ref{prop:zograd} in terms of $s$ and $d$ (we take $\sus=d$), $\mu=1e-4$, and present our results in Figure \ref{fig:sensas}, for six values of $q$. We can observe on Figure \ref{fig:sensas_f} that the smaller the $q$, the less $f(\x)$ can descend. Interestingly, we can also see on Figure \ref{fig:sensas_dist} that for $q=1$ and $20$, $\|\x^{(t)} - \x^*\|$ diverges: we can indeed compute that $\rho\gamma>1$ for those $q$, which explains the divergence, from Theorem \ref{thrm:cvrate}.
\label{sec:experiments}
\subsection{Baselines}
\label{sec:experimental-setting}
We compare our SZOHT algorithms with state of the art zeroth-order algorithms that can deal with sparsity constraints, that appear in Table \ref{table:sota}:
\begin{enumerate}[leftmargin=0.2in]
\item \textbf{ZSCG}  \cite{Balasubramanian18} is a Frank-Wolfe ZO algorithm, for which we consider an $\ell_1$ ball constraint.
\item \textbf{RSPGF} \cite{ghadimi2016mini} is a proximal ZO algorithm, for which we consider an $\ell_1$ penalty.
  \item \textbf{ZORO} \cite{Cai20} is a proximal ZO algorithm, that makes use of sparsity of gradients assumptions, using a sparse reconstruction algorithm at each iteration to reconstruct the gradient from a few measurements. Similarly, as for ZSCG, we consider an $\ell_1$ penalty.
  \end{enumerate}
In all the applications below, we will tune the sparsity $k$ of SZOHT, the penalty of RSPGF and ZORO, and the radius of the constraint of ZSCG, such that all algorithms attain a similar converged objective value, for fair comparison.
\subsection{Applications}
\label{sec:applications}
We compare the algorithms above on two tasks: a sparse asset risk management task from \cite{chang2000heuristics}, and an adversarial attack task \cite{chen2017zoo} with a sparsity constraint. 

\paragraph{Sparse asset risk management}
\label{sec:portf-manag-task}
We consider the portfolio management task and dataset from \cite{chang2000heuristics}, similarly to \cite{Cai20}. We have a given portfolio of $d$ assets, with each asset $i$ giving an expected return $\bm{m}_i$, and with a global covariance matrix of the return of assets denoted as $\bm{C}$. The cost function we minimize is the portfolio risk: $\frac{\x^T\bm{C}\x}{2(\sum_{i=1}^d \bm{x}_i)^2}$, where $\x$  is a vector where each component $\x_i$ denotes how much is invested in each asset, and we require to minimize it under a constraint of minimal return $r$: $\frac{\sum_{i=1}^{d} \bm{m}_{i} \x_{i}}{\sum_{i=1}^{d} \x_{i}}$. We enforce that constraint using the Lagrangian form below. Finally, we add a sparsity constraint, to restrict the investments to only $k$ assets. Therefore, we obtain the cost function below:
\begin{align*}
\min_{x \in \mathbb{R}^{d}} \frac{\x^{\top} \bm{C} \x}{2\left(\sum_{i=1}^{d} \x_{i}\right)^{2}}+\lambda\left(\min \left\{\frac{\sum_{i=1}^{d} \bm{m}_{i} \x_{i}}{\sum_{i=1}^{d} \x_{i}}-r, 0\right\}\right)^{2} \quad\text{s.t.} \quad \|\x\|_0 \leq k
\end{align*}
We use three datasets: port3, port4 and port5 from the OR-library \cite{beasley1990or}, of respective dimensions $d=89; 98; 225$. We keep $r$ and $\lambda$ the same for the 4 algorithms: $r=0.1$, $\lambda=10$ (for port3 and  port4); and $r=1e-3$, $\lambda=1e-3$ for port5. For SZOHT, we set $k=10$, $\sus=10$, $q=10$, and $(\mu, \eta)=(0.015, 0.015)$ for port4, and $(\mu, \eta)=(0.1, 1)$ for port5 ($\mu$ and $\eta$ are both obtained by grid search over the interval $[10^{-3}, 10^{3}]$). For all other algorithms, we got the optimal hyper-parameters through grid search. We present our results in Figure \ref{fig:sota_ass}.
\paragraph{Few pixels adversarial attacks}
\label{sec:zeroth-order-few}
We consider the problem of adversarial attacks with a sparse constraint. Our goal is to minimize $\min_{\delta} f(\x+\bm{\delta})$ such that $\|\bm{\delta}\|_0 \leq k$, where $f$ is the Carlini-Wagner cost function \cite{chen2017zoo}, that is computed from the outputs of a pre-trained model on the corresponding dataset. We consider three different datasets for the attacks: MNIST, CIFAR, and Imagenet, of dimension respectively $d=784; 3072; 268203$. All algorithms are initialized with $\bm{\delta}=\bm{0}$. We set the hyperparameters of SZOHT as follows: MNIST: $k=20$, $\sus=100$, $q=100$, $\mu=0.3$, $\eta=1$; CIFAR: $k=60$, $\sus=100$, $q=1000$, $\mu=1e-3$, $\eta=0.01$; ImageNet: $k=100000$, $\sus=1000$, $q=100$, $\mu=0.01$, $\eta=0.015$. We present our results in Figure \ref{fig:sota_adv}.
All experiments are conducted in the workstation with four NVIDIA RTX A6000 GPUs, and take about one day to run.

\subsection{Results and Discussion}
\label{sec:results-discussion}
We can observe from Figures \ref{fig:sota_ass} and \ref{fig:sota_adv} that the performance of SZOHT is comparable or better than the other algorithms. This can be explained by the fact that SZOHT has a linear convergence, but the query complexity of ZSCG and RSPGF is in $\OO(1/T)$. We can also notice that RSPGF is faster than ZSCG, which is natural since proximal algorithms are faster than Frank-Wolfe algorithms (indeed, in case of possible strong-convexity, vanilla Frank-Wolfe algorithms maintain a $\mathcal{O}(1/T)$ rate \cite{garber2015faster}, when proximal algorithms get a linear rate \cite[Theorem 10.29]{beck2017first}). Finally, it appears that the convergence of ZORO is sometimes slower, particularly at the early stage of training, which may come from the fact that ZORO assumes sparse gradients, which is not necessarily verified in real-world use cases like the ones we consider; in those cases where the gradient is not sparse, it is possible that the sparse gradient reconstruction step of ZORO does not work well. This motivates even further the need to consider algorithms able to work without those assumptions, such as SZOHT.
\begin{figure}[htp]
  \vspace{-5pt}
  \centering
  \begin{minipage}{0.29\textwidth}
    \centering
       \subfigure[$f(\x)$]{\includegraphics[scale=0.20]{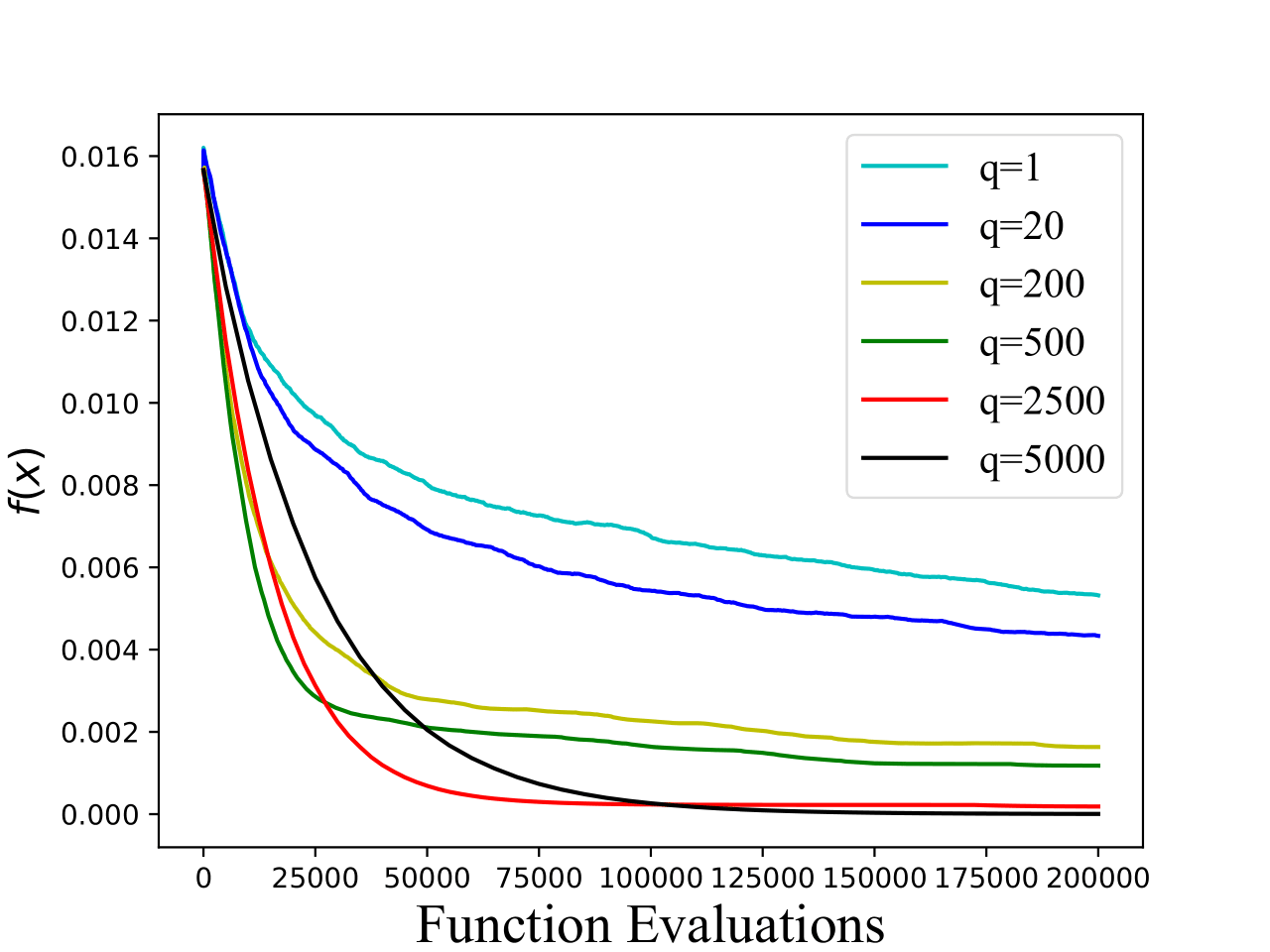}\label{fig:sensas_dist}}
   \subfigure[$\|\x - \x^*\|$]{\includegraphics[scale=0.21]{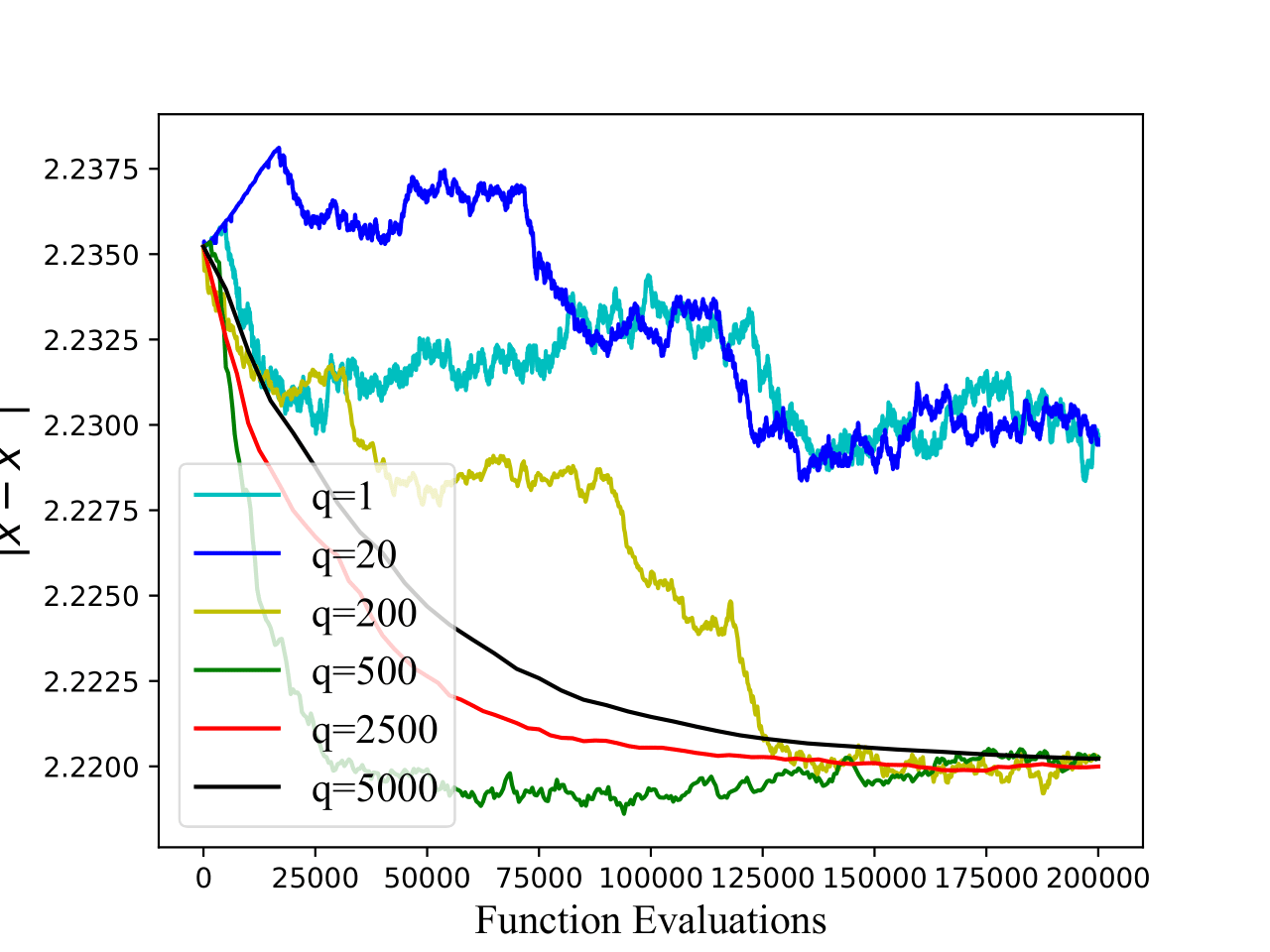}\label{fig:sensas_f}}
   \caption{Sensitivity analysis}
   \label{fig:sensas}
  \end{minipage}
  \hfill
  \begin{minipage}{0.70\textwidth}
    \centering
     \subfigure[port3]{\includegraphics[scale=0.195]{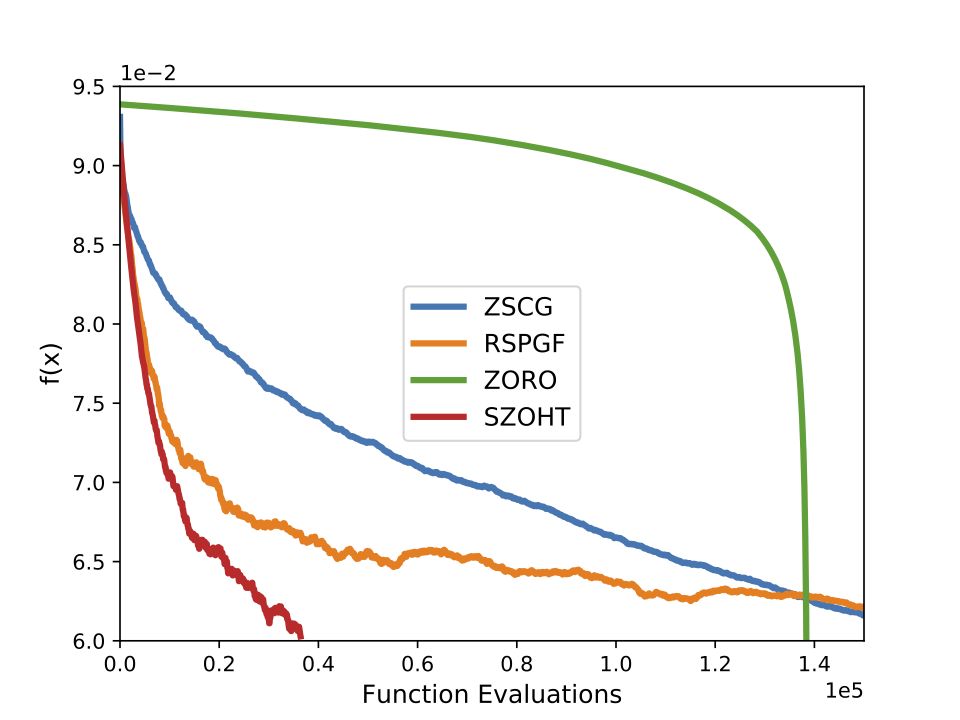}\label{fig:asset3_sota}}
   \subfigure[port4]{\includegraphics[scale=0.195]{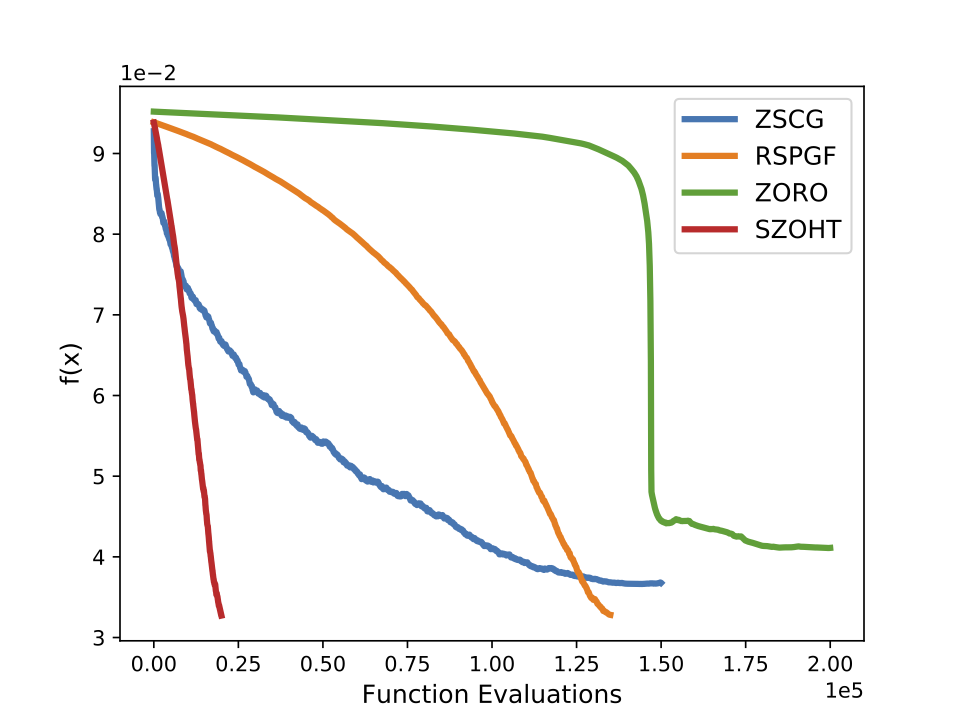}\label{fig:asset4_sota}}
     \subfigure[port5]{\includegraphics[scale=0.195]{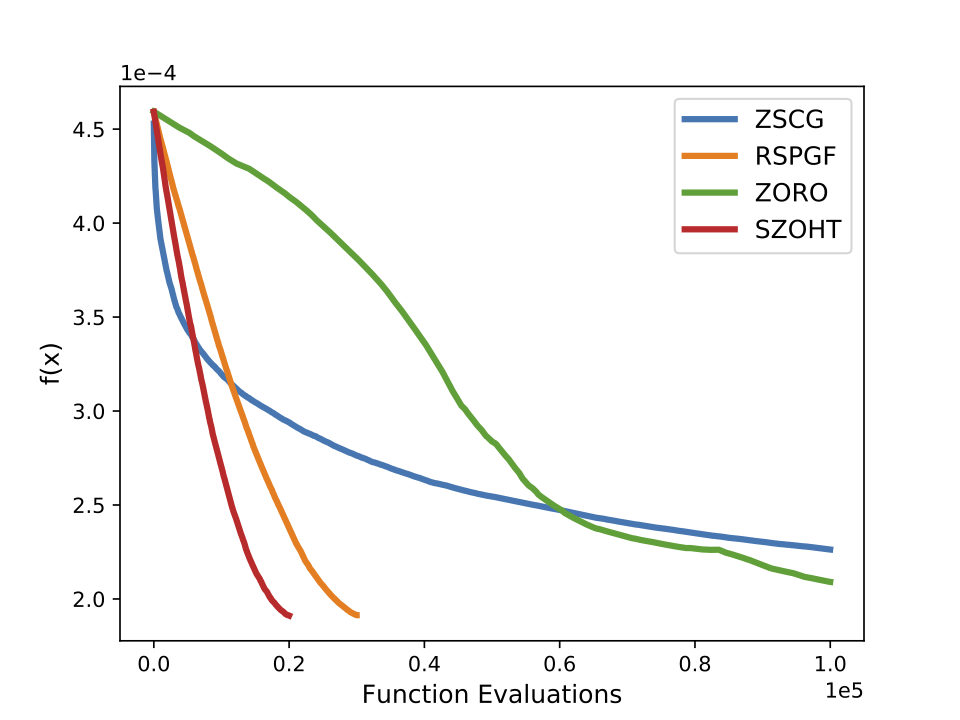}\label{fig:asset5_sota}}
  \caption{$f(\x)$ vs. \# queries (asset management)}
   \label{fig:sota_ass}
  \subfigure[MNIST]{\includegraphics[scale=0.195]{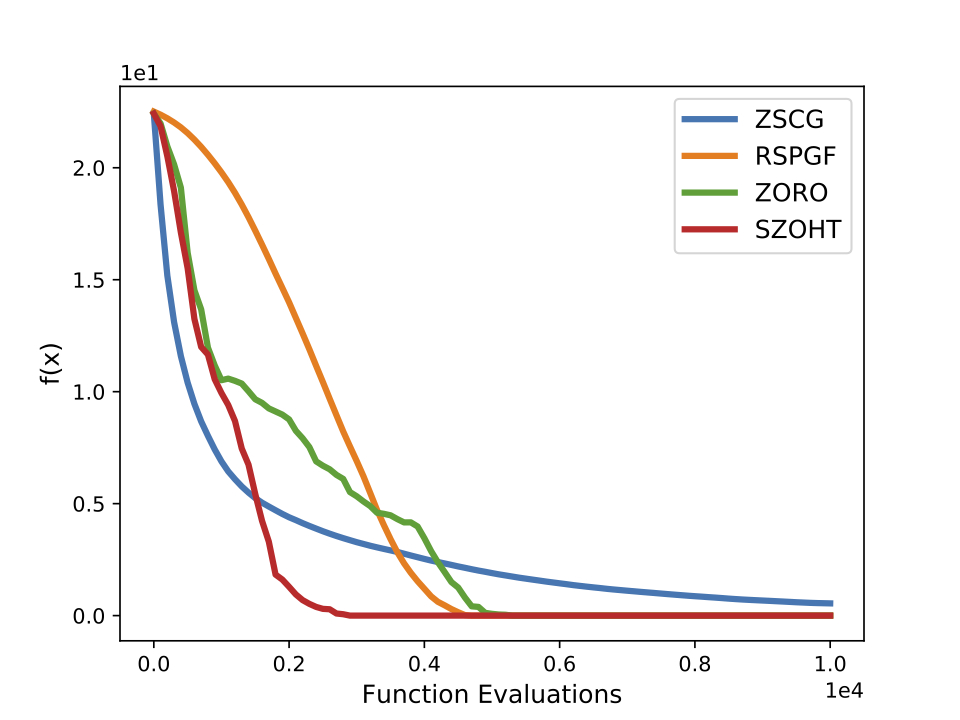}\label{fig:MNIST}}
   \subfigure[CIFAR]{\includegraphics[scale=0.195]{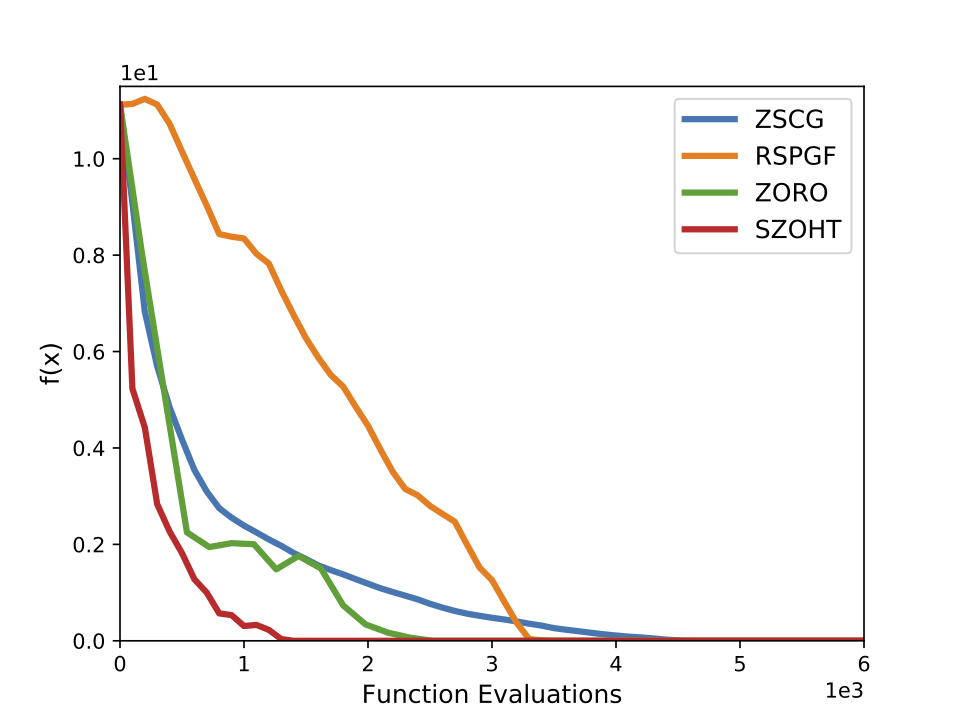}\label{fig:CIFAR}}
        \subfigure[Imagenet]{\includegraphics[scale=0.195]{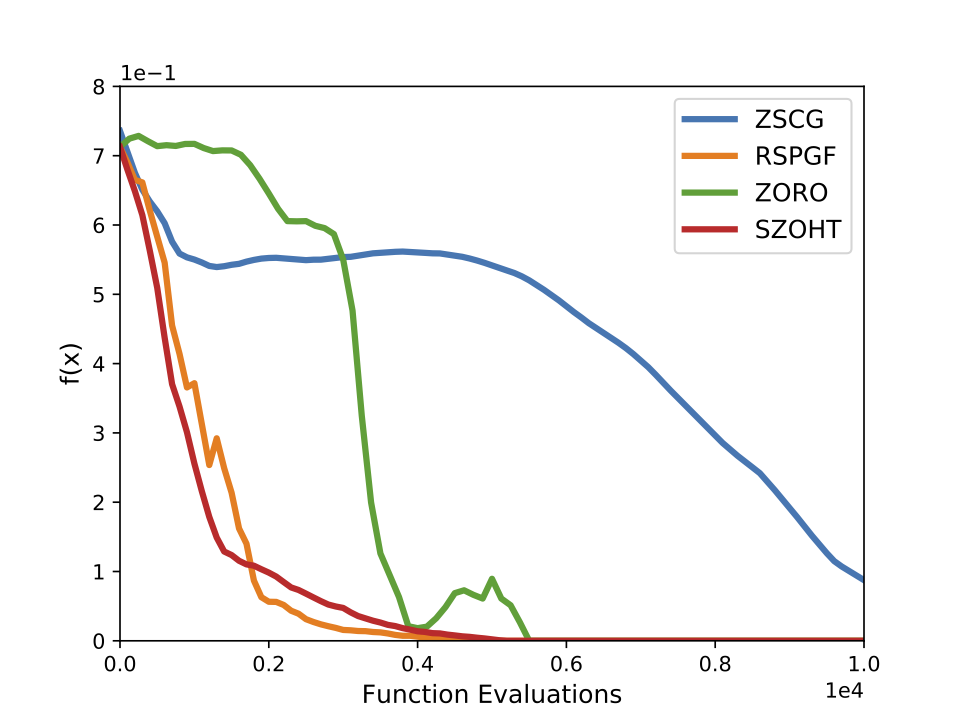}\label{fig:Imagenet}}
        \caption{$f(\x)$ vs. \# queries  (adversarial attack)}
        \label{fig:sota_adv}
  \end{minipage}
  \vspace{-5pt}
 \end{figure}
\section{Conclusion}
In this paper, we proposed a new algorithm, SZOHT, for sparse zeroth-order optimization. We gave its convergence analysis and showed that it is dimension independent in the smooth case, and weak dimension-dependent in the RSS case. We further verified experimentally the efficiency of SZOHT in several settings. Moreover, throughout the paper, we showed how the condition number of $f$ as well as the gradient error have an important impact on the convergence of SZOHT. As such, it would be interesting to study whether we can improve the query complexity by regularizing $f$, by using an adaptive learning rate or acceleration methods, or by using recent variance reduction techniques. Finally, it would also be interesting to extend this work to a broader family of sparse structures, such as low-rank approximations or graph sparsity. We leave this for future work.

\begin{ack}
  Xiao-Tong Yuan is funded in part by the National Key Research and Development Program of China under Grant No. 2018AAA0100400, and in part by the Natural Science Foundation of China (NSFC) under Grant No.U21B2049, No.61876090 and No.61936005.
  \end{ack}

\bibliography{bibli}

\begin{thebibliography}{47}
\providecommand{\natexlab}[1]{#1}
\providecommand{\url}[1]{\texttt{#1}}
\expandafter\ifx\csname urlstyle\endcsname\relax
  \providecommand{\doi}[1]{doi: #1}\else
  \providecommand{\doi}{doi: \begingroup \urlstyle{rm}\Url}\fi

\bibitem[Arfken and Weber(1999)]{arfken1999mathematical}
George~B Arfken and Hans~J Weber.
\newblock \emph{Mathematical methods for physicists}.
\newblock American Association of Physics Teachers, 1999.

\bibitem[Balasubramanian and Ghadimi(2018)]{Balasubramanian18}
Krishnakumar Balasubramanian and Saeed Ghadimi.
\newblock Zeroth-order (non)-convex stochastic optimization via conditional
  gradient and gradient updates.
\newblock In \emph{Advances in Neural Information Processing Systems},
  volume~31, 2018.

\bibitem[Beasley(1990)]{beasley1990or}
John~E Beasley.
\newblock Or-library: distributing test problems by electronic mail.
\newblock \emph{Journal of the operational research society}, 41\penalty0
  (11):\penalty0 1069--1072, 1990.

\bibitem[Beck(2017)]{beck2017first}
Amir Beck.
\newblock \emph{First-order methods in optimization}.
\newblock SIAM, 2017.

\bibitem[B{\"u}hlmann and Van De~Geer(2011)]{buhlmann2011statistics}
Peter B{\"u}hlmann and Sara Van De~Geer.
\newblock \emph{Statistics for high-dimensional data: methods, theory and
  applications}.
\newblock Springer Science \& Business Media, 2011.

\bibitem[Cai et~al.(2021)Cai, Lou, McKenzie, and Yin]{cai2021zeroth}
HanQin Cai, Yuchen Lou, Daniel McKenzie, and Wotao Yin.
\newblock A zeroth-order block coordinate descent algorithm for huge-scale
  black-box optimization.
\newblock In \emph{International Conference on Machine Learning}, pages
  1193--1203. PMLR, 2021.

\bibitem[Cai et~al.(2022)Cai, McKenzie, Yin, and Zhang]{Cai20}
HanQin Cai, Daniel McKenzie, Wotao Yin, and Zhenliang Zhang.
\newblock Zeroth-order regularized optimization (zoro): Approximately sparse
  gradients and adaptive sampling.
\newblock \emph{SIAM Journal on Optimization}, 32\penalty0 (2):\penalty0
  687--714, 2022.

\bibitem[Chang et~al.(2000)Chang, Meade, Beasley, and
  Sharaiha]{chang2000heuristics}
T-J Chang, Nigel Meade, John~E Beasley, and Yazid~M Sharaiha.
\newblock Heuristics for cardinality constrained portfolio optimisation.
\newblock \emph{Computers \& Operations Research}, 27\penalty0 (13):\penalty0
  1271--1302, 2000.

\bibitem[Chen et~al.(2017)Chen, Zhang, Sharma, Yi, and Hsieh]{chen2017zoo}
Pin-Yu Chen, Huan Zhang, Yash Sharma, Jinfeng Yi, and Cho-Jui Hsieh.
\newblock Zoo: Zeroth order optimization based black-box attacks to deep neural
  networks without training substitute models.
\newblock In \emph{Proceedings of the 10th ACM workshop on artificial
  intelligence and security}, pages 15--26, 2017.

\bibitem[Chen et~al.(2019)Chen, Liu, Xu, Li, Lin, Hong, and Cox]{chen2019zo}
Xiangyi Chen, Sijia Liu, Kaidi Xu, Xingguo Li, Xue Lin, Mingyi Hong, and David
  Cox.
\newblock Zo-adamm: Zeroth-order adaptive momentum method for black-box
  optimization.
\newblock In \emph{Advances in Neural Information Processing Systems},
  volume~32, 2019.

\bibitem[Choromanski et~al.(2020)Choromanski, Pacchiano, Parker-Holder, Tang,
  Jain, Yang, Iscen, Hsu, and Sindhwani]{choromanski2020provably}
Krzysztof Choromanski, Aldo Pacchiano, Jack Parker-Holder, Yunhao Tang, Deepali
  Jain, Yuxiang Yang, Atil Iscen, Jasmine Hsu, and Vikas Sindhwani.
\newblock Provably robust blackbox optimization for reinforcement learning.
\newblock In \emph{Conference on Robot Learning}, pages 683--696. PMLR, 2020.

\bibitem[Gao et~al.(2018)Gao, Jiang, and Zhang]{gao2018information}
Xiang Gao, Bo~Jiang, and Shuzhong Zhang.
\newblock On the information-adaptive variants of the admm: an iteration
  complexity perspective.
\newblock \emph{Journal of Scientific Computing}, 76\penalty0 (1):\penalty0
  327--363, 2018.

\bibitem[Garber and Hazan(2015)]{garber2015faster}
Dan Garber and Elad Hazan.
\newblock Faster rates for the frank-wolfe method over strongly-convex sets.
\newblock In \emph{International Conference on Machine Learning}, pages
  541--549. PMLR, 2015.

\bibitem[Ghadimi et~al.(2016)Ghadimi, Lan, and Zhang]{ghadimi2016mini}
Saeed Ghadimi, Guanghui Lan, and Hongchao Zhang.
\newblock Mini-batch stochastic approximation methods for nonconvex stochastic
  composite optimization.
\newblock \emph{Mathematical Programming}, 155\penalty0 (1):\penalty0 267--305,
  2016.

\bibitem[Golovin et~al.(2019)Golovin, Karro, Kochanski, Lee, Song, and
  Zhang]{Golovin19}
Daniel Golovin, John Karro, Greg Kochanski, Chansoo Lee, Xingyou Song, and
  Qiuyi Zhang.
\newblock Gradientless descent: High-dimensional zeroth-order optimization.
\newblock In \emph{International Conference on Learning Representations}, 2019.

\bibitem[Gower et~al.(2019)Gower, Loizou, Qian, Sailanbayev, Shulgin, and
  Richt{\'a}rik]{Gower19}
Robert~Mansel Gower, Nicolas Loizou, Xun Qian, Alibek Sailanbayev, Egor
  Shulgin, and Peter Richt{\'a}rik.
\newblock {SGD}: General analysis and improved rates.
\newblock In \emph{Proceedings of the 36th International Conference on Machine
  Learning}, volume~97, pages 5200--5209. PMLR, 2019.

\bibitem[Jain et~al.(2014)Jain, Tewari, and Kar]{Jain14}
Prateek Jain, Ambuj Tewari, and Purushottam Kar.
\newblock On iterative hard thresholding methods for high-dimensional
  m-estimation.
\newblock In \emph{Advances in Neural Information Processing Systems},
  volume~27, 2014.

\bibitem[Jamieson et~al.(2012)Jamieson, Nowak, and Recht]{Jamieson12}
Kevin~G Jamieson, Robert Nowak, and Ben Recht.
\newblock Query complexity of derivative-free optimization.
\newblock In \emph{Advances in Neural Information Processing Systems},
  volume~25, 2012.

\bibitem[Levitin and Polyak(1966)]{levitin1966constrained}
Evgeny~S Levitin and Boris~T Polyak.
\newblock Constrained minimization methods.
\newblock \emph{USSR Computational mathematics and mathematical physics},
  6\penalty0 (5):\penalty0 1--50, 1966.

\bibitem[Li et~al.(2016)Li, Arora, Liu, Haupt, and Zhao]{li2016nonconvex}
Xingguo Li, Raman Arora, Han Liu, Jarvis Haupt, and Tuo Zhao.
\newblock Nonconvex sparse learning via stochastic optimization with
  progressive variance reduction.
\newblock \emph{arXiv preprint arXiv:1605.02711}, 2016.

\bibitem[Lian et~al.(2016)Lian, Zhang, Hsieh, Huang, and
  Liu]{lian2016comprehensive}
Xiangru Lian, Huan Zhang, Cho-Jui Hsieh, Yijun Huang, and Ji~Liu.
\newblock A comprehensive linear speedup analysis for asynchronous stochastic
  parallel optimization from zeroth-order to first-order.
\newblock \emph{Advances in Neural Information Processing Systems}, 29, 2016.

\bibitem[Liu and Yang(2021)]{Liu21}
Hongcheng Liu and Yu~Yang.
\newblock A dimension-insensitive algorithm for stochastic zeroth-order
  optimization.
\newblock \emph{arXiv preprint arXiv:2104.11283}, 2021.

\bibitem[Liu et~al.(2018{\natexlab{a}})Liu, Chen, Chen, and Hero]{liuadmm}
Sijia Liu, Jie Chen, Pin-Yu Chen, and Alfred Hero.
\newblock Zeroth-order online alternating direction method of multipliers:
  Convergence analysis and applications.
\newblock In \emph{International Conference on Artificial Intelligence and
  Statistics}, pages 288--297. PMLR, 2018{\natexlab{a}}.

\bibitem[Liu et~al.(2018{\natexlab{b}})Liu, Kailkhura, Chen, Ting, Chang, and
  Amini]{liu2018zeroth}
Sijia Liu, Bhavya Kailkhura, Pin-Yu Chen, Paishun Ting, Shiyu Chang, and Lisa
  Amini.
\newblock Zeroth-order stochastic variance reduction for nonconvex
  optimization.
\newblock In \emph{Advances in Neural Information Processing Systems},
  volume~31, 2018{\natexlab{b}}.

\bibitem[Liu et~al.(2020)Liu, Chen, Kailkhura, Zhang, Hero~III, and
  Varshney]{Liu20}
Sijia Liu, Pin-Yu Chen, Bhavya Kailkhura, Gaoyuan Zhang, Alfred~O Hero~III, and
  Pramod~K Varshney.
\newblock A primer on zeroth-order optimization in signal processing and
  machine learning: Principals, recent advances, and applications.
\newblock \emph{IEEE Signal Processing Magazine}, 37\penalty0 (5):\penalty0
  43--54, 2020.

\bibitem[Loh and Wainwright(2013)]{loh2013regularized}
Po-Ling Loh and Martin~J Wainwright.
\newblock Regularized m-estimators with nonconvexity: Statistical and
  algorithmic theory for local optima.
\newblock \emph{Advances in Neural Information Processing Systems}, 26, 2013.

\bibitem[Mania et~al.(2018)Mania, Guy, and Recht]{mania2018simple}
Horia Mania, Aurelia Guy, and Benjamin Recht.
\newblock Simple random search of static linear policies is competitive for
  reinforcement learning.
\newblock In \emph{Advances in Neural Information Processing Systems},
  volume~31, 2018.

\bibitem[Negahban et~al.(2009)Negahban, Yu, Wainwright, and
  Ravikumar]{negahban2009unified}
Sahand Negahban, Bin Yu, Martin~J Wainwright, and Pradeep Ravikumar.
\newblock A unified framework for high-dimensional analysis of $ m $-estimators
  with decomposable regularizers.
\newblock \emph{Advances in neural information processing systems}, 22, 2009.

\bibitem[Negahban et~al.(2012)Negahban, Ravikumar, Wainwright, and
  Yu]{negahban2012unified}
Sahand~N Negahban, Pradeep Ravikumar, Martin~J Wainwright, and Bin Yu.
\newblock A unified framework for high-dimensional analysis of $ m $-estimators
  with decomposable regularizers.
\newblock \emph{Statistical science}, 27\penalty0 (4):\penalty0 538--557, 2012.

\bibitem[Nesterov and Spokoiny(2017)]{Nesterov17}
Yurii Nesterov and Vladimir Spokoiny.
\newblock Random gradient-free minimization of convex functions.
\newblock \emph{Foundations of Computational Mathematics}, 17\penalty0
  (2):\penalty0 527--566, 2017.

\bibitem[Nguyen et~al.(2017)Nguyen, Needell, and Woolf]{nguyen2017linear}
Nam Nguyen, Deanna Needell, and Tina Woolf.
\newblock Linear convergence of stochastic iterative greedy algorithms with
  sparse constraints.
\newblock \emph{IEEE Transactions on Information Theory}, 63\penalty0
  (11):\penalty0 6869--6895, 2017.

\bibitem[Peste et~al.(2021)Peste, Iofinova, Vladu, and Alistarh]{peste2021ac}
Alexandra Peste, Eugenia Iofinova, Adrian Vladu, and Dan Alistarh.
\newblock Ac/dc: Alternating compressed/decompressed training of deep neural
  networks.
\newblock \emph{Advances in Neural Information Processing Systems}, 34, 2021.

\bibitem[Raskutti et~al.(2011)Raskutti, Wainwright, and
  Yu]{raskutti2011minimax}
Garvesh Raskutti, Martin~J Wainwright, and Bin Yu.
\newblock Minimax rates of estimation for high-dimensional linear regression
  over $\ell_q$-balls.
\newblock \emph{IEEE transactions on information theory}, 57\penalty0
  (10):\penalty0 6976--6994, 2011.

\bibitem[Salimans et~al.(2017)Salimans, Ho, Chen, Sidor, and
  Sutskever]{salimans2017evolution}
Tim Salimans, Jonathan Ho, Xi~Chen, Szymon Sidor, and Ilya Sutskever.
\newblock Evolution strategies as a scalable alternative to reinforcement
  learning.
\newblock \emph{arXiv preprint arXiv:1703.03864}, 2017.

\bibitem[Shamir(2017)]{Shamir17}
Ohad Shamir.
\newblock An optimal algorithm for bandit and zero-order convex optimization
  with two-point feedback.
\newblock \emph{The Journal of Machine Learning Research}, 18\penalty0
  (1):\penalty0 1703--1713, 2017.

\bibitem[Shen and Li(2017)]{shen2017tight}
Jie Shen and Ping Li.
\newblock A tight bound of hard thresholding.
\newblock \emph{The Journal of Machine Learning Research}, 18\penalty0
  (1):\penalty0 7650--7691, 2017.

\bibitem[Sokolov et~al.(2018)Sokolov, Hitschler, Ohta, and
  Riezler]{sokolov2018sparse}
Artem Sokolov, Julian Hitschler, Mayumi Ohta, and Stefan Riezler.
\newblock Sparse stochastic zeroth-order optimization with an application to
  bandit structured prediction.
\newblock \emph{arXiv preprint arXiv:1806.04458}, 2018.

\bibitem[Sykora(2005)]{Sykora2005}
Stanislav Sykora.
\newblock Surface integrals over n-dimensional spheres.
\newblock \emph{Stan’s Library}, \penalty0 (Volume I), May 2005.
\newblock \doi{10.3247/sl1math05.002}.

\bibitem[Tibshirani(1996)]{tibshirani1996regression}
Robert Tibshirani.
\newblock Regression shrinkage and selection via the lasso.
\newblock \emph{Journal of the Royal Statistical Society: Series B
  (Methodological)}, 58\penalty0 (1):\penalty0 267--288, 1996.

\bibitem[Tu et~al.(2019)Tu, Ting, Chen, Liu, Zhang, Yi, Hsieh, and Cheng]{Tu19}
Chun-Chen Tu, Paishun Ting, Pin-Yu Chen, Sijia Liu, Huan Zhang, Jinfeng Yi,
  Cho-Jui Hsieh, and Shin-Ming Cheng.
\newblock Autozoom: Autoencoder-based zeroth order optimization method for
  attacking black-box neural networks.
\newblock In \emph{Proceedings of the AAAI Conference on Artificial
  Intelligence}, volume~33, pages 742--749, 2019.

\bibitem[Van~de Geer(2008)]{van2008high}
Sara~A Van~de Geer.
\newblock High-dimensional generalized linear models and the lasso.
\newblock \emph{The Annals of Statistics}, 36\penalty0 (2):\penalty0 614--645,
  2008.

\bibitem[Wainwright et~al.(2008)Wainwright, Jordan,
  et~al.]{wainwright2008graphical}
Martin~J Wainwright, Michael~I Jordan, et~al.
\newblock Graphical models, exponential families, and variational inference.
\newblock \emph{Foundations and Trends{\textregistered} in Machine Learning},
  1\penalty0 (1--2):\penalty0 1--305, 2008.

\bibitem[Walck et~al.(2007)]{walck2007hand}
Christian Walck et~al.
\newblock Hand-book on statistical distributions for experimentalists.
\newblock \emph{University of Stockholm}, 10:\penalty0 96--01, 2007.

\bibitem[Wang et~al.(2018)Wang, Du, Balakrishnan, and Singh]{Wang18}
Yining Wang, Simon Du, Sivaraman Balakrishnan, and Aarti Singh.
\newblock Stochastic zeroth-order optimization in high dimensions.
\newblock In \emph{International Conference on Artificial Intelligence and
  Statistics}, pages 1356--1365. PMLR, 2018.

\bibitem[Yuan et~al.(2017)Yuan, Li, and Zhang]{yuan2017gradient}
Xiao-Tong Yuan, Ping Li, and Tong Zhang.
\newblock Gradient hard thresholding pursuit.
\newblock \emph{Journal of Machine Learning Research}, 18\penalty0
  (1):\penalty0 6027--6069, 2017.

\bibitem[Yuan and Li(2021)]{yuan2021stability}
Xiaotong Yuan and Ping Li.
\newblock Stability and risk bounds of iterative hard thresholding.
\newblock In \emph{International Conference on Artificial Intelligence and
  Statistics}, pages 1702--1710. PMLR, 2021.

\bibitem[Zhou et~al.(2018)Zhou, Yuan, and Feng]{Zhou18}
Pan Zhou, Xiaotong Yuan, and Jiashi Feng.
\newblock Efficient stochastic gradient hard thresholding.
\newblock In \emph{Advances in Neural Information Processing Systems},
  volume~31, 2018.

\end{thebibliography}
\bibliographystyle{plainnat}

\section*{Checklist}

\begin{enumerate}

\item For all authors...
\begin{enumerate}
  \item Do the main claims made in the abstract and introduction accurately reflect the paper's contributions and scope?
    \answerYes{}
  \item Did you describe the limitations of your work?
    \answerYes{In the conclusion we mention some possible ways to improve upon our work}
  \item Did you discuss any potential negative societal impacts of your work?
    \answerNA{}
  \item Have you read the ethics review guidelines and ensured that your paper conforms to them?
    \answerYes{}
\end{enumerate}

\item If you are including theoretical results...
\begin{enumerate}
  \item Did you state the full set of assumptions of all theoretical results?
    \answerYes{See Section~\ref{sec:prelim}}
        \item Did you include complete proofs of all theoretical results?
    \answerYes{See Appendix}
\end{enumerate}

\item If you ran experiments...
\begin{enumerate}
  \item Did you include the code, data, and instructions needed to reproduce the main experimental results (either in the supplemental material or as a URL)?
    \answerYes{Our code will be included in the supplementary material}
  \item Did you specify all the training details (e.g., data splits, hyperparameters, how they were chosen)?
    \answerYes{}
        \item Did you report error bars (e.g., with respect to the random seed after running experiments multiple times)?
    \answerNo{}
        \item Did you include the total amount of compute and the type of resources used (e.g., type of GPUs, internal cluster, or cloud provider)?
    \answerYes{}
\end{enumerate}

\item If you are using existing assets (e.g., code, data, models) or curating/releasing new assets...
\begin{enumerate}
  \item If your work uses existing assets, did you cite the creators?
    \answerYes{}
  \item Did you mention the license of the assets?
    \answerYes{}
  \item Did you include any new assets either in the supplemental material or as a URL?
    \answerYes{The code will be submitted in the supplementary material.}
  \item Did you discuss whether and how consent was obtained from people whose data you're using/curating?
    \answerNA{}
  \item Did you discuss whether the data you are using/curating contains personally identifiable information or offensive content?
    \answerNA{}
\end{enumerate}

\item If you used crowdsourcing or conducted research with human subjects...
\begin{enumerate}
  \item Did you include the full text of instructions given to participants and screenshots, if applicable?
    \answerNA{}
  \item Did you describe any potential participant risks, with links to Institutional Review Board (IRB) approvals, if applicable?
    \answerNA{}
  \item Did you include the estimated hourly wage paid to participants and the total amount spent on participant compensation?
    \answerNA{}
\end{enumerate}

\end{enumerate}

\newpage

\appendix
\part*{Supplementary material}

\section{Notations and Definitions}
\label{sec:notations}

Throughout this appendix, we will use the following notations:

\begin{itemize}
\item we denote the vectors in bold letters.
\item $\nabla f(\x)$ denotes the gradient of $f$ at $\x$.
\item $[d]$ denotes the set of all integers between $1$ and $d$: $\{1, .., d\}$.
\item $\uu_i$ denotes the $i$-th coordinate of vector $\uu$, and $\nabla_i f(\x)$ the $i$-th coordinate of $\nabla f(\x)$.
\item $\|\cdot{} \|_0$ denotes the $\ell_0$ norm (which is not a proper norm).
\item  $\|\cdot{} \|$ denotes the $\ell_{2}$ norm.
\item $\|\cdot{} \|_{\infty}$ denotes the maximum absolute component of a vector.
\item $\x \sim \mathcal{P}$ denotes that the random variable $\X$ (denoted as $\x$), of realization $\x$, follows a probability distribution $\mathcal{P}$ (we abuse notation by denoting similarly a random variable and its realization).
\item $\x_1, .., \x_n \overset{i.i.d}{\sim} \mathcal{P}$ denotes that we draw $n$ i.i.d. samples of a random variable $\x$, each from the distribution $\mathcal{P}$. 
\item $P(\x)$ denotes the value of the probability of $\x$ according to its  probability distribution.
\item $\E_{\x\sim \mathcal{P}}$ (or simply $\E_{\x}$ if there is no possible confusion) to denote the expectation of $\x$ which follows the distribution $\mathcal{P}$.
\item We denote by $\text{supp}(\x)$ the support of a vector $\x$, that is the set of its non-zero coordinates.
\item $|F|$ the cardinality (number of elements) of a set $F$.
\item All the sets we consider are subsets of $[d]$. So for a given set $F$, $F^c$ denotes the complement of $F$ in $[d]$
\item $\mathcal{S}^d(R)$ (or $\mathcal{S}^d(R)$ for simplicity if $R=1$) denotes the $d$-sphere of radius $R$, that is $\mathcal{S}^d(R) = \{\uu \in \mathbb{R}^d / \|\uu\|=R \}$.
\item $\mathcal{U}(\mathcal{S}^d)$ the uniform distribution on that unit sphere.
  \item  $\beta(d)$ is the surface area of the unit $d$-sphere defined above.
\item $\mathcal{S}^d_{S}$ denotes a set that we call the restricted $d$-sphere on $S$, described as:  $\{ \uu_S / \uu \in \{ \bm{v} \in \mathbb{R}^d / \|\bm{v}_S\|=1 \}\}$, that is the set of unit vectors supported by $S$.
  \item $\mathcal{U}(\mathcal{S}^d_S)$ denotes the uniform distribution on that restricted sphere above.
  \item We denote by $\uu_F$ (resp. $\nabla_F f(\x)$) the hard-thresholding of $\uu$ (resp. $\nabla f(\x)$) over the support $F$, that is, a vector which keeps $\uu$ (resp. $\nabla f(\x)$) untouched for the set of coordinates in $F$, but sets all other coordinates to $0$.
  \item $\binom{[d]}{s}$ denotes the set of all subsets of $[d]$ that contain $s$ elements: $\binom{[d]}{s} = \{ S: |S|=s, S \subseteq [d]\}$.
\item $\mathcal{U}(\binom{[d]}{s})$ denotes the uniform distribution on the set above.
\item $\bm{I}$ denotes the identity matrix $\bm{I}_{d\times d}$.
\item $\bm{I}_S$ denotes the identity matrix with 1 on the diagonal only at indices belonging to the support $S$: $\bm{I}_{i, i}=1 \text{ if } i \in S$, and 0 elsewhere.
\item  $S \ni e$ denotes  that set $S$ contains the element $e$.
\item  $(\uu_i)_{i=1}^n$ denotes the $n$-uple of elements $\uu_1, .., \uu_n$.
\item $\Gamma$ denotes the Gamma function \cite{arfken1999mathematical}.
 \item $\int_{A}f(\uu)d\uu$ denotes the integral of $f$ over the set $A$.
 \item $\log$ denotes the natural logarithm (in base $e$).
\end{itemize}

\section{Auxilliary Lemmas}
\label{sec:proj-vect-onto}

\begin{appxlem}[\cite{Sykora2005} (10)] \label{lem:pp}
  Let $\bm{p} \in \mathbb{N}^{d}$, and denote $p:= \sum_{i=1}^d \bm{p}_i$, we have:
  $$\int_{\mathcal{S}^d} \prod_{i=1}^d \left( \uu_i^2 \right)^{\bm{p}_i}d\uu = 2 \frac{\prod_{i=1}^n \Gamma(\bm{p}_i + 1/2)}{\Gamma(p + d/2)}$$
  
\end{appxlem}
\begin{proof}
  The proof is given in \cite{Sykora2005}.
  \end{proof}
\begin{appxlem}\label{lemma:expecto}
  Let $F$ be a subset of $[d]$, of size $s$, with $(s, d) \in \mathbb{N}_*^2$. 
  We have the following:
  \begin{equation}
    \E_{\uu \sim \mathcal{U}(\mathcal{S}^d)}\|\uu_{F}\| \leq  \sqrt{\frac{s}{d}} \label{eq:poww1}\\
    \end{equation}
    \begin{equation}
    \E_{\uu \sim \mathcal{U}(\mathcal{S}^d)} \|\uu_{F}\|^2  = \frac{s}{d} \label{eq:poww2}\\
          \end{equation}
\begin{equation}
        \E_{\uu \sim \mathcal{U}(\mathcal{S}^d)} \|\uu_{ F}\|^4  = \frac{(s+2)s}{(d+2)d} \label{eq:poww4} \\
  \end{equation}
  
\end{appxlem}

\begin{proof}
    We start by proving \eqref{eq:poww2}. Decomposing the norm onto every component, we get:
    \begin{equation}
   \E_{\uu \sim \mathcal{U}(\mathcal{S}^d)} \|\uu_{F}\|^2  = \E_{\uu \sim \mathcal{U}(\mathcal{S}^d)} \sum_{i\in F} \uu_i^2 =  \sum_{i\in F}\E_{\uu \sim \mathcal{U}(\mathcal{S}^d)} \uu_i^2      \label{eq:sumforallf}
    \end{equation}
    By symmetry, each $\uu_i$ has the same marginal probability distribution, so:
    \begin{equation}
  \forall i \in [d]: \quad\quad  \E_{\uu \sim \mathcal{U}(\mathcal{S}^d)} \uu_i^2=     \frac{1}{d}\sum_{i=1}^d \E_{\uu \sim \mathcal{U}(\mathcal{S}^d)} \uu_i^2 \label{eq:sym1}
    \end{equation}
    We also know, from the definition of the $\ell_2$ norm, and the fact that $\uu$ is a unit vector, that:
\begin{equation}
   \sum_{i=1}^d \E_{\uu \sim \mathcal{U}(\mathcal{S}^d)} \uu_i^2 =\E_{\uu \sim \mathcal{U}(\mathcal{S}^d)} \sum_{i=1}^d  \uu_i^2 = \E_{\uu \sim \mathcal{U}(\mathcal{S}^d)}\|\uu\|^2 = \E_{\uu \sim \mathcal{U}(\mathcal{S}^d)} 1 = 1     \label{eq:sym2}
\end{equation}
Therefore, combining \eqref{eq:sym1} and \eqref{eq:sym2}:
\begin{equation*}
\forall i \in [d]: \quad\quad \E_{\uu \sim \mathcal{U}(\mathcal{S}^d)} \uu_i^2 = \frac{1}{d}  
\end{equation*}
Plugging this into \eqref{eq:sumforallf}, we get  \eqref{eq:poww2}: $$   \E_{\uu \sim \mathcal{U}(\mathcal{S}^d)} \|\uu_{F}\|^2  = \frac{s}{d}$$
Using Jensen's inequality, \eqref{eq:poww1} follows from \eqref{eq:poww2}.
Let us now prove \eqref{eq:poww4}. By definition of the expectation for a uniform distribution on the unit sphere:
$$  \E_{\uu \sim \mathcal{U}(\mathcal{S}^d)} \|\uu_{F}\|^4 = \frac{1}{\beta(d)} \int_{\mathcal{S}^d} \|\uu_{F}\|^4 d\uu $$ We further develop the integral as follows:
\begin{align*}
  \int_{\mathcal{S}^d} \|\uu_{F}\|^4 d\uu &=  \int_{\mathcal{S}^d} (\|\uu_{F}\|^2)^2 d\uu =  \int_{\mathcal{S}^d} (\sum_{i\in F}\uu_i^4 + \sum_{(i, j)\in F, j\neq i}\uu_i^2 \uu_j^2) d\uu\\
                                    &= s  \int_{\mathcal{S}^d} \uu_1^4 d\uu + 2 {s \choose 2} \int_{\mathcal{S}^d} \uu_1^2\uu_2^2 d\uu \quad \text{(by symmetry)}\\
\end{align*}
Using Lemma \ref{lem:pp} in the expression above, with  $\bm{p}^{(a)}:= (2, 0, ..., 0)$, and $\bm{p}^{(b)} := (1, 1, 0, ..., 0)$, we obtain:
\begin{align*}
  \int_{\mathcal{S}^d} \|\uu_{F}\|^4 d\uu &= s \frac{\prod_{i=1}^d \Gamma(\bm{p}^{(a)}_k + \frac{1}{2}) }{\Gamma(2 + d/2)} + 2\frac{s(s-1)}{2} 2 \frac{\prod_{i=1}^d \Gamma(\bm{p}^{(b)}_k + 1/2)}{\Gamma(2+d/2)} \\
  &\overset{(a)}{=} \frac{6s\sqrt{\pi}^d}{(d+2)d\Gamma(d/2)} + \frac{2s(s-1)\sqrt{\pi}^d}{(d+2)d\Gamma(d/2)} = \frac{2(s+2)s\sqrt{\pi}^d}{(d+2)d\Gamma(d/2)}
\end{align*}
Where in (a) we used the fact that $\Gamma(\frac{1}{2}) = \sqrt{\pi}$ and $\Gamma(\frac{3}{2}) = \frac{\sqrt{\pi}}{2}$. So:
$$  \E_{\uu \sim \mathcal{U}(\mathcal{S}^d)}  \|\uu_F\|^4 = \frac{1}{\beta(d)} \int_{S^d}\|\uu_F\|^4 d\uu \overset{(b)}{=} \frac{s+2}{d+2}\frac{s}{d}$$
Where (b) comes from the closed form for the area of a $d$ unit sphere: $\beta(d) = \frac{2\sqrt{\pi}^d}{\Gamma(\frac{d}{2})}$
\end{proof}

\begin{appxlem}[\cite{gao2018information}, Lemma 7.3.b] \label{lemma:id}
  \begin{equation*}
\E_{\uu \sim \mathcal{U}(\mathcal{S}^d)}\uu \uu^T = \frac{1}{d} \bm{I}
  \end{equation*}
\end{appxlem}
\begin{proof}
  The proof is given in  \cite{gao2018information}.
\end{proof}
\begin{appxlem}[\cite{shen2017tight}, Theorem 1; \cite{yuan2017gradient}, Lemma 17]\label{lemma:lemma17}
 Let $\bm{b} \in \mathbb{R}^{d}$ be an arbitrary $d$-dimensional vector and $\bm{a} \in \mathbb{R}^{d}$ be any $k$-sparse vector. Denote $\bar{k}=\|\bm{a}\|_{0} \leq k$, and  $\bm{b}_k$ the vector $\bm{b}$ with all the $d-k$ smallest components set to 0 (that is, $\bm{b}_k$ is the best $k$-sparse approximation of $\bm{b}$). Then, we have the following bound:
$$
\left\|\bm{b}_{k}-\bm{a}\right\|^{2} \leq \delta\|\bm{b}-\bm{a}\|^{2}, \quad \delta=1+\frac{\beta+\sqrt{(4+\beta) \beta}}{2}, \quad \beta=\frac{\min \{\bar{k}, d-k\}}{k-\bar{k}+\min \{\bar{k}, d-k\}}
$$
\end{appxlem}
\begin{proof}
  The proof is given in \cite{shen2017tight}.
\end{proof}
\begin{appxcor}\label{cor:cor17}
   With the notations and variables above in Lemma \ref{lemma:lemma17},   we also have the following, simpler  bound, from  \cite{yuan2017gradient}:
$$\left\|\bm{b}_{k}-\bm{a}\right\| \leq \gamma \|\bm{b}-\bm{a}\| $$ with $$  \gamma = \sqrt{1 + \left( \bar{k}/k + \sqrt{\left(4 + \bar{k}/k\right)\bar{k}/k}\right)/2}    $$
\end{appxcor}
\begin{proof}
  
  There are two possibilities for $\beta$ in Lemma \ref{lemma:lemma17}: either $\beta=\frac{\bar{k}}{k}$ (if $d-k >  \bar{k}$) or $\beta=\frac{d-k}{d-\bar{k}}$  (if $d-k \leq \bar{k}$). In the latter case:\\
  $$d-k \leq \bar{k} \implies d-\bar{k} \leq k \implies \frac{k-\bar{k}}{d-\bar{k}} \geq \frac{k-\bar{k}}{k} \implies 1- \frac{k-\bar{k}}{d-\bar{k}} \leq 1- \frac{k-\bar{k}}{k} \implies \frac{d-k}{d-\bar{k}} \leq \frac{\bar{k}}{k}$$
  Therefore, in both cases, $\beta \leq \frac{\bar{k}}{k}$, which, plugging into Lemma \ref{lemma:lemma17}, gives Corollary \ref{cor:cor17}.
\end{proof}

\section{Proof of Proposition \ref{prop:zograd} }  \label{sec:proof_zograd}

With an abuse of notation, let us denote by $f$ any function $f_{\stoch}$ for some given value of the noise $\stoch$.
First, we derive in section \ref{sec:one-direct-estim} the error of the gradient estimate if we sample only one direction ($q=1$). Then, in section \ref{sec:batched}, we show how sampling $q$ directions reduces the error of the gradient estimator, producing the results of Proposition \ref{prop:zograd}.

\subsection{One direction estimator}
\label{sec:one-direct-estim}

Throughout all this section, we assume that $q=1$ for the gradient estimator $\hat{\nabla} f(x)$ defined in \eqref{eq:zoest}.
\subsubsection{Expected deviation from the mean}
\label{sec:expect-devi-from}

\begin{appxlem}
  \label{lemma:one_direc_mu}
For any $(L_{\sus}, \sus)$-RSS function $f$, using the gradient estimator  $\hat{\nabla} f(x)$ defined in \eqref{eq:zoest} with $q=1$, we have, for any support $F \in [d]$, with $|F|=s$: 
\begin{equation*}
  \left\| \E \left[\hat{\nabla}_F f(\x)\right]- \nabla_F f(\x)\right\|^2 \leq \varepsilon_{\mu} \mu^2
\end{equation*}
with $\varepsilon_{\mu} = L_{\sus}^2sd$
\end{appxlem}
\begin{proof}

From the definition of the gradient estimator in \eqref{eq:zoest}: 

\begin{align*}
 \|\E[\hat{\nabla}_F f(\x)] -\nabla_F f(\x)\|=\left\|\E d \frac{f(\x+\mu \uu)-f(\x)}{\mu} \uu_F-\nabla_F f(\x) \right\|\\
\end{align*}
Now, $(L_{\sus}, \sus)$-RSS implies continuous differentiability over an $\sus$-sparse direction (since $(L_{\sus}, \sus)$-RSS  actually equals Lipschitz continuity of the gradient over any $\sus$-sparse set, which implies continuity of the gradient over those sets). Therefore, from the mean value theorem, , we have, for some $c \in [0, \mu]$: $\frac{f(\x+\mu \uu)-f(\x)}{\mu} = \langle \nabla f(\x+c \uu), \uu \rangle$.
We now use the following result:
\begin{align*}
 \E \uu \uu^T = \E_{S\sim \rs{\sus}{d}} \E_{\uu \sim \mathcal{U}(\mathcal{S}_S^d)} \uu \uu^T \overset{(a)}{=} \E_{S\sim \rs{\sus}{d}} \frac{1}{\sus} \bm{I}_S = \frac{1}{\sus} \E_{S\sim \rs{\sus}{d}}\bm{I}_S\overset{(b)}{=} \frac{1}{\sus} \frac{\sus}{d} \bm{I} = \frac{1}{d}\bm{I}
\end{align*}
Where for (a) comes from applying Lemma \ref{lemma:id} to the unit sub-sphere on the support  $S$, and (b) follows by observing that each diagonal element of index $i$ actually follows a Bernoulli distribution of parameter $\frac{\sus}{d}$, since there are $\binom{d-1}{\sus-1}$ arrangements of the support which contain $i$, over $\binom{d}{\sus}$ total arrangements, which gives a probability $p = \frac{\binom{d-1}{\sus-1}}{\binom{d}{\sus}} = \frac{(d-1)!\sus!(d-\sus)!}{(\sus-1)!(d-1-(\sus-1))!d!} = \frac{\sus}{d}$ to get the value $1$ at $i$.

This allows to factor the true gradient into the scalar product:
\begin{align*}
  \|\E[\hat{\nabla}_F f(\x)]-\nabla_F f(\x)\|&=d \|\E \langle \nabla f(\x+c \uu) - \nabla f(\x), \uu \rangle \uu_F \|\\
 &\leq d \E \|\uu_F \uu^T[ \nabla f(\x+c \uu) - \nabla f(\x)]  \|
\end{align*}
where the last inequality follows from the property $\E\| \X - \E \X \|^2 = \E\|  \X \|^2 - \| \E \X \|^2$, which implies $ \| \E \X \| = \sqrt{\E\|  \X \|^2  - \E\| (\X - \E \X) \|^2 } \leq\E\|  \X \| $, for any multidimensional random variable $\X$.
Using the Cauchy-Schwarz inequality, we obtain:
\begin{align*}
  \|\E[\hat{\nabla}_F f(\x)]-\nabla_F f(\x)\|&\leq \E_{S\sim \rs{\sus}{d}} \E_{\uu \sim \mathcal{U}(\mathcal{S}_S^d)} \|\uu_F\|  \|\uu\| \|\nabla_S f(\x+c \uu) - \nabla_S f(\x)\|
\end{align*}
Since $f \in (L_{\sus}, \sus)$-RSS and $\|\uu_s\|_0\leq \sus$, we have: $\|\nabla_S f(\x+c \uu) - \nabla_S f(\x)\| \leq L_{\sus} \|c \uu\|$. We also have  $c \in [0, \mu]$, which implies $\|c \uu \| \leq \mu \|\uu\|$. Therefore:
\begin{align*}
  \|\E[\hat{\nabla}_F f(\x)]-\nabla_F f(\x)\|&\leq \E_S \E_{\uu}dL_{\sus}\mu \|\uu_F\|  \|\uu\| \|\uu\| = \E_S \E_{\uu}dL_{\sus}\mu\|\uu_F\|  \|\uu\|^2  =\E_S  \E_{\uu}dL_{\sus}\mu \|\uu_F\| \\
                                                                     & \overset{(a)}{\leq}  dL_{\sus}\mu \E_S\E_{\uu}\sqrt{\frac{|S\cap F|}{\sus}} \label{eq:expecto1} \nonumber \\ 
                                                                     &\overset{(b)}{\leq}  dL_{\sus}\mu\sqrt{ \E_S \frac{|S\cap F|}{\sus}}=  dL_{\sus}\mu\sqrt{ \E_k \E_{S||S\cap F|=k} \frac{k}{\sus}}\\
                                                                     &=  dL_{\sus}\mu\sqrt{\frac{s \sus}{d \sus}}=  L_{\sus}\mu\sqrt{sd}
                                                        \end{align*}
Where (a) follows from Lemma \ref{lemma:expecto}, restricted to the support $S$, and (b) follows from Jensen's inequality.
\end{proof}
\subsubsection{Expected norm}

\begin{appxlem}
  \label{lemma:one_direc_norm}
For any $(L_{\sus}, \sus)$-RSS function $f$, using the gradient estimator  $\hat{\nabla} f(x)$ defined in \eqref{eq:zoest} with $q=1$, we have, for any support $F \in [d]$, with $|F|=s$: 
\begin{align*}
      &\E \| \hat{\nabla}_F f(\x) \|^2  =  \varepsilon_{F}   \| \nabla_F f(\x) \|^2  + \varepsilon_{F^c} \| \nabla_{F^c} f(\x) \|^2 + \varepsilon_{\text{abs}} \mu^2
\end{align*}
with:\\
(i) $\varepsilon_{F} =  \frac{2d}{(\sus + 2)}   \left(\frac{(s-1)(\sus-1)}{d-1} + 3\right)  $\\
(ii) $\varepsilon_{F^c} =  \frac{2d}{(\sus + 2)}  \left( \frac{s(\sus-1)}{d-1}\right) $\\
(iii) $\varepsilon_{\text{abs}} = 2d L_s^2 s \sus\left(\frac{(s-1)(\sus-1)}{d-1}+1\right)$
\end{appxlem}
\begin{proof}

\begin{align*}
  \E \| \hat{\nabla}_F f(\x) \|^2  &=  \E \left\|  d\frac{f(\x+\mu \uu) - f(\x)}{\mu} \uu_F  \right\|^2 \\
  &=\E  \frac{d^{2}}{\mu^{2}}|f(\x+\mu \uu)-f(\x)|^{2}\|\uu_F\|^{2} \\
  &= \frac{d^{2}}{\mu^{2}}\E[f(\x+\mu \uu)-f(\x)-\langle\nabla f(\x), \mu \uu\rangle+\langle\nabla f(\x), \mu \uu)]^{2} \| \uu_F \|^2
\end{align*}
Using the mean value theorem, we obtain that for a certain $c \in (0, \mu)$, we have:
\begin{equation*}
  f(\x+\mu \uu) - f(\x)= \langle\nabla f(\x+c), \mu \uu \rangle
\end{equation*}
Therefore, plugging this in the above: 
\begin{align}
\E \| \hat{\nabla}_F f(\x) \|^2&\leq   d^{2} \E[\langle \nabla f(\x + c \uu) - \nabla f(\x), \uu\rangle+\langle\nabla f(\x), \uu\rangle]^{2} \| \uu_F \|^2 \nonumber\\
                            &  \overset{(a)}{\leq}    d^{2} \E\left[2\langle \nabla f(\x +c \uu) - \nabla f(\x), \uu\rangle^2\| \uu_F \|^2+\langle\nabla f(\x), \uu\rangle^{2} \| \uu_F \|^2 \right]\nonumber\\
                            &  \leq   2 d^{2} \E [ \|  \nabla f(\x +c \uu) - \nabla f(\x) \|^2\| \uu \|^2\| \uu_F \|^2+\langle\nabla f(\x), \uu\rangle^{2} \| \uu_F \|^2 ]\nonumber\\
                            &  \overset{\leq}{(b)} 2   d^{2} \E  [L_s^2\mu^2\| \uu \|^2\| \uu \|^2\| \uu_F \|^2+\langle\nabla f(\x), \uu\rangle^{2}\| \uu_F \|^2] \nonumber\\
                               &  \overset{\text(c)}{=}  2  d^{2} \E   [L_s^2\mu^2\| \uu_F \|^2+\langle\nabla f(\x), \uu\rangle^2\| \uu_F \|^2]\nonumber\\
&=   2d^{2} [ L_s^2\mu^2\E \| \uu_F \|^2+ {\nabla f(\x)}^T\left(\E \uu \uu^T \| \uu_F \|^2 \right)\nabla f(\x) ]\nonumber \\
                               &=    2d^{2} [ L_{\sus}^2\mu^2\E \| \uu_F \|^2+ {\nabla f(\x)}^T(\E_{S\sim \rs{\sus}{d}} \E_{\uu \sim \mathcal{U}(\mathcal{S}_S^d)} \uu \uu^T \| \uu_F \|^2 )\nabla f(\x)] \nonumber\\
  &\overset{(d)}{=}    2d^{2} [ L_{\sus}^2\mu^2\E \| \uu_F \|^2+ \E_{S\sim \rs{\sus}{d}} [{\nabla f(\x)}^T( \E_{\uu \sim \mathcal{U}(\mathcal{S}_S^d)} \uu \uu^T \| \uu_F \|^2 )\nabla f(\x) ]]  \label{eq:befouter}
\end{align}
Where (a) follows from the fact that for any $ (a, b)\in \mathbb{R}^2: (a+b)^2 \leq 2a^2 + 2b^2$, (b) follows from the Cauchy-Schwarz inequality, (c) follows from the fact that $\| \uu \|=1$ since $\uu \in \mathcal{S}_S^d$, and (d) follows by linearity of expectation.
Let us turn to computing the following expression above: $\E_{\uu \sim \mathcal{U}(\mathcal{S}_S^d)} \uu \uu^T \| \uu_F \|^2 $. We start by distinguishing the indices that belong to $F$ and those that do not.
By symmetry, denoting $i_1, ..., i_s$ the elements of $F$:
\begin{align*}
  \E_{\uu \sim  \mathcal{U}(\mathcal{S}_S^d) } \uu_{i_1}^2 \|\uu_F\|^2 = ... = \E_{\uu \sim \mathcal{U}(\mathcal{S}_S^d)}  \uu_{i_s}^2 \|\uu_F\|^2  
\end{align*}
Therefore, for all $i\in F$:
  \begin{align}
\E_{\uu \sim \mathcal{U}(\mathcal{S}_S^d)}  \uu_{i}^2 \|\uu_F\|^2    &= \frac{1}{|S\cap F|} \sum_{j=1}^s \E_{\uu \sim \mathcal{U}(\mathcal{S}_S^d)}\uu_{i_j}^2 \|\uu_F\|^2 \nonumber\\
    &= \frac{1}{|S\cap F|}\E_{\uu \sim \mathcal{U}(\mathcal{S}_S^d)}  \sum_{j=1}^s\uu_{i_j}^2 \|\uu_F\|^2 
  = \frac{1}{|S\cap F|}\E_{\uu  \sim \mathcal{U}(\mathcal{S}_S^d)}  \|\uu_F\|^4  \label{eq:suppex}
  \end{align}

  By definition of the restricted $d$-sphere on $F$ (see section \ref{sec:notations}), for all $\uu \in \mathcal{S}_S^d$, if $i\not\in S$: $\uu_i=0$. Therefore, since the exact indices of the elements of $F$ do not matter in the expected value \eqref{eq:suppex}, but only their cardinality, \eqref{eq:suppex} can be rewritten using a simpler expectation over a unit $|S|$-sphere as follows :
  $$\E_{\uu  \sim \mathcal{U}(\mathcal{S}_S^d)}  \|\uu_F\|^4 = \E_{\uu \sim \mathcal{U}(\mathcal{S}^{|S|})} \|\uu_{[|S\cap F|]}\|^4$$
Using Lemma \ref{lemma:expecto} to get a closed form expression of the expected value above, we further obtain:
\begin{equation}
\forall i \in F:\E_{\uu \sim \mathcal{U}(\mathcal{S}^d)} \uu_i^2 \|\uu_F\|^2 =\frac{1}{|S\cap F|} \frac{|S\cap F|(|S\cap F|+2)}{d(d+2)}  = \frac{|S\cap F|+2}{d(d+2)}  \label{eq:iiF}
\end{equation}
Similarly, by symmetry, denoting  $i_1, ..., i_{d-s}$ the elements of $F^c$:
\begin{align*}
  \E_{\uu \sim  \mathcal{U}(\mathcal{S}_S^d) } \uu_{i_j}^2 \|\uu_F\|^2 &= ... =  \E_{\uu \sim \mathcal{U}(\mathcal{S}_S^d)} \uu_{i_j}^2 \|\uu_F\|^2
\end{align*}
Therefore, for all $i\not\in F$:
\begin{align*}
              \E_{\uu \sim \mathcal{U}(\mathcal{S}_S^d)} \uu_{i}^2 \|\uu_F\|^2 &=                                       \frac{1}{d-s} \sum_{j=1}^{d-s} \E_{\uu \sim \mathcal{U}(\mathcal{S}_S^d)}\uu_{i_j}^2 \|\uu_F\|^2 =  \frac{1}{d-s}  \E_{\uu \sim \mathcal{U}(\mathcal{S}_S^d)} \sum_{j=1}^{d-s}\uu_{i_j}^2 \|\uu_F\|^2\\
                                                                                   &\overset{(a)}{=} \frac{1}{d-s}  \E_{\uu \sim \mathcal{U}(\mathcal{S}_S^d)} (\|\uu\|^2 - \|\uu_F\|^2)\|\uu_F\|^2\\
   &\overset{(b)}{=} \frac{1}{d-s}  (\E_{\uu \sim \mathcal{U}(\mathcal{S}_S^d)} \|\uu_F\|^2 - \E_{\uu \sim \mathcal{U}(\mathcal{S}_S^d)} \|\uu\|^4)
\end{align*}
Where (a) follows from the Pythagorean theorem and (b) follows from $\|\uu\|=1$. Similarly as before, rewriting those expected values and using Lemma \ref{lemma:expecto}, we obtain:
\begin{equation}
  \forall i \not\in F:\E_{\uu \sim \mathcal{U}(\mathcal{S}^d)} \uu_i^2 \|\uu_F\|^2 = \frac{1}{d-|S\cap F|}\frac{|S\cap F|(d+2-(|S\cap F|+2))}{d(d+2)} = \frac{|S\cap F|}{d(d+2)}\label{eq:iiFc}
\end{equation}
Finally,  by symmetry of the distribution $\mathcal{U}(\mathcal{S}_S^d)$, we have, for all  $(i, j) \in [d]^2$ with $i\neq j$:
\begin{align*}
\E_{\uu \sim  \mathcal{U}(\mathcal{S}_S^d) } \uu_{i} \uu_{j} \|\uu_F\|^2   = \E_{\uu \sim  \mathcal{U}(\mathcal{S}_S^d) } (- \uu_{i}) \uu_{j} \|\uu_F\|^2 = -  \E_{\uu \sim  \mathcal{U}(\mathcal{S}_S^d) } \uu_{i} \uu_{j} \|\uu_F\|^2
\end{align*}
Therefore, for all $(i, j) \in [d]^2,  i\neq j$:
\begin{equation}
\E_{\uu \sim  \mathcal{U}(\mathcal{S}_S^d) } \uu_{i} \uu_{j} \|\uu_F\|^2 = 0 \label{eq:ij}
\end{equation}
Therefore, combining \eqref{eq:iiF}, \eqref{eq:iiFc} and \eqref{eq:ij}, we obtain:
\begin{equation*}
    \E_{\uu \sim \mathcal{U}(\mathcal{S}_S^d)} \uu \uu^T \| \uu_F \|^2 =
 \begin{bmatrix}
   a_{1} &  &  & \\ 
   & a_{2} &  & \\ 
   &  &  \ddots & \\ 
   &  &   & a_{d} 
 \end{bmatrix}
\end{equation*}
With, for all $i\in [d]: a_i = \begin{cases}
                 \frac{|S\cap F|+2}{d(d+2)} \text{ if } i\in F\\
               \frac{|S\cap F|}{d(d+2)} \text{ if } i\not\in F\\
            \end{cases}$. Plugging this back into \eqref{eq:befouter}, we obtain:
\begin{align}
  A &:= \E_{S\sim \rs{\sus}{d}}  [{\nabla f(\x)}^T\left(  \E_{\uu \sim \mathcal{U}(\mathcal{S}_S^d)} \uu \uu^T \| \uu_F \|^2 \right)\nabla f(\x)]\nonumber\\
  &= \E_{S\sim \rs{\sus}{d}} \left[\frac{|S\cap F| +2}{\sus (\sus+2)}\|\nabla_{S \cap F} f(\x)\|^2 + \frac{|S\cap F|}{\sus(\sus+2)} \|\nabla_{S \backslash (S\cap F)} f(\x)\|^2\right]\nonumber\\
  &= \frac{1}{\sus(\sus+2)}\left[\E_{S\sim \rs{\sus}{d}} \left[|S\cap F|~\|\nabla_{F \cap S} f(\x)\|^2 \right]\right.\nonumber\\
  & \quad\left. + 2 \E_{S\sim \rs{\sus}{d}} \left[\|\nabla_{F \cap S} f(\x)\|^2 + |S\cap F|~\|\nabla_{S \backslash (S\cap F)} f(\x)\|^2\right]\right]\label{eq:beftte}
\end{align}
We will now develop the expected values above using the law of total expectation, to exhibit the role of the random variable $k$ denoting the size of $S\cap F$. Given that we sample $\sus$ indices from $[d]$ without replacement, $k$ follows a hypergeometric distribution with, as parameters, population size $d$, number of success states $s$ and number of draws $\sus$, which we denote $\mathcal{H}(d, s, \sus)$. For simplicity, we will use the following notations for the expected values: $\E_k [\cdot] := \E_{k\sim \mathcal{H}(d, s, \sus)} [\cdot]$, and $\E_{S||S\cap F|=k}[\cdot] =\E_{S \sim \rs{\sus}{d} | |S\cap F|=k}[\cdot]$. Therefore, rewriting \eqref{eq:beftte} using the law of total expectation, we obtain:
  \begin{align}
    A&=    \frac{1}{\sus(\sus+2)}\left[\E_k \E_{S| |S \cap F|=k} k\|\nabla_{S \cap F} f(\x)\|^2 + 2 \E_k \E_{S| |S \cap F|=k}\|\nabla_{S \cap F} f(\x)\|^2 \right. \nonumber\\
     &\quad \left. + \E_k\E_{S| |S \cap F|=k}k \|\nabla_{S \backslash (S\cap F)} f(\x)\|^2\right]\nonumber\\
    &= \frac{1}{\sus(\sus+2)}\left[\E_k k \E_{S| |S \cap F|=k} \|\nabla_{S \cap F} f(\x)\|^2 + 2 \E_S \E_{S| |S \cap F|=k}\|\nabla_{S \cap F} f(\x)\|^2 \right. \nonumber\\
     & \left. \quad + \E_kk\E_{S| |S \cap F|=k}\|\nabla_{S \backslash (S\cap F)} f(\x)\|^2\right]\label{eq:Afin}
  \end{align}
  To compute the conditional expectations above, let us consider the first of them (the other ones will follow similarly) :  $ \E_{S| |S \cap F|=k} \|\nabla_{S \cap F} f(\x)\|^2$. Given some $k$, from the multiplication principle in combinatorics, we can have $\binom{d}{k}\binom{d-s}{\sus - k}$ arrangements of supports such that $k$ elements of that support are in $F$ (because it means there are $k$ elements in $F$ and $\sus - k$ elements outside of $F$). So the conditional probability of each of those supports $S$, assuming they indeed have at least one element in common with $F$, is $\left(\binom{d}{k}\binom{d-s}{\sus - k}\right)^{-1}$. Otherwise it is $0$. To rewrite it: $$P(S||S\cap F|=k)=\begin{cases}
     \left(\binom{d}{k}\binom{d-s}{\sus - k}\right)^{-1} \text{ if  } S\cap F \neq \varnothing\\
    0 \text{ if } S\cap F \neq \varnothing
  \end{cases}
  $$
  So, developing   $ \E_{S| |S \cap F|=k} \|\nabla_{S \cap F} f(\x)\|^2$ using the definition of conditional probability, we have:
  \begin{align}
    \E_{S| |S \cap F|=k} \|\nabla_{S \cap F} f(\x)\|^2 &= \sum_{S}P(S|~|S\cap F|=k) \sum_{i\in S\cap F}\nabla_i f(\x)^2\nonumber\\
                                                       &=  \sum_{S / |S\cap F|=k}    \left(\binom{d}{k}\binom{d-s}{\sus - k}\right)^{-1} \sum_{i\in S\cap F}\nabla_i f(\x)^2\nonumber\\
                                                       &=  \left(\binom{d}{k}\binom{d-s}{\sus - k}\right)^{-1} \sum_{S / |S\cap F|=k}     \sum_{i\in S\cap F}\nabla_i f(\x)^2\nonumber\\
                                                           &\overset{(a)}{=} \left(\binom{d}{k}\binom{d-s}{\sus - k}\right)^{-1}\sum_{i\in F} \sum_{S / ((|S\cap F|=k), (S \ni i))}    \nabla_i f(\x)^2\nonumber\\
                                                       &\overset{(b)}{=} \left(\binom{d}{k}\binom{d-s}{\sus - k}\right)^{-1}\sum_{i\in F}  \binom{s-1}{k-1}\binom{d-s}{\sus-k} \nabla_i f(\x)^2\nonumber\\
                                                       &= \frac{s}{k} \sum_{i\in F}\nabla_i f(\x)^2\nonumber\\
    & = \frac{s}{k} \|\nabla_F f(\x)\|^2\label{eq:normF1}
  \end{align}
  Where (a) follows by re-arranging the sum, and (b) follows by observing that by the multiplication principle, there are $\binom{s-1}{k-1}\binom{d-s}{\sus-k}$ possible arrangements of support such that: $(|S\cap F|=k), (S \ni i)$, since one element of $S$ is already fixed to be $i$, so there remains $k-1$ indices to arrange over $s-1$ possibilities, and still $\sus-k$ indices to arrange over $d-s$ possibilities.
  Similarly, to  \eqref{eq:normF1} we have, for the second expectation:
  \begin{equation}
    \label{eq:normF2}
    \E_{S| |S \cap F|=k}\|\nabla_{S \backslash (S\cap F)} f(\x)\|^2 = \frac{\sus-k}{d-s} \|\nabla_{F^c} f(\x)\|^2
  \end{equation}
Therefore, plugging \eqref{eq:normF1} and \eqref{eq:normF2} into \eqref{eq:Afin} 
    \begin{align}
A &=  \frac{1}{\sus(\sus+2)}\left[\E_k k \frac{k}{s} \|\nabla_F f(\x)\|^2+ 2 \E_k \frac{k}{s}\|\nabla_{F} f(\x)\|^2 + \E_k k \frac{\sus-k}{d-s} \|\nabla_{F^c} f(\x)\|^2  \right]\nonumber \\
    &=  \frac{1}{\sus(\sus+2)}\left[\frac{1}{s}\|\nabla_F f(\x)\|^2 \left[ \E_kk^2 + 2\E_kk   \right]+  \|\nabla_{F^c} f(\x)\|^2  \left[\frac{\sus}{d-s}\left(\E_kk \right)- \frac{1}{d-s} \E_k k^2  \right] \right] \label{eq:plugexpect}
\end{align}
Since $k$ follows a hypergeometric distribution $\mathcal{H}(d, s, \sus)$, its expected value is given in closed form  by: $ \E_k k =  \frac{s \sus}{d}$ (see \cite{walck2007hand}, section 2.1.3). We can also express the non-centered moment of order 2, using the formula for $\text{Var}(X) = \E[X^2] - (\E[X])^2$, which holds for a random variable $X$, where $Var(X)$ denotes the variance of $X$:
\begin{align*}
  \E_k k^2 &= Var(k) + (\E_k[k])^2 \overset{(a)}{=} \frac{s \sus}{d}\frac{d-s}{d}\frac{d-\sus}{d-1} + \left( \frac{s \sus}{d}\right)^2 = \frac{s \sus}{d}\left(\frac{d-s}{d}\frac{d-\sus}{d-1} + \frac{s \sus}{d}\right)\\
           &= \frac{s \sus}{d}\left(\frac{d^2 - s d - \sus d + s \sus + s\sus d - s \sus}{d(d-1)}\right) = \frac{s \sus}{d}\left(\frac{d - s - \sus +  s\sus }{d-1}\right)\\
  &=\frac{s \sus}{d}\left(\frac{(s-1)(\sus-1)}{d-1}+1\right)
\end{align*}
Where (a) follows by the closed form for the variance of a hypergeometric variable given in \cite{walck2007hand}. Therefore, plugging in into \eqref{eq:plugexpect}:
\begin{align}
 &\E_S {\nabla f(\x)}^T\left(\E_{\mathcal{U}_S|S} \uu \uu^T \| \uu_F \|^2 \right)\nabla f(\x)\nonumber\\
  &=   \frac{1}{\sus(\sus+2)}\left[\frac{1}{s}\|\nabla_F f(\x)\|^2 \left[ \frac{s \sus}{d} \left(\frac{(s-1)(\sus-1)}{d-1} +1\right) + 2 \frac{s\sus}{d} \right]\right]\nonumber\\
  &\quad +   \frac{1}{\sus(\sus+2)} \|\nabla_{F^c} f(\x)\|^2  \left[\frac{\sus}{d-s}  \frac{s \sus}{d} - \frac{1}{d-s} \frac{s\sus}{d}\left( \frac{(s-1)(\sus-1)}{d-1} +1  \right) \right]\nonumber\\
  &=   \frac{1}{\sus+2}\left[\|\nabla_F f(\x)\|^2 \left[ \frac{1}{d} \left(\frac{(s-1)(\sus-1)}{d-1} +3\right)  \right] \right.\nonumber\\
  &\quad + \left.\|\nabla_{F^c} f(\x)\|^2  \left[ \frac{s}{(d-s)d} \left(\sus - \left( \frac{(s-1)(\sus-1)}{d-1} + 1\right)\right)  \right] \right]\nonumber\\
  &=   \frac{1}{d(\sus+2)}\left[\|\nabla_F f(\x)\|^2 \left[ \left(\frac{(s-1)(\sus-1)}{d-1} +3\right)  \right]\right.\nonumber\\
  &\quad \left. +  \|\nabla_{F^c} f(\x)\|^2  \left[ \frac{s}{(d-s)} \left(\sus - \left( \frac{(s-1)(\sus-1)}{d-1} + 1\right)\right)  \right]\right] \label{eq:plugsimp}
\end{align}
Let us simplify the rightmost term:
\begin{align*}
  \frac{s}{(d-s)} \left(\sus - \left( \frac{(s-1)(\sus-1)}{d-1} + 1\right)\right) &=    \frac{s(\sus-1)}{d-s} \left[ 1  - \frac{s-1}{d-1}\right]\\
  &=  \frac{s(\sus-1)}{(d-s)} \left[\frac{d-s}{d-1}\right] =  \frac{s(\sus - 1)}{d-1}
\end{align*}
Plugging it back into \eqref{eq:plugsimp}:
\begin{align*}
  &\E_S {\nabla f(\x)}^T\left(\E_{\mathcal{U}_S|S} \uu \uu^T \| \uu_F \|^2 \right)\nabla f(\x)\\
  &= \frac{1}{d(\sus + 2)} \left[ \| \nabla_F f(\x) \|^2 \left(\frac{(s-1)(\sus-1)}{d-1} + 3\right) + \| \nabla_{F^c} f(\x) \|^2 \left( \frac{s(\sus-1)}{d-1}\right)\right]\\
\end{align*}
Finally, plugging this back into \eqref{eq:befouter}:
\begin{align*}
  \E \| \hat{\nabla}_F f(\x) \|^2  &  =   2  d^{2} \left[ L_{\sus}^2\mu^2\E \| \uu_F \|^2+ {\nabla f(\x)}^T\left(\E \uu \uu^T \| \uu_F \|^2 \right)\nabla f(\x) \right] \\
                                &  =   2  d^{2} \left[ L_{\sus}^2\mu^2\E_k \E_{\uu | |S \cap F| = k} \| \uu_F \|^2+ {\nabla f(\x)}^T\left(\E \uu \uu^T \| \uu_F \|^2 \right)\nabla f(\x) \right] \\
                                &  =    2 d^{2} \left[ L_{\sus}^2\mu^2\E_kk^2+  {\nabla f(\x)}^T\left(\E \uu \uu^T \| \uu_F \|^2 \right)\nabla f(\x) \right] \\
  &=   d 2L_{\sus}^2\mu^2 s \sus\left(\frac{(s-1)(\sus-1)}{d-1}+1\right) \\&\quad+ \frac{2d}{(\sus + 2)}   \left[ \| \nabla_F f(\x) \|^2 \left(\frac{(s-1)(\sus-1)}{d-1} + 3\right) + \| \nabla_{F^c} f(\x) \|^2 \left( \frac{s(\sus-1)}{d-1}\right)\right]
\end{align*}

\end{proof}
\subsection{Batched-version of the one-direction estimator}
\label{sec:batched}
We now describe how sampling $q\geq 1$ random directions improves the gradient estimate. Our proof is similar to the proof of  Lemma 2 in \cite{liu2018zeroth}, however we make sure that it works for our random support gradient estimator, and with our new expression in \ref{lemma:one_direc_norm}, which depends on the two terms $\| \nabla_F f(\x) \|^2$ and $\| \nabla_{F^c} f(\x) \|^2$. We express our results here in the form of a general lemma, depending only on the general bounding factors $\varepsilon_{F}$, $\varepsilon_{F^c}$, $\varepsilon_{\text{abs}}$ and $\varepsilon_{\mu}$ defined below, in such a way that the proof of Proposition \ref{prop:zograd} follows immediately from plugging the results of Lemma \ref{lemma:one_direc_mu} and \ref{lemma:one_direc_norm}  into  Lemma \ref{lemma:multibatch} below.

\begin{appxlem}
  \label{lemma:multibatch}
  For any $(L_{\sus}, \sus)$-RSS function $f$, we use the gradient estimator  $\hat{\nabla} f(x)$ defined in \eqref{eq:zoest} with $q\geq 1$. Let us suppose that the estimator  $\hat{\nabla} f(x)$ is such that for $q=1$, it verifies the following bounds for some $\varepsilon_{F}$, $\varepsilon_{F^c}$, $\varepsilon_{\text{abs}}$ and $\varepsilon_{\mu}$ in $\mathbb{R}_+^*$, for any support $F \in [d]$, with $|F|=s$:\\
  (i) $\|\E \hat{\nabla}_F f(\x) - \nabla_F f(\x)\|^2\leq \varepsilon_{\mu} \mu^2 $,  and\\
  (ii) $ \|\E  \hat{\nabla}_F f(\x) \|^2 \leq  \varepsilon_{F} \|\nabla_F f(\x)\|^2 + \varepsilon_{F^c} \|\nabla_{F^c} f(x)\|^2 + \varepsilon_{\text{abs}} \mu^2 $\\
  Then, the estimator  $\hat{\nabla} f(x)$ also verifies, for arbitrary $q\geq 1$ :\\
  (a) $   \|\E \hat{\nabla}_F f(\x) - \nabla_F f(\x)\|^2\leq \varepsilon_{\mu} \mu^2$\\
  (b) $\E\left\|\hat{\nabla}_{F} f(\x)\right\|^{2} \leq    \left( \frac{ \varepsilon_F}{q} + 2\right)  \|\nabla_F f(\x)\|^{2} + \frac{ \varepsilon_{F^c}}{q}  \|\nabla_{F^c} f(\x)\|^{2}+    \left(\frac{\varepsilon_{abs}}{q}+ 2\varepsilon_{\mu}\right) \mu^2$
\end{appxlem}
\begin{proof}

Let us  denote by $\hat{\nabla}f (\x; (\uu_i)_{i=1}^q)$ the gradient estimate from \eqref{eq:zoest} along the i.i.d. sampled directions $(\uu_i)_{i=1}^q$ (we simplify it into $\hat{\nabla}f (\x; \uu)$ if there is only one direction $\uu$).
  We can first see that, since the random directions $\uu_i$ are independent identically distributed (i.i.d.) we have:
  \begin{align*}
    \E \hat{\nabla} f(\x; (\uu_i)_{i=1}^q) &= \E  \frac{1}{q} \sum_{i=1}^q\hat{\nabla} f(\x; \uu_i) = \frac{1}{q} \sum_{i=1}^q\E\hat{\nabla} f(\x; \uu_1)  = \E \hat{\nabla} f(\x; \uu_1)    
  \end{align*}
  This proves \ref{lemma:multibatch} (a). Let us now turn to \ref{lemma:multibatch} (b).
We have:
  \begin{align} &\E\left[\left\|\hat{\nabla}_{F} f(\x; (\uu_i)_{i=1}^q)\right\|^{2}\right]=\E \left\|\frac{1}{q} \sum_{i=1}^q \hat{\nabla}_{F} f(\x; \uu_i)\right\|^{2}\nonumber\\ &\quad=\frac{1}{q^{2}} \E
\left(\sum_{i=1}^{q} \hat{\nabla}_{F} f(\x; \uu_i)\right)^{\top}\left(\sum_{i=1}^{q}\hat{\nabla}_{F} f(\x; \uu_i)\right)\nonumber\\
                 &\quad =\frac{1}{q^{2}} \sum_{i=1}^{q} \sum_{j=1}^{q} \E\left[\hat{\nabla}_{F} f(\x; \uu_i)^{\top}\hat{\nabla}_{F} f(\x; \uu_j)\right]\nonumber\\
&\quad \overset{\text(a)}{=}\frac{1}{q^{2}} \left[q \E\| \hat{\nabla}_{F} f(\x; \uu_1) \|^2 +\sum_{i=1}^{q} \sum_{j=1(j\neq i)}^{q}( \E \hat{\nabla}_{F}    f(\x; \uu_i))^{\top}(\E \hat{\nabla}_{F} f(\x; \uu_j))\right]\nonumber
  \\ &\quad =\frac{1}{q^{2}} \left[q \E||\hat{\nabla}_{F} f(\x; \uu_1)\| ^2 +q (q-1)  \|\E \hat{\nabla}_F f(\x; \uu_1)||^2\right]\nonumber\\
&\quad\overset{(b)}{\leq}\frac{1}{q^{2}}\left[q\left[    \varepsilon_{F}\|\nabla_F f(\x)\|^{2} +   \varepsilon_{F^c}\|\nabla_{F^c} f(\x)\|^{2}+\varepsilon_{abs}  \mu^2  \right]+q\left(q - 1\right)\left\|\E \hat{\nabla}_F f(\x; \uu_1)\right\|^{2}\right] \label{eq:plugaddbounhere}
\end{align}
Where (a) comes from the fact that the random directions are i.i.d. and (b) comes from assumptions (i) and (ii) of the current Lemma (Lemma \ref{lemma:multibatch}). Assumption (ii) also allows to bound the last term above in the following way:
  \begin{align} \|\E \hat{\nabla}_{F} f(\x; \uu_1)\|^{2} & \leq 2\|\nabla_{F} f(\x; \uu_1)-  \E \hat{\nabla}_{F} f(\x; \uu_1)\|^{2}+2\|\nabla_F f(\x; \uu_1)\|^{2} \nonumber\\ & \leq 2 \varepsilon_{\mu}\mu^2+2\|\nabla_F f(\x; \uu_1)\|^{2}\label{eq:addboun}
\end{align}
Plugging \eqref{eq:addboun} into \eqref{eq:plugaddbounhere}, we obtain:
\begin{align*} \E\left[\left\|\hat{\nabla}_{F}
  f(\x)\right\|^{2}\right]&\leq\frac{1}{q}\left[\varepsilon_{F}+2(q-1)\right]\|\nabla_F f(\x)\|^{2}+ \frac{ \varepsilon_{F^c}}{q}  \|\nabla_{F^c} f(\x)\|^{2}  \\
  &\quad+\frac{1}{q}\left[\varepsilon_{abs} \mu^2+2 \left(q-1\right) \varepsilon_{\mu} \mu^2 \right]\\ &\leq
\left(\frac{\varepsilon_{F}}{q}+2\right)\|\nabla_F
f(\x)\|^{2}+   \frac{\varepsilon_{F^c}}{q}  \|\nabla_{F^c} f(\x)\|^{2}+    \left(\frac{\varepsilon_{abs}}{q}+ 2\varepsilon_{\mu}\right)\mu^2
\end{align*}

\end{proof}

\subsection{Proof of Proposition \ref{prop:zograd}}
\label{sec:proof_final_zo}
\begin{proof}
  Proposition \ref{prop:zograd} (a) and (b) follow by plugging the values of  $\varepsilon_{F}$, $\varepsilon_{F^c}$, $\varepsilon_{\text{abs}}$ and $\varepsilon_{\mu}$ from Lemma \ref{lemma:one_direc_mu} and Lemma \ref{lemma:one_direc_norm} into Lemma \ref{lemma:multibatch}.
  Proposition (c) follows from the inequality $\|\bma +\bmb\|^2 \leq 2 \|\bma \|^2 + 2 \|\bmb\|^2$, for $\bma$ and $\bmb$ in $\mathbb{R}^p$ with $p\in \mathbb{N}^*$.
\end{proof}

\section{Proofs of section \ref{sec:glob}}
\label{sec:proofs-sect-refs}

\subsection{Proof of Theorem \ref{thrm:cvrate}}\label{sec:proof_cvrate}

\begin{proof}
 We will combine the proof from \cite{yuan2017gradient} and \cite{Nesterov17}, using ideas of the proof of Theorem 8 from Nesterov to deal with zeroth order gradient approximations, and ideas from the proof of \cite{yuan2017gradient} (Theorem 2 and 5, Lemma 19), to deal with the hard thresholding operation in the convergence rate.
Let us call $\eta$ an arbitrary learning rate, that will be fixed later in the proof.
Let us call $F$ the following support $F = F^{(t-1)}\cup F^{(t)} \cup \text{supp}(\x^*)$, with $F^{(t)} = \text{supp}(\x^{t})$. We have, for a given random direction $\uu$ and function noise $\stoch$, at a given timestep $t$ of SZOHT:

\begin{align*}
  \| \x^t - \x^* - \eta \hat{\nabla}_Ff_{\stoch}(\x^t) + \eta \nabla_F f_{\stoch}(\x^*) \|^2 &= \| \x^t - \x^* \|^2 - 2 \eta\langle \x^t - \x^*,  \hat{\nabla}_F f_{\stoch}(\x^t)- \nabla_F f_{\stoch}(\x^*)\rangle \\
  &\quad+ \eta^2 \| \hat{\nabla}_F f_{\stoch}(\x^t)  - \nabla_F f_{\stoch}(\x^*) \|^2
\end{align*}
Taking the expectation with respect to $\stoch$ and to the possible random directions $\uu_1, ..., \uu_q$ (that we denote with a simple $\uu$, abusing notations) at step $t$, we get:
\begin{align}
  &\E_{\stoch, \uu}  \| \x^t - \x^* - \eta \hat{\nabla}_Ff_{\stoch}(\x^t) + \eta \nabla_F f_{\stoch}(\x^*) \|^2\nonumber\\
  &= \| \x^t - \x^* \|^2 - 2 \eta \langle \x^t - \x^*, \E_{\stoch, \uu} [\hat{\nabla}_F f_{\stoch}(\x^t)- \nabla_F f_{\stoch}(\x^*)]\rangle + \eta^2\E_{\stoch, \uu} \| \hat{\nabla}_F f_{\stoch}(\x^t)  - \nabla_F f_{\stoch}(\x^*) \|^2\nonumber\\
  &= \| \x^t - \x^* \|^2 - 2 \eta \langle \x^t - \x^*,\E_{\stoch, \uu} [\nabla_F f_{\stoch}(\x^t)- \nabla_F f_{\stoch}(\x^*)]\rangle \nonumber\\
  &\quad- 2\eta\langle \x^t - \x^*, \E_{\stoch}[\E_{\uu} \hat{\nabla}_F f_{\stoch}(\x^t)- \nabla_F f_{\stoch}(\x^t)]\rangle  +\E_{\stoch, \uu} \eta^2 \| \hat{\nabla}_F f_{\stoch}(\x^t)  - \nabla_F f_{\stoch}(\x^*) \|^2\nonumber\\
  &= \| \x^t - \x^* \|^2 - 2 \eta \langle \x^t - \x^*, \nabla_F f(\x^t)- \nabla_F f(\x^*)\rangle \nonumber\\
  &\quad- 2\eta\langle \sqrt{\eta}L_{s'} \left(\x^t - \x^*\right), \frac{1}{\sqrt{\eta }L_{s'}}( \E_{\stoch}\E_{\uu} [\hat{\nabla}_F f_{\stoch}(\x^t)- \nabla_F f_{\stoch}(\x^t))]\rangle  \nonumber\\ &\quad+ \E_{\stoch, \uu} \eta^2 \| \hat{\nabla}_F f_{\stoch}(\x^t)  - \nabla_F f_{\stoch}(\x^*) \|^2\nonumber\\
  &\overset{(a)}{\leq} \| \x^t - \x^* \|^2 - 2 \eta \langle \x^t - \x^*, \nabla_F f(\x^t)- \nabla_F f(\x^*)\rangle  + \eta^2 L_{s'}^2 \| \x^t - \x^* \|^2 \nonumber\\
  &\quad+ \frac{1}{L_{s'}^2} \E_{\stoch}\| \E_{\uu} \hat{\nabla}_F f_{\stoch}(\x^t)- \nabla_F f_{\stoch}(\x^t)) \|^2  + \eta^2 \E_{\stoch, \uu} \| \hat{\nabla}_F f_{\stoch}(\x^t)  - \nabla_F f_{\stoch}(\x^*) \|^2 \label{eq:plugboundhere}
\end{align}
Where (a) follows from the inequality $2 \langle \bm{u}, \bm{v}\rangle \leq \|\bm{u}\|^2 + \|\bm{v}\|^2$ for any $(\bm{u}, \bm{v}) \in (\mathbb{R}^d)^2$. From Proposition \ref{prop:zograd} (b), since almost each $f_{\stoch}$ is $(L_{s'}, s')$-RSS (hence also  $(L_{s'}, \sus)$-RSS), we know that for the $\varepsilon_{F}$, $\varepsilon_{F^c}$ and $\varepsilon_{\text{abs}}$ defined in Proposition \ref{prop:zograd} (b), we have for almost all  $\stoch$: $\E_{\uu} \| \hat{\nabla}_F f_{\stoch}(\x^t) \|^2 \leq \varepsilon_{F} \| \nabla_F{f_{\stoch}(\x^t)}  \|^2 + \varepsilon_{F^c} \| \nabla_{F^c}{f_{\stoch}(\x^t)}  \|^2 + \varepsilon_{\text{abs}}\mu^2$. This allows to develop the last term of \eqref{eq:plugboundhere} into the following:
\begin{align*}
  \E_{\stoch. \uu} \| \hat{\nabla}_F f_{\stoch}(\x^t)  - \nabla_F f_{\stoch}(\x^*) \|^2 &\leq    2\E_{\stoch, \uu}  \| \hat{\nabla}_F f_{\stoch}(\x^t) \|^2 +    2 \E_{\stoch}\| \nabla_F f_{\stoch}(\x^*) \|^2 \\
                                                                                        &\leq 2 \varepsilon_{F}\E_{\stoch}\| \nabla_F f_{\stoch}(\x^t) \|^2 +  2 \varepsilon_{F^c}\E_{\stoch}\| \nabla_{F^c} f_{\stoch}(\x^t) \|^2   \\
  &\quad+ 2 \varepsilon_{\text{abs}}\mu^2 + 2 \E_{\stoch}\| \nabla_F f_{\stoch}(\x^*) \|^2  \\
                                                                                        &\leq 2 \varepsilon_{F} \left[ 2\E_{\stoch}\|  \nabla_F f_{\stoch}(\x^t) - \nabla_F f_{\stoch}(\x^*) \|^2 + 2\E_{\stoch}\| \nabla_F f_{\stoch}(\x^*)  \|^2 \right] \\
  &\quad +  2 \varepsilon_{F^{c}} \left[ 2\E_{\stoch}\|  \nabla_{F^c} f_{\stoch}(\x^t) - \nabla_{F^c} f_{\stoch}(\x^*) \|^2 + 2\E_{\stoch}\| \nabla_{F^c} f_{\stoch}(\x^*)  \|^2 \right] \\
  &\quad+ 2 \varepsilon_{\text{abs}}\mu^2 + 2 \E_{\stoch}\| \nabla_F f_{\stoch}(\x^*) \|^2
\end{align*}
Just like the proof in \cite{yuan2017gradient}, we will express our result in terms of the infinity norm of $\nabla f(\x^*)$. For that, we will plug above the two following inequalites:
Same as their proof of Lemma 19, we have  $\| \nabla_Ff(\x^*) \| \leq \|  \nabla_sf(\x^*) \|$ (that is because we will have equality if the sets in the definition of $F$, namely  $F^{(t-1)}$, $F^{(t)}$ and $\text{supp}(\x^*)$, are disjoints (because their cardinality is respectively $k$, $k$ and $k^*$), but they may intersect). And we also have $\| \nabla_s f(\x^*) \|_2^2 \leq s \| \nabla f(\x^*) \|_{\infty}^2 $ (by definition of the $\ell_2$ norm and of the $\ell_{\infty}$ norm). Similarly, we also have:  $\| \nabla_{F^c} f(\x^*) \|_2^2 \leq (d-k) \| \nabla f(\x^*) \|_{\infty}^2 $, since $ |F^c| \leq d-k$.

Therefore, we obtain:
\begin{align*}
  \quad &\E_{\stoch, \uu} \| \hat{\nabla}_F f_{\stoch}(\x^t)  - \nabla_F f_{\stoch}(\x^*) \|^2 \\
        &\leq 4 \varepsilon_{F} \E_{\stoch}\|  \nabla_F f_{\stoch}(\x^t) - \nabla f_{\stoch}(\x^*) \|^2 + 4 \varepsilon_{F^c} \E_{\stoch}\|  \nabla_{F^c} f_{\stoch}(\x^t) - \nabla f_{\stoch}(\x^*) \|^2 \\
  &\quad+  ((4 \varepsilon_{F} s+2) +  \varepsilon_{F^c} (d-k) ) \E_{\stoch}\| \nabla f_{\stoch}(\x^*) \|_{\infty}^2 + 2 \varepsilon_{\text{abs}}\mu^2\\
  &\overset{(a)}{\leq} 4 \varepsilon_F \E_{\stoch}\|  \nabla f_{\stoch}(\x^t) - \nabla f_{\stoch}(\x^*) \|^2  +  ((4 \varepsilon_{F} s+2) +  \varepsilon_{F^c} (d-k) ) \E_{\stoch}\| \nabla f_{\stoch}(\x^*) \|_{\infty}^2 + 2 \varepsilon_{\text{abs}}\mu^2\\
\end{align*}
Where (a) follows by observing in  Proposition \ref{prop:zograd} (b) that $\varepsilon_{F^c}\leq \varepsilon_{F}$, and using the definition of the Euclidean norm. Let us plug the above into \eqref{eq:plugboundhere}, and use the fact that, from Proposition \ref{prop:zograd} (a), since each $f_{\stoch}$ is $(L_{s'}, s':=\max(\sus, s))$-RSS, it is also $(L_{s'}, \sus)$-RSS, so for the $\varepsilon_{\mu}$ from  Proposition \ref{prop:zograd} (a), we have, for almost any given $\stoch$: $ \| \E_{\uu}\hat{\nabla}_F f_{\stoch}(\x^t)- \nabla_F f_{\stoch}(\x^t)) \|^2\leq \varepsilon_{\mu} \mu^2$, and let us also use the fact that since each $f_{\stoch}$ is $(L_{s'}, \max(\sus, s))$-RSS , it is  also $(L_{s'}, |F|)$-RSS (since $|F|\leq s$)  which gives that for almost any $\stoch$: $f_{\stoch}$: $\| \nabla f_{\stoch}(\x^t) - \nabla f_{\stoch}(\x^*) \|^2    \leq L_{s'}^2 \|\x^t - \x^*\|^2$, to finally obtain:

\begin{align*}
  &\E_{\stoch, \uu}  \| \x^t - \x^* - \eta \hat{\nabla}_Ff_{\stoch}(\x^t) + \eta \nabla_F f_{\stoch}(\x^*) \|^2\\
  &\quad\leq (1+\eta^2 L_{s'}^2  + 4  \varepsilon_{F} \eta^2 L_{s'}^2 )\| \x^t - \x^* \|^2  - 2 \eta \langle \x^t - \x^*, \E_{\stoch}[\nabla f_{\stoch}(\x^t)- \nabla f_{\stoch}(\x^*)]\rangle \\
                                                                                                      &\quad \quad  +  \frac{\varepsilon_{\mu}}{L_{s'}^2} \mu+ 2\eta^2 \varepsilon_{\text{abs}}\mu^2 + \eta^2 ((4 \varepsilon_{F} s+2) +  \varepsilon_{F^c} (d-k) )   \E_{\stoch}\| \nabla f(\x^*) \|_{\infty}^2\\
                                                                                                             &\quad= (1+\eta^2 L_{s'}^2  + 4  \varepsilon_F\eta^2 L_{s'}^2 )\| \x^t - \x^* \|^2 - 2 \eta \langle \x^t - \x^*, \nabla f(\x^t)- \nabla f(\x^*)\rangle \\
                                                                             &\quad \quad +  \frac{\varepsilon_{\mu}}{L_{s'}^2}\mu + 2\eta^2 \varepsilon_{\text{abs}}\mu^2 + \eta^2  ((4 \varepsilon_{F} s+2) +  \varepsilon_{F^c} (d-k) )  \E_{\stoch}\| \nabla f(\x^*) \|_{\infty}^2
\end{align*}
Since $f$ is $(\nu_s, s)$-RSC, it is also $(\nu_s, |F|)$-RSC, since $|F|\leq 2k+k^*\leq s$, therefore, we have: $\langle \x^t - \x^*, \nabla f(\x^t)- \nabla f(\x^*)\rangle \geq \nu_s \| \x^t - \x^* \|^2$ (this can be proven by adding together the definition of $(\nu_s,s)$-RSC written respectively at $\x=\x^t$,$\y=\x^*$, and at $\x=\x^*$,$\y=\x^t$). Plugging this into the above:
\begin{align*}
&  \E_{\stoch, \uu}  \| \x^t - \x^* - \eta \hat{\nabla}_Ff_{\stoch}(\x^t) + \eta \nabla_F f_{\stoch}(\x^*) \|^2\\
  &\quad \leq \left(1 - 2 \eta \nu_s +  \left(4  \varepsilon_{F}+ 1\right) L_{s'}^2 \eta^2\right)\| \x^t - \x^* \|^2 \\
                                                                             &\quad\quad  +  \frac{\varepsilon_{\mu}}{L_{s'}^2} \mu^2 + 2\eta^2 \varepsilon_{\text{abs}}\mu^2 + \eta^2   ((4 \varepsilon_{F} s+2) +  \varepsilon_{F^c} (d-k) )  \E_{\stoch}\| \nabla f_{\stoch}(\x^*) \|_{\infty}^2
\end{align*}
The value of $\eta$ that minimizes the left term in $\eta$ is equal to $\frac{\nu_s}{(4\varepsilon_F+1 )L_{s'}^2}$ (because the optimum of the quadratic function $ax^2 + bx + c$ is attained in $-\frac{b}{2a}$ and its value is $-\frac{b^2}{4a}+ c$). Let us choose it, that is, we fix $\eta = \frac{\nu_s}{(4\varepsilon_F+1 )L_{s'}^2}$. Let us now define the following $\rho$:
\begin{equation*}
  \label{eq:3}
\rho^2 = 1 - \frac{4\nu_s^2}{4(4  \varepsilon_F+1)L_{s'}^2}   =  1 - \frac{\nu_s^2}{(4  \varepsilon_F+1)L_{s'}^2} 
\end{equation*}
We therefore have:
\begin{align*}
  &\E_{\stoch, \uu} \| \x^t - \x^* - \eta \hat{\nabla}_Ff_{\stoch}(\x^t) + \eta \nabla_F f_{\stoch}(\x^*) \|^2 \\
  &\quad\leq\rho^2 \| \x^t - \x^* \|^2  +  \frac{\varepsilon_{\mu}}{L_{s'}^2} \mu^2 + 2\eta^2 \varepsilon_{\text{abs}}\mu^2 + \eta^2  ((4 \varepsilon_{F} s+2) +  \varepsilon_{F^c} (d-k) )  \E_{\stoch}\| \nabla f_{\stoch}(\x^*) \|_{\infty}^2
\end{align*}
 We can now use the fact that for all $(a, b)\in (\mathbb{R}_+)^2: \sqrt{a + b} \leq \sqrt{a} + \sqrt{b}$, as well as Jensen's inequality, to obtain:
\begin{align*}
  &\E_{\stoch, \uu}  \| \x^t - \x^* - \eta \hat{\nabla}_Ff_{\stoch}(\x^t) + \eta \nabla_F f_{\stoch}(\x^*) \|         \\
  &\quad\leq \rho \| \x^t - \x^* \|  + \frac{\sqrt{\varepsilon_{\mu}}}{L_{s'}} \mu^2+ \eta \sqrt{2 \varepsilon_{\text{abs}}\mu^2 } + \eta \sqrt{         ((4 \varepsilon_{F} s+2) +  \varepsilon_{F^c} (d-k) ) ) \E_{\stoch}\| \nabla f_{\stoch}(\x^*) \|_{\infty}^2}
 \end{align*}
 We can now formulate a first decrease-rate type of result, before the hard thresholding operation, as follows, using for $\eta$ the value previously defined, and with:
 \begin{equation}
   \label{eq:defy}
\y^t := \x^{t}-\eta \hat{\nabla}_{F} f_{\stoch}\left(\x^{t}\right)    
 \end{equation}
 \begin{align}
\E_{\stoch, \uu}  \|\y^t-\x^*\| &= \E_{\stoch, \uu}\left\|\x^{t}-\eta \hat{\nabla}_{F} f_{\stoch}\left(\x^{t}\right)-\x^*\right\|\nonumber \\
                       & \leq\E_{\stoch, \uu}\left\|\x^{t}-\x^*-\eta \hat{\nabla}_{F} f_{\stoch}\left(\x^{t}\right)+\eta \nabla_{F} f_{\stoch}(\x^*)\right\|+\eta\E_{\stoch}\left\|\nabla_{F} f_{\stoch}(\x^*)\right\| \nonumber\\
                       &=\E_{\stoch, \uu}\left\|\x^{t}-\x^*-\eta \hat{\nabla}_{F} f_{\stoch}\left(\x^{t}\right)+\eta \nabla_{F} f_{\stoch}(\x^*)\right\|+\eta\E_{\stoch}\sqrt{\left\|\nabla_{F} f_{\stoch}(\x^*)\right\|^2} \nonumber\\
      &\leq\E_{\stoch, \uu}\left\|\x^{t}-\x^*-\eta \hat{\nabla}_{F} f_{\stoch}\left(\x^{t}\right)+\eta \nabla_{F} f_{\stoch}(\x^*)\right\|+\eta\sqrt{\E_{\stoch}\left\|\nabla_{F} f_{\stoch}(\x^*)\right\|^2} \nonumber\\
                              & \leq \rho\left\|\x^{t}-\x^*\right\| +  \eta     (\sqrt{   ((4 \varepsilon_{F} s+2) +  \varepsilon_{F^c} (d-k) )   \E_{\stoch}\left\|\nabla f(\x^*)\right\|_{\infty}^2 } \nonumber\\
   &\quad+\sqrt{s}  \sqrt{\E_{\stoch}\left\|\nabla f_{\stoch}(\x^*)\right\|_{\infty}^2}) + \frac{\sqrt{\varepsilon_{\mu}}}{L_{s'}} \mu^2+ \eta \sqrt{2 \varepsilon_{\text{abs}}\mu^2 }\nonumber\\
                                &=  \rho\left\|\x^{t}-\x^*\right\| +   \eta (\sqrt{ (4 \varepsilon_{F} s+2) +  \varepsilon_{F^c} (d-k)        } +\sqrt{s})  \sqrt{\E_{\stoch}\left\|\nabla f_{\stoch}(\x^*)\right\|_{\infty}^2}\nonumber\\
                                &\quad+ \frac{\sqrt{\varepsilon_{\mu}}}{L_{s'}} \mu+ \eta \sqrt{2 \varepsilon_{\text{abs}}\mu^2 } \nonumber\\
                                  &\overset{(a)}{\leq} \rho\left\|\x^{t}-\x^*\right\| +   \eta (\sqrt{ (4 \varepsilon_{F} s+2) +  \varepsilon_{F^c} (d-k)        } +\sqrt{s})  \sigma\nonumber\\
                                &\quad+ \frac{\sqrt{\varepsilon_{\mu}}}{L_{s'}} \mu+ \eta \sqrt{2 \varepsilon_{\text{abs}}\mu^2 } \nonumber\\
                                                                    &\leq \rho\left\|\x^{t}-\x^*\right\| +   \eta (\sqrt{ (4 \varepsilon_{F} s+2) +  \varepsilon_{F^c} (d-k)        } +\sqrt{s})  \sigma\nonumber\\
                                &\quad+ \frac{\sqrt{\varepsilon_{\mu}}}{L_{s'}} \mu+ \eta \sqrt{2 \varepsilon_{\text{abs}}\mu^2 } \label{eq:pluglem17}
 \end{align}
 Where (a) follows from the $\sigma$-FGN assumption. We now consider $\x^{t+1}$,  that is, the best-$k$-sparse approximation of $\z^t:=\x^t - \eta \hat{\nabla} f_{\stoch}\left(\x^{t}\right) $ from the hard thresholding operation in SZOHT.  We can notice that $\x_F^t = \x^t$ (because $\text{supp}(\x^t)=F^{(t)}\subset F$),  which gives $\y^t = \z_F^t$. Since $F^{(t+1)}\subset F$, the coordinates of the top $k$ magnitude components of $\z^t$ are in $F$, so they are also those of the top $k$ magnitude components of  $\z_F^t=\y^t$. Therefore,  $\x^{t+1}$  is also the best k-sparse approximation of $\y^t$. Therefore, using Corollary \ref{cor:cor17},  we obtain:
 \begin{equation*}
 \|  \x^{t+1}  - \x^* \| \leq \gamma \| \y^t-\x^* \|      
\end{equation*}
with:
\begin{equation}
\gamma:= \sqrt{1 + \left( k^*/k + \sqrt{\left(4 + k^*/k\right)k^*/k}\right)/2}    \label{eq:gammadef}
\end{equation}
Where $k^*=\|\x^*\|_0$. Plugging this into \eqref{eq:pluglem17} gives:
  \begin{align*}
    \E_{\stoch, \uu}  \|\x^{t+1}-\x^*\| &\leq \gamma\rho\left\|\x^{t}-\x^*\right\|+ \gamma \eta   (\sqrt{ (4 \varepsilon_{F} s+2) +  \varepsilon_{F^c} (d-k)  }   ) + \sqrt{s})  \sigma\\
    &\quad+  \gamma  \frac{\sqrt{\varepsilon_{\mu}}}{L_{s'}}\mu + \eta \sqrt{2 \varepsilon_{\text{abs}} } \mu
 \end{align*}
  This will allow us to obtain the following final result:
   \begin{align}
     \E  \|\x^{t+1}-\x^*\| &\leq \gamma\rho\left\|\x^{t}-\x^*\right\|+ \gamma      \underbrace{\eta \left(\sqrt{ (4 \varepsilon_{F} s+2) +  \varepsilon_{F^c} (d-k) }  +\sqrt{s}\right)}_{:= a} \sigma\nonumber\\
     &\quad+  \gamma \underbrace{\left(\frac{\sqrt{\varepsilon_{\mu}}}{L_{s'}} + \eta \sqrt{2 \varepsilon_{\text{abs}} }\right)}_{:= b} \mu  \label{eq:onestep}
\end{align}
with $\eta = \frac{\nu_s}{(4\varepsilon_F +1 )L_{s'}^2} $ and $\rho^2 =  1 - \frac{2\nu_s^2}{(4  \varepsilon_F +1)L_{s'}^2} $.
   We need to have $\rho \gamma < 1 $ in order to have a contraction at each step. Let us suppose that $k \geq \rho^2 k^*/(1-\rho^2)^2$: we will show that this value for $k$ allows to verify that condition on $\rho \gamma$. That implies $\frac{k^*}{k} \leq \frac{(1-\rho^2)^2}{\rho^2}$.  We then have, from the definition of $\gamma$ in \eqref{eq:gammadef}:
   \begin{align}
     \gamma^2& \leq 1 +  \left(\frac{(1-\rho^2)^2}{\rho^2} + \sqrt{\left(4+\frac{(1-\rho^2)^2}{\rho^2}\right)\frac{(1-\rho^2)^2}{\rho^2}}\right) \frac{1}{2} \nonumber\\
     &=  1 +  \left(\frac{(1-\rho^2)^2}{\rho^2} + \sqrt{\left(\frac{4\rho^2 + 1+ \rho^4 - 2\rho^2}{\rho^2}\right)\frac{(1-\rho^2)^2}{\rho^2}}\right) \frac{1}{2} \nonumber\\
             &=  1 +  \left(\frac{(1-\rho^2)^2}{\rho^2} + \sqrt{\frac{(1+\rho^2)^2(1-\rho^2)^2}{\rho^4}}\right) \frac{1}{2} \nonumber\\
             &= 1 +  \left(\frac{(1-\rho^2)^2}{\rho^2} + \frac{(1+\rho^2)(1-\rho^2)}{\rho^2}\right) \frac{1}{2} =   1 +  \left(\frac{(1-\rho^2)(1- \rho^2 + 1 + \rho^2)}{\rho^2}\right) \frac{1}{2} \nonumber\\
     &=     1 +  \frac{(1-\rho^2)}{\rho^2} = \frac{1}{\rho^2} \label{eq:smoothsimplification}
   \end{align}
Therefore, we indeed have $\rho \gamma \leq 1$ when choosing $k \geq \rho^2 k^*/(1-\rho^2)^2$.

   Unrolling inequality \eqref{eq:onestep}  through time, we then have, at iteration $t+1$, and denoting by $\stoch^{t+1}$ the noise drawn at time step $t+1$ and $\uu^{t+1}$ the random directions $\uu_1, ..., \uu_q$ chosen at time step $t+1$, from the law of total expectations: 
   \begin{align*}
     \E  \|\x^{t+1}-\x^*\| &=   \E_{ \stoch^{t}, \uu^t, .., \stoch^1, \uu^1}\E_{\stoch^{t+1}, \uu^{t+1} | \stoch^{t}, \uu^{t}, .., \stoch^1, \uu^1} \|\x^{t+1}-\x^*\| \\
                           &\leq   \E_{ \stoch^{t}, \uu^t, .., \stoch^1, \uu^1} [\gamma\rho\|\x^{t}-\x^*\|+ \gamma a  \sigma+ \gamma b \mu]\\
                           &= \gamma\rho  \E_{ \stoch^{t}, \uu^t, .., \stoch^1, \uu^1} [\|\x^{t}-\x^*\|]+ \gamma a  \sigma+ \gamma b \mu\\
                           &\leq (\gamma\rho)^2  \E_{ \stoch^{t-1}, \uu^{t-1}, .., \stoch^1, \uu^1} [\|\x^{t-1}-\x^*\|]+ (\gamma \rho)^2 a  \sigma\\
     &\quad+ \gamma a  \sigma+ (\gamma \rho)^2  b \mu + \gamma b \mu\\
                           &\leq (\gamma \rho)^{t+1}\|\x^{(0)}-\x^*\| + \left(\sum_{i=0}^{t} (\gamma \rho)^i\right) \gamma a  \sigma+ \left(\sum_{i=0}^{t} (\gamma \rho)^i\right) \gamma b \mu\\
                           &=       (\gamma \rho)^{t+1}\|\x^{(0)}-\x^*\| + \frac{1 - (\gamma\rho)^t}{1 - \gamma\rho} \gamma a  \sigma+ \frac{1 - (\gamma\rho)^t}{1 - \gamma\rho}  \gamma b \mu\\
     &\leq (\gamma \rho)^{t+1}\|\x^{(0)}-\x^*\| + \frac{1}{1 - \gamma\rho} \gamma a  \sigma+ \frac{1}{1 - \gamma\rho}  \gamma b \mu
   \end{align*}
   Where the last inequality follows from the fact that $\rho \gamma <1$.
\end{proof}

\subsection{Proof of Remark \ref{rem:firstcond}}
\label{sec:firstcond}

\begin{proof}
We show below that, due to the complex impact of $q$ and $k$ on the convergence analysis in our ZO \textbf{+} HT (hard-thresholding) setting (compared to ZO only), $q$ cannot be taken as small as we want here (in particular we can never take $q=1$, which is different from classical ZO algorithms such as \cite[Corollary 3]{liuadmm}), if we want Theorem 1 to apply with $\rho \gamma <1$. In other words, there is a necessary (but not sufficient) minimal (i.e. $>1$) value for $q$.

 A necessary condition for Theorem 1 to describe convergence of SZOHT is that $\rho\gamma <1$. From the expressions of $\rho$ and $\gamma$ We have $\rho=\rho(q, k)$, and $\gamma= \gamma(k)$. We recall those expressions below:

$ \gamma = \sqrt{1 + \left( k^*/k + \sqrt{\left(4 + k^*/k\right)k^*/k}\right)/2} $

$\rho^2=  1 - \frac{\nu_s^2}{(4  \varepsilon_F+1)L_{s'}^2}  = 1 - \frac{1}{(4  \varepsilon_F+1)\kappa^2}  $ with $\kappa=\frac{L_{s'}}{\nu_s}$.\\
with:
$\varepsilon_F = \frac{2d}{q(\sus + 2)}   \left(\frac{(s-1)(\sus-1)}{d-1} + 3\right) + 2$, with $s =2k+k^* $ (we consider the smallest $s$ possible from Theorem \ref{thrm:cvrate})\\
So therefore:
\begin{align*}
  \rho^2 &= 1 - \frac{1}{\left[\frac{8d}{q(\sus +2)}(\frac{(s-1)(\sus-1)}{d-1} + 3) + 9\right]\kappa^2}  \\
  &= 1 - \frac{1}{\left[\frac{8d}{q(\sus +2)}(\frac{(2k + k^*-1)(\sus-1)}{d-1} + 3) + 9\right]\kappa^2}  
\end{align*}

Let us define  $a := \frac{16d\kappa^2(\sus-1)}{q(\sus+2)(d-1)}$ and $b := \kappa^2\left[\frac{8d}{q(\sus + 2)} [\frac{(\sus - 1)(k^*-1)}{d-1}+ 3] +9\right]$

We then have:

$$\rho^2 = 1 - \frac{1}{a k + b} $$

To ensure convergence, we need to have $\rho\gamma <1$, therefore (following the same derivation as in \eqref{eq:smoothsimplification}) a necessary condition that we need to verify is  $k\geq \rho^2 k^*/(1-\rho^2)^2$.

Which means we need:

\begin{align}
k &\geq \frac{\left(1-\frac{1}{a k+b}\right) k^{*}}{\left(\frac{1}{a k+b}\right)^{2}} \label{eq:biggerk}\\
k & \geq\left[(a k+b)^{2}-(a k+b)\right] k^{*} \\
k &\geq k^{*}\left[a^{2} k^{2}+2 a b k+b^{2}-a k-b\right] \\
 0 & \geq k^{*}a^2 k^{2}+\left(2 a b-\frac{1}{k^{*}}-a\right) k^*k+\left(b^{2}-b\right) k^{*}
\end{align}
If we want that there exist a $k$ such that this is true, we need (since $k^* \geq 0$): 
$$\Delta  \geq 0$$
with: 
\begin{align*}
  \Delta&:=k^{*2}(2 a b-\frac{1}{k^*}-a)^{2}-4 k^{*^{2}} a^{2}\left(b^{2}-b\right)\\
&= k^{*^{2}}\left(4 a^{2} b^{2}+\left(\frac{1}{k^{*}}+a\right)^{2}-4 a b\left(\frac{1}{k^{*}}+a\right)\right) -4 k^{*2} a^{2}\left(b^{2}-b\right)
\end{align*}
\begin{align*}
&=k^{* 2}\left[4 a^{2} b^{2}+\frac{1}{k^{*^{2}}}+a^{2}+\frac{2a}{k^{* }}-\frac{4 a b}{k^{*}}-4 a^{2} b-4 a^{2} b^{2}+4 a^2 b\right]\\
&=1+a^{2} k^{* 2}+2 a k^{*}-4 a b k^{*}
\end{align*}
\begin{equation}
  \label{eq:ii}
\Delta \geq 0 \Rightarrow  1+a^{2} k^{* 2}+2 a k^{*}\geq 4 a b k^{*}  
\end{equation}
Let us express $a$ and $b$ in terms of $q$, as: 
\begin{align}
&a=\frac{A}{q} \quad \text { with } \quad A=\frac{16 d \kappa^{2}\left(\sus-1\right)}{\left(\sus+2\right)(d-1)}\label{eq:eqfora}\\
&b=\frac{B}{q}+C \quad \text{ with}  \quad B=\kappa^{2}\left[\frac{8 d}{\left(\sus+2\right)}\left( \frac{\left(\sus-1\right)\left(k^{*} -1\right)}{d-1}+ 3\right) \right]\label{eq:eqforb}\\
&\text{and with } C=9 \kappa^{2}
\end{align}
So plugging in \eqref{eq:ii}, what we need is: 
\begin{align*}
1+\frac{A^{2}}{q^{2}} k^{*^{2}}+2 \frac{A}{q} k^{*} &\geq 4 \frac{A}{q}\left(\frac{B}{q}+C\right) k^{*}\\
q^{2}+A^{2} k^{* 2}+2 A k^{*} q &\geq 4 A B k^{*}+4 C A q k^{*}\\
q^{2}+q\left(2 A k^{*}-4 C A k^{*}\right)+A^{2} k^{* 2}-4 A B k^{*} &\geq 0
\end{align*}
To ensure that, we need to compute $\Delta'$, defined as:
  \begin{align*}
\Delta^{\prime} &:=\left(2 A k^{*}-4 C A k^{*}\right)^{2}-4\left(A^{2} k^{* 2}-4 A B k^{*}\right)\\
&=4 A^{2} k^{* 2}+16 C^{2} A^{2} k^{* 2}-16 C A^{2} k^{* 2}-4 A^{2} k^{* 2}+16 A B k^{*}\\
&=16 C A^{2} k^{* 2}(C - 1)+16 A B k^{*}=16 A k^{*}\left[k^{*} C(C-1) A + B \right]
\end{align*}
We now have:
$$C=9 \kappa^{2} \Rightarrow C \geq 1 \Rightarrow \Delta^{\prime} \geq 0$$
Therefore, there is a minimal value for $q$, and it is:
\begin{align*}
q \geq q_{\text{min}}
\end{align*}
With:
\begin{align}
  q_{\text{min}} &=  \frac{-\left(2 A k^{*}-4 C A k^{*}\right)+\sqrt{16 C A^{2} k^{* 2}(C-1)+16 A B k^{*}}}{2} \label{eq:qminit}\\
  &=  \frac{2 A k^{*}\left(2 C -1\right)+\sqrt{16 A^2 k^{* 2}\left[ C(C-1) + \frac{B}{Ak^*}\right]}}{2}
\end{align}
\textbf{Case $\sus > 1$:} Assuming $\sus>1$ gives $A>0$, and since
$A=\frac{16 d \kappa^{2}\left(\sus-1\right)}{\left(\sus+2\right)(d-1)}$ and $B= \frac{8\kappa^2d}{\sus + 2} (\frac{(\sus - 1)(k^* - 1)}{d-1} + 3 )$\\
This gives: $\frac{B}{Ak^*} = \frac{1}{2} - \frac{1}{2k^*}+ \frac{3}{2} \frac{d-1}{k^*(\sus-1)}$\\
Therefore: 
$q_{\text{min}}=A k^{*}\left[2 C-1+2 \sqrt{C(C-1)+\frac{1}{2}-\frac{1}{2 k^{*}}+\frac{3}{2} \frac{d-1}{k^*(\sus-1)}}\right]$\\
with $C=9\kappa^2$, which reads:

$$q_{\text{min}} = \frac{16 d (\sus-1) k^* \kappa^{2}}{\left(\sus+2\right)(d-1)} \left[18 \kappa^2 -1+2 \sqrt{9\kappa^2(9 \kappa^2-1)+\frac{1}{2}-\frac{1}{2 k^{*}}+\frac{3}{2} \frac{d-1}{k^*(\sus-1)}}\right]$$
\textbf{Case $\sus=1$:} In the case $\sus =1$, we have $A=0$, so therefore, from \eqref{eq:qminit}, $q_{\text{min}} = 0$, so the necessary condition on $q$ as above so that there exist $k$ such that: $k \geq \rho^2 k^*/(1-\rho^2)^2$ does not apply here.
We may therefore think that it may be  possible to take $q=1$ in that case. However, there is another condition on $k$ that should also be enforced, which is that $k \leq d$ (since we cannot keep more components than $d$). And in that $\sus=1$ case, we have $a=0$, and $b=\kappa^2[8\frac{d}{q} + 9]$ (from \eqref{eq:eqfora} and \eqref{eq:eqforb}). Now, enforcing the condition $k \geq k^*[(ak+b)^2 - (ak+b)] = k^* b(b-1)$ leads to the following chain of implications (i.e. each downstream assertion is a necessary condition for the upstream assertion):

\begin{align}
  \frac{k}{k^*} \geq b(b-1) \quad \text{and} \quad k \leq d &\implies  \frac{d}{k^*} \geq (b-1)^2 \implies  \sqrt{\frac{d}{k^*}} + 1 \geq b \implies \sqrt{\frac{d}{k^*}} + 1 \geq \frac{B}{q} + C \nonumber \\
                              &\implies \sqrt{\frac{d}{k^*}} + 1 - C\geq \frac{B}{q}\nonumber\\
  &\implies q\geq \frac{B}{\sqrt{\frac{d}{k^*}} + 1 - C} \quad \text{and}\quad C - \sqrt{\frac{d}{k^*}} + 1 > 0\nonumber\\
  &\implies q \geq \frac{B}{\sqrt{\frac{d}{k^*}} + 1} \implies q \geq  \frac{8 \kappa^2 d}{\sqrt{\frac{d}{k^*}} + 1}  \label{eq:rhsboundq}
\end{align}
Where the last inequality follows from the expression of $B$ in \eqref{eq:eqforb} when $\sus=1$.

So the right hand side in \eqref{eq:rhsboundq} is also a minimal necessary value for $q$ in this case, though for a different reason than in the case $\sus>1$.

\end{proof}

\subsection{Proof of Corollary \ref{cor:specialq}}
\label{sec:proof_specialq}

\begin{proof}

  We first restrict the result of Theorem \ref{thrm:cvrate} to a particular $q$. By inspection of Proposition \ref{prop:zograd} (b), we choose  $q$ such that the part of  $\varepsilon_{F}$ that depends on $q$ becomes $1$: we believe this will allow to better understand the dependence between variables in our convergence rate result, although other choices of $q$ are possible. Therefore, we choose:
  \begin{equation}
    \label{eq:qval}
q^{\prime}  := \frac{2d}{\sus + 2}\left(\frac{(s-1)(\sus-1)}{d-1} + 3\right)     
  \end{equation}
  so that we obtain: $\varepsilon_{F}^{\prime} :=  1+2=3$ (from Proposition \ref{prop:zograd} (b)), which also implies :
  $$\eta^{\prime} := \frac{\nu_s}{(4\varepsilon_F^{\prime} +1 )L_{s'}^2} = \frac{\nu_s}{13L_{s'}^2}$$
  and:
  \begin{equation}
    \label{eq:newrho}
{\rho^{\prime}}^2:=  1 - \frac{2\nu_s^2}{(4  \varepsilon_F^{\prime} +1)L_{s'}^2} = 1 - \frac{2\nu_s^2}{13L_{s'}^2}    
\end{equation}

Now, regarding the value of $q$, we also note that any value of random directions $q^{\prime \prime} \geq q^{\prime}$ can be taken too, since the bound in Proposition \ref{prop:zograd} (b) would then still be verified for $ \varepsilon_{F}^{\prime}$ (that is, we would still have $\E\|\hat{\nabla}_{F} f_{\stoch}(\x)\|^{2} \leq \varepsilon_F^{\prime} \|\nabla_F f_{\stoch}(\x)\|^{2} +  \varepsilon_{F^c}^{\prime}  \|\nabla_{F^c} f_{\stoch}(\x)\|^{2}+   \varepsilon_{abs} \mu ^2$) (with $ \varepsilon_{F^c}^{\prime}$ the value of $\varepsilon_{F^c}$ for $q=q^{\prime}$).\\
Therefore, we will choose a value $q^{\prime \prime}$ so that our result is simpler. First, notice that $s\leq d \implies 1 - \frac{1}{s} \leq 1-\frac{1}{d}\implies \frac{s-1}{s} \leq \frac{d-1}{d} \implies \frac{s-1}{d-1} \leq \frac{s}{d}$. Therefore, if we take $q\geq2s + 6\frac{d}{\sus} $, we will also have $q\geq  \frac{2d}{\sus + 2}\left(\frac{(s-1)(\sus-1)}{d-1} + 3\right)   =q^{\prime}$.\\

Let us now impose a lower bound on $k$ that is slightly (twice) bigger than the lower bound from Theorem \ref{thrm:cvrate}. As will become clear below, this allows us to have a $\rho\gamma$ enough bounded away from 1, which guarantees a reasonable constant  in the $\OO$ notation for the query complexity (see the end of the proof).
Let us therefore take:
\begin{equation}
  \label{eq:goodk}
k \geq 2 k^* \frac{\rho^2}{(1-\rho^2)^2}  
\end{equation}

and plug the value of $\rho$ above into the expression:
  \begin{align*}
    k \geq 2 k^* \frac{{\rho^{\prime}}^2}{(1-{\rho^{\prime}}^2)^2}&\iff k \geq 2 k^* \frac{1 - \frac{2\nu_s^2}{13L_{s'}^2}}{(\frac{2\nu_s^2}{13L_{s'}^2})^2}\iff k \geq 2 k^*\left(\left(\frac{13L_{s'}^2}{2\nu_s^2}\right)^2 - \frac{13L_{s'}^2}{2\nu_s^2}\right) \\
    &\iff  k \geq 2 k^*(\frac{13}{2}\kappa^2)(\frac{13}{2}\kappa^2 - 1)    
  \end{align*}
  With $\kappa$ denoting $\frac{L_{s'}}{\nu_s}$. Therefore, if we take:

  \begin{equation*}
k \geq (86\kappa^4 - 12 \kappa^2 )k^*    
  \end{equation*}
  we will indeed verify the formula above $k \geq 2 k^*(\frac{13}{2}\kappa^2)(\frac{13}{2}\kappa^2 - 1)$.\\
We now turn to describing the query complexity of the algorithm:
To ensure that $(\gamma\rho)^t \|\x^{(0)} - \x^*\| \leq \varepsilon$, we need:
\begin{equation}
  \label{eq:rawcomplexity}
t  \geq \frac{1}{\log \frac{1}{\gamma \rho}}\log (\frac{1}{\varepsilon})\log(\|\x^{(0)} - \x^*\|)  
\end{equation}
with $\gamma \rho$ belonging to the interval $(0, 1)$.
Let us compute more precisely an upper bound to $\rho \gamma$ in this case, to show that it is reasonably enough bounded away from 1:
Taking $k$ as described in \eqref{eq:goodk}, and plugging that value into the expression of $\gamma$ from Theorem \ref{thrm:cvrate}, we obtain:
\begin{align*}
  \gamma^2 &=  1 + \left(    \frac{(1-\rho^2)^2}{2\rho^2}      + \sqrt{\left(4 +  \frac{(1-\rho^2)^2}{2\rho^2}\right)  \frac{(1-\rho^2)^2}{2\rho^2}}\right)/2  \\
           &\leq 1 + \frac{1}{\sqrt{2}}    \left(    \frac{(1-\rho^2)^2}{\rho^2}      + \sqrt{\left(4 +  \frac{(1-\rho^2)^2}{\rho^2}\right)  \frac{(1-\rho^2)^2}{\rho^2}}\right)/2  \\
  &\overset{(a)}{=} 1 + \frac{1}{\sqrt{2}} \frac{1 - \rho^2}{\rho^2}
\end{align*}
Where the simplification in (a) above follows similarly to \eqref{eq:smoothsimplification}.
Therefore, in that case, we have:
\begin{align*}
  \rho^2\gamma^2 &\leq \rho^2 + \frac{1}{\sqrt{2}} (1- \rho^2) = \frac{1}{\sqrt{2}} + \rho^2(1 - \frac{1}{\sqrt{2}}) \\
  &=  \frac{1}{\sqrt{2}} + (1 - \frac{2}{13\kappa^2})(1 - \frac{1}{\sqrt{2}})  = 1 - \frac{2 (1 - \frac{1}{\sqrt{2}})}{13 \kappa^2} \overset{(a)}{\leq} 1 - \frac{1}{26\kappa^2}
\end{align*}
Where (a) follows because $(1 - \frac{1}{\sqrt{2}}) \approx 0.29 \geq 1/4$
Therefore:
\begin{equation}
  \label{eq:1}
  \frac{1}{(\rho\gamma)^2} \geq \frac{1}{ 1 - \frac{1}{26\kappa^2}} 
\end{equation}

Given that $\log(\frac{1}{1 - x}) \geq x $ for all $x\in[0, 1)$, we have:
$$ \log\left(  \frac{1}{(\rho\gamma)^2} \right) \geq  \frac{1}{26\kappa^2}$$
  Therefore:
  $$ \frac{1}{\log(\frac{1}{\rho\gamma})}  =  \frac{2}{\log(\frac{1}{(\rho\gamma)^2})} \leq 52 \kappa^2$$
  Therefore, plugging this into \eqref{eq:rawcomplexity}, we obtain that with  $t \geq 52 \kappa^2 \log(\frac{1}{\varepsilon})\log(\|\x^{(0)} - \x^*\|) = \OO(\kappa^2\log(\frac{1}{\varepsilon}))$ iterations, we can get $(\gamma\rho)^t \|\x -\x^*\| \leq \varepsilon $.

  To obtain the query complexity (QC), we therefore just need to multiply the number of iterations by the number of queries per iteration $q=2s + 6 \frac{d}{\sus}$: to ensure  $(\gamma\rho)^t \|\x -\x^*\| \leq \varepsilon $, we need to query the zeroth-order oracle at least the following number of times: $(2s + 6 \frac{d}{\sus})  52 \kappa^2 \log(\frac{1}{\varepsilon}) \log(\|\x^{(0)} - \x^*\|)  = \OO((k + \frac{d}{\sus})\kappa \log(\frac{1}{\varepsilon}))$, since $s = 2 k + k^*$.

\subsection{Proof of Corollary \ref{cor:s2d}}
\label{sec:proof_s2d}

Almost all $f_{\stoch}$ are $L$-smooth, which is equivalent to saying that they are $(L, d)$-RSS. So we can directly plug $\sus=d$ in equation \eqref{eq:qval}, which gives a necessary value for $q$ of:
  \begin{equation}
    \label{eq:qvals2d}
 q  = \frac{2d}{d+2} (s+2)    
  \end{equation}
  Since any value of $q$ larger than the one in \eqref{eq:qvals2d} is valid, we choose $q \geq 2(s+2) (\geq \frac{2d}{d+2} (s+2)) $ for simplicity. The query complexity is obtained similarly as in the proof of Corollary \ref{cor:specialq} above, with that new value for $q$ (the number of iterations needed is unchanged from the proof of Corollary \ref{cor:specialq}), only the query  complexity $q$ per iteration changes), which means we need to query the zeroth-order oracle the following number of times:  $2(s + 2)  52 \kappa^2 \log(\frac{1}{\varepsilon}) \log(\|\x^{(0)} - \x^*\|)  = \OO(k\kappa \log(\frac{1}{\varepsilon}))$
\end{proof}

\section{Projection of the gradient estimator onto a sparse support}
\label{sec:appfig}

Below we plot the true gradient  $\nabla f(x)$ and its estimator  $\hat{\nabla} f(x)$ (for $q=1$), as well as their respective projections  $\nabla_F f(x)$ and  $\hat{\nabla}_F f(x)$, with $F=\{0, 1\}$ (i.e. $F$ is the hyperplane $z=0$), for $n_{\text{dir}}$ random directions. In Figure \ref{fig:3d}, due to the large number of random directions, we plot them as points not vectors. For simplicity, the figure is plotted for $\mu \rightarrow 0$, and $s_2=d$. We can see that even though gradient estimates $\hat{\nabla} f(x)$ are poor estimates of  $\nabla f(x)$, $\hat{\nabla}_F f(x)$ is a better estimate of $\nabla_F f(x)$. 
\begin{figure}[htbp]
  \centering
  \subfigure[$n_{\text{dir}}=1$]{\includegraphics[scale=0.8]{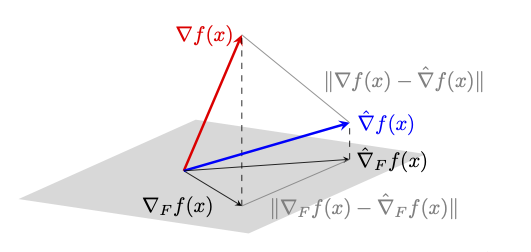}}\label{fig:vectors}
  \subfigure[$n_{\text{dir}}=10^6$]{\includegraphics[scale=0.55, trim={0cm 0.5cm 0cm 2.3cm},clip]{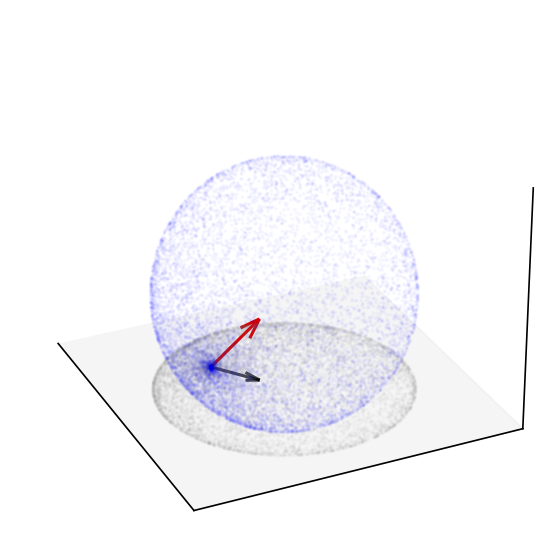}\label{fig:3d}}
  \caption{$\nabla f(x)$ and $\hat{\nabla} f(x)$ and their  projections  $\nabla_F f(x)$ and  $\hat{\nabla}_F f(x)$ onto $F$}
  \label{fig:fig}
\end{figure}
\begin{remark}
  An interesting fact that can be observed in Figure \ref{fig:3d} above is that when $\mu \rightarrow 0$ and $\sus=d$, the ZO gradient estimates belong to a sphere. This comes from the fact that, in that case, the ZO estimate using the random direction $\uu$ is actually a  directional derivative (scaled by d): $\hat{\nabla} f(\x) = d \langle \nabla f(\x), \uu\rangle \uu $, for which we have :
  \begin{align*}
    \|\hat{\nabla} f(\x) - \frac{d}{2}\nabla f(\x)\|^2 &= d^2(\langle \nabla f(\x), \uu\rangle)^2 \langle \uu , \uu \rangle + \frac{d^2}{4} \|\nabla f(\x)\|^2 \\
    &\quad- d^2 \langle \nabla f(\x) , \uu\rangle \langle  \uu, \nabla f(\x)\rangle \\
    &= \frac{d^2}{4} \|\nabla f(\x)\|^2
  \end{align*}

  (since $\|\uu\|=1$). That is, gradient estimates belong to a sphere of center $\frac{d}{2} \nabla f(\x)$ and radius $\frac{d}{2} \|\nabla f(\x)\|$. However, the distribution of  $ \hat{\nabla} f(\x)$ is not uniform on that sphere: it is more concentrated around $\bm{0}$ as we can observe in Figure \ref{fig:3d}.
\end{remark}

\section{Value of $\rho \gamma$ depending on $q$ and $k^*$}
\label{sec:value-rho-gamma}

In this section, we further illustrate the importance on the value of $q$ as discussed in Remark \ref{rem:firstcond}, by showing in Figure \ref{fig:rhogamma}  that if $q$ is too small, then there does not exist any $k$ that verifies the condition $k \geq \frac{k^*\rho^2}{(1 - \rho^2)^2}$, no matter how small is $k^*$ (i.e., even if $k^*=1$). However,  if $q$ is large enough, then there exist some $k^*$ such that this condition is true. To generate the curves below, we simply use the formulas for $\gamma = \gamma(k, k^* )$  and $\rho = \rho(s, q)$ with $s = 2k + k^*$ from Theorem \ref{thrm:cvrate}, and with $d=30000$ and $\sus=d$.

\begin{figure}[htp]
    \centering
    \subfigure[$q=200$]{\includegraphics[scale=0.28]{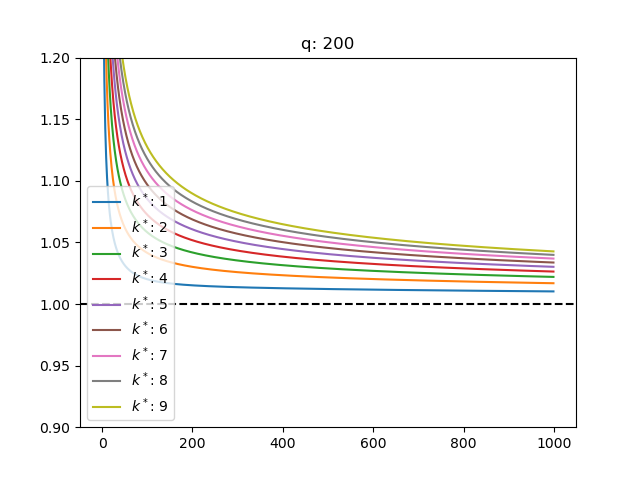}\label{fig:q1}}
    \subfigure[$q=5000$]{\includegraphics[scale=0.28]{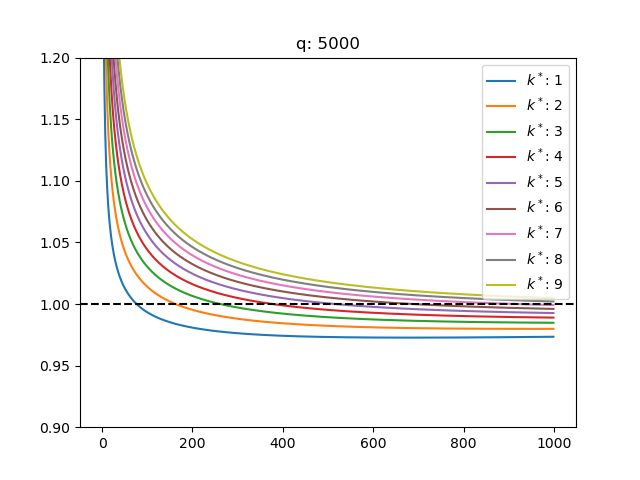}\label{fig:q2}}
    \subfigure[$q=30000$]{\includegraphics[scale=0.28]{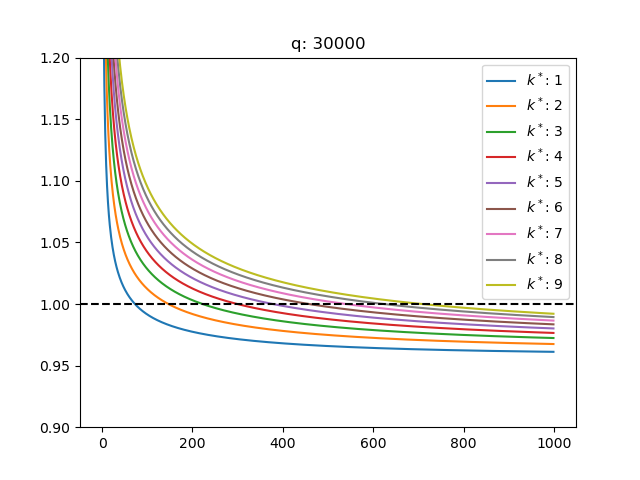}\label{fig:q3}}
      \caption{$\rho\gamma$ ($y$ axis) as a function of $k$ ($x$ axis) for several values of $q$ and $k^*$}\label{fig:rhogamma}
 \end{figure}

\section{Dimension independence/weak-dependence}
\label{sec:dimens-indep}
In this section, we show the dependence of SZOHT on the dimension. To that end, we consider minimizing the following synthetic problem:
$$\min_{\x} f(\x) \quad \text{s.t.} \quad \|\x\|_0 \leq k$$
with $k=500$, and $f$ chosen as: $f(\x) = \frac{1}{2}\|\x - \y\|^2$, with $\y_i = 0$ if $i<d-k^*$ and $\y_i = \frac{1}{(k^*-(d-i))}$ if $i > d-k^*$ with $k^*=5$. In other words, the $k^*$ last components of $\y$ are regularly spaced from $1/k^*$ to $1$: in a way, this simulates the recovery of a $k^*$-sparse vector $\y$  by observing only the squared deviation of some queries $\x$.
In that case, we can easily check that $f$ verifies the following properties:
\begin{itemize}
\item $f$ is  $L$-smooth with $L=1$, as well as $(L_{s'}, s')$-RSS for any $s'$ such that $1\leq s' \leq d$, with $L_{s'}=1$, and $(\nu_s, s)$-RSC with $s=2k+k^*$ and $\nu_s = 1$ (so $\kappa=\frac{L}{\nu_s} = \frac{L_{s'}}{\nu_s} = 1$)
\item $\y = \x^*= \arg\min_{\x} f(\x) \quad s.t. \quad \|\x\|_0 \leq k^* $
  \item $f(\y) = f(\x^*) =  0$
  \item $\nabla f(\y) = \bm{0}$ so $f$ is $\sigma$-FGN with $\sigma=0$
\end{itemize}
We also note that the above setting of $k$ and $k^*$ verifies $k \geq (86 \kappa^4 - 12 \kappa^2)k^*$ (since $\kappa=1$).
Finally, we initialize $\x^0$ such that ${\x^0}_i = 1/d$ if $d-k^* \geq i $  and $0$ otherwise. We choose this initialization and not $\x^0=\bm{0}$, just to ensure that $\nabla f(\x^0)_i \neq 0$ for any $i$: this way the optimization is really done over all $d$ variables, not just the $k^*$ last ones. In addition, this initialization ensures that $\|\x^0 - \x^*\|$ is constant no matter the $d$, which makes the convergence curves comparable.
We consider several settings of $\sus$ to showcase the dependence on the dimension below.
\paragraph{Dimension Independence}
\begin{itemize}
\item $\sus = d$: As from Corollary \ref{cor:s2d}, we take $q =  2(s + 2)$ with $s = 2k + k^*$ (i.e. $q = 2014$). We choose $\mu=1e-8$, to have the smallest possible system error due to zeroth-order approximations. As we can see in Figure \ref{fig:dimindep}, all curves are superimposed, which shows that the query complexity is indeed dimension independent, as described by Corollary \ref{cor:s2d}
\item $\sus = \OO(\frac{d}{k})$ (We choose $\sus=\lfloor\frac{d}{k}\rfloor$): As from Corollary \ref{cor:specialq}, we take $q  = 2s + 6\frac{d}{\sus} $ with $s = 2k + k^*$. In that case, from Corollary \ref{cor:specialq},  the query complexity will still be $\OO(k)$ (i.e. dimension independent), as a sum of two $\OO(k)$ terms, although larger than in the case $\sus=d$ above (since the constant from the $\OO$ notation in Corollary \ref{cor:specialq} will be larger here). We can observe that this is indeed the case in Figure \ref{fig:dimindepbis}.
\end{itemize}
\paragraph{Dimension weak-dependence}
We now turn to the case where $\sus$ is fixed. We choose $q$ as in Corollary \ref{cor:specialq} ($q  = 2s + 6\frac{d}{\sus} $ with $s = 2k + k^*$ ): the query now depends on $d$ in that case, as predicted by Corollary \ref{cor:specialq}, which can indeed be observed in Figure \ref{fig:dimweakdep}.
\begin{figure}[htp]
  \centering
  \begin{minipage}{0.28\textwidth}
    \centering
    \subfigure[$f(\x)$]{\includegraphics[scale=0.28]{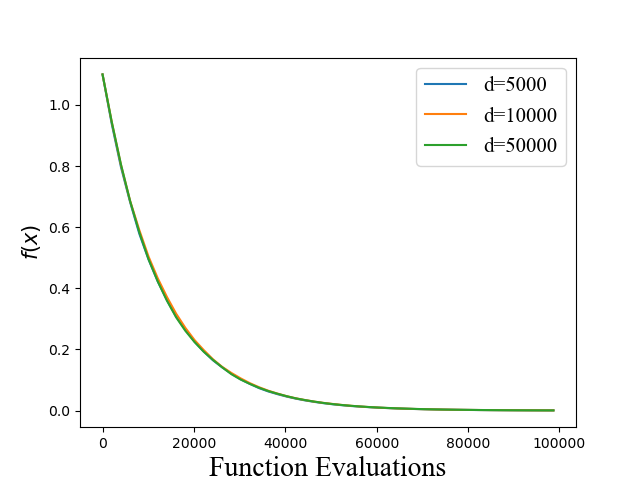}}
    \subfigure[$\|\x - \x^*\|$]{\includegraphics[scale=0.28]{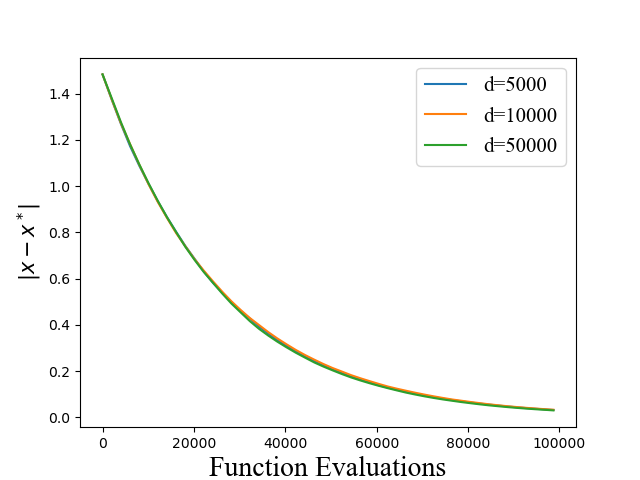}}
      \caption{$\sus=d$}\label{fig:dimindep}
  \end{minipage}
  \hfill
  \begin{minipage}{0.28\textwidth}
    \centering
    \subfigure[$f(\x)$]{\includegraphics[scale=0.28]{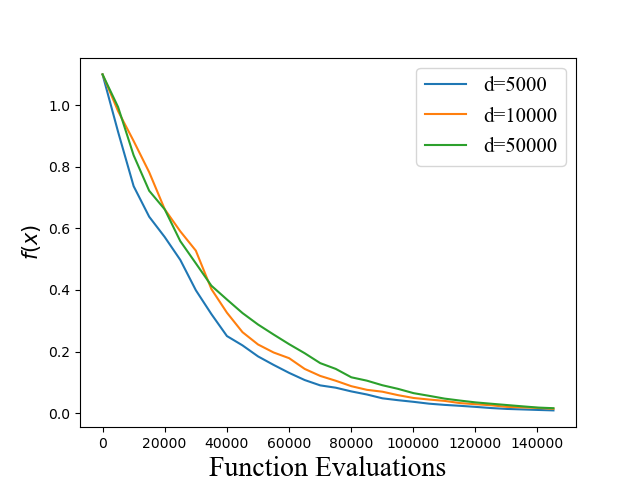}}
    \subfigure[$\|\x - \x^*\|$]{\includegraphics[scale=0.28]{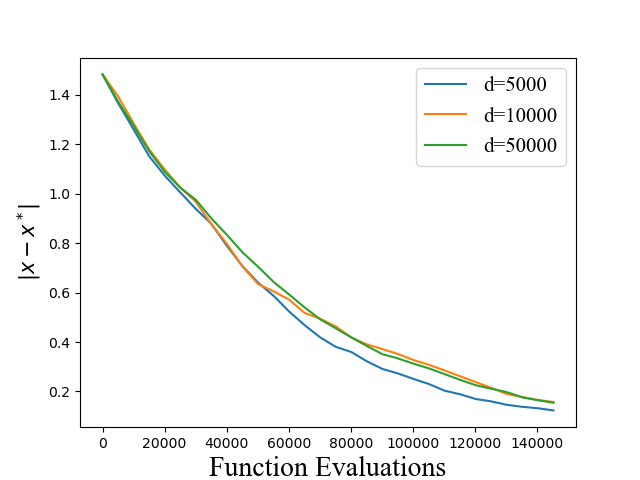}}
      \caption{$\sus = \lfloor \frac{d}{k} \rfloor$}\label{fig:dimindepbis}
      \end{minipage}
      \hfill
  \begin{minipage}{0.28\textwidth}
    \centering
    \subfigure[$f(\x)$]{\includegraphics[scale=0.28]{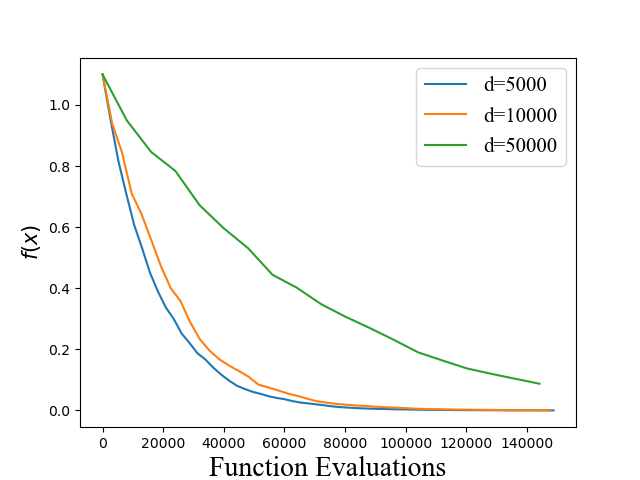}}
    \subfigure[$\|\x - \x^*\|$]{\includegraphics[scale=0.28]{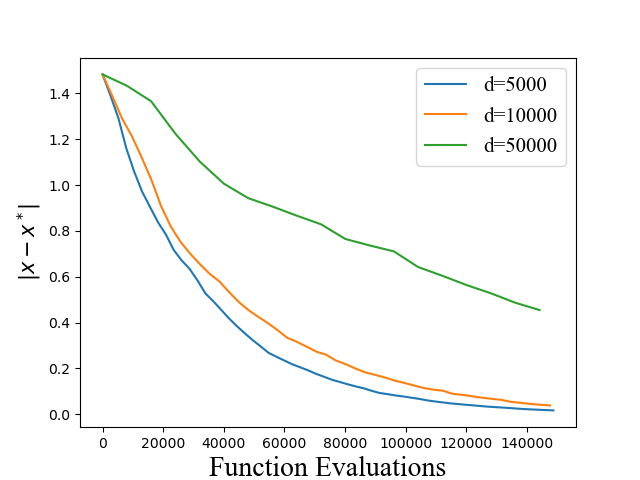}}
      \caption{$\sus = 50$}\label{fig:dimweakdep}
    \end{minipage}
    \caption{Dependence on the dimensionality of the query complexity}
 \end{figure}
 \section{Additional results on adversarial attacks}
\label{sec:addit-results-advers}
In this section, we provide additional results for the adversarial attacks problem in \ref{sec:zeroth-order-few}, in Figure \ref{fig:tableadv}. The parameters we used for SZOHT to generate that table are the same as in \ref{sec:zeroth-order-few}, except for MNIST, for which we choose $k=20$, $q=10$, and $\sus=10$, and for ImageNet, for which we choose $k=100000, \sus=20000$ and $q=100$. As we can see, SZOHT allows to obtain sparse attacks, contrary to the other algorithms, and with a smaller $\ell_2$ distance and a larger success rate, using less iterations: this shows that SZOHT allows to enforce sparsity, and efficiently exploits that sparsity in order to have a lower query complexity than vanilla sparsity constrained ZO algorithms.
\begin{figure}[htp]
  \centering
  \subfigure[MNIST]{\begin{tabular}{c c c c c}
  Method & ASR & $\ell_0$ dist. & $\ell_2$ dist.    & Iter \\
      \hhline{= = = = =}
  RSPGF & $78 \%$ & $100 \%$ & $10.9$ & 67 \\
    \midrule
  ZORO & $75 \%$ & $100 \%$ & $15.1$ & 550 \\
    \midrule
  ZSCG & $79 \%$ & $100 \%$ & $10.3$ & 252 \\
    \midrule
  \textbf{SZOHT} & $79 \%$ & $2.5 \%$ & $8.5$ & 36 \\
  \bottomrule
                    \end{tabular}}
                  \hfill
                                   \subfigure[CIFAR]{\begin{tabular}{c c c c c}
  Method & ASR &$\ell_0$ dist. & $\ell_2$ dist. & Iter \\
  \hhline{= = = = =}
  RSPGF & $83 \%$ & $100 \%$ & $4.1$ & 326 \\
    \midrule
  ZORO & $86 \%$ & $100 \%$ & $62.9$ & 592 \\
    \midrule
  ZSCG & $86 \%$ & $100 \%$ & $8.4$ & 126 \\
    \midrule
  \textbf{SZOHT} & $91 \%$ & $1.9 \%$ & $2.6$ & 26 \\
    \bottomrule
                                                     \end{tabular}}
                                                   \hfill
  \subfigure[ImageNet]{\begin{tabular}{c c c c c}
  Method & ASR & $\ell_0$ dist. & $\ell_2$ dist. & Iter. \\
  \hhline{= = = = =}
  RSPGF & $91 \%$ & $100 \%$ & $19.9$ & 137 \\
  \midrule
  ZORO & $90 \%$ & $100 \%$ & $111.9$ & 674 \\
    \midrule
  ZSCG & $76 \%$ & $100 \%$ & $111.3$ & 277 \\
    \midrule
  \textbf{SZOHT} & $95 \%$ & $37.3 \%$ & $10.5$ & 61 \\
    \bottomrule
                       \end{tabular}}
                \caption{Summary of results on adversarial attacks} \label{fig:tableadv}
\end{figure}

\end{document}